\documentclass[graybox]{article}
\usepackage{fullpage}
\usepackage{hyperref}
\usepackage{times,amssymb,url,amsmath,graphicx,ifpdf,booktabs,bm,amsthm}
\usepackage[ruled,vlined,linesnumbered,lined]{algorithm2e}
\usepackage[justification=centering]{caption}
\usepackage{subcaption}
\usepackage{multirow}
\usepackage{amsmath,amssymb}

\usepackage{mathptmx}       
\usepackage{helvet}         
\usepackage{courier}        
\usepackage{type1cm}        
\usepackage{makeidx}         
\usepackage{graphicx}        
\usepackage{multicol}        
\usepackage[bottom]{footmisc}

\usepackage{listings}  
\usepackage{xcolor}
\definecolor{backcolour}{rgb}{0.95,0.95,0.92}
\lstdefinestyle{pystyle}{
    backgroundcolor=\color{backcolour},
    keywordstyle=\color{magenta},
    basicstyle=\ttfamily\footnotesize,
    breakatwhitespace=false,
    breaklines=true,
    captionpos=b,
    keepspaces=true,
    numbers=left,
    numbersep=5pt,
    showspaces=false,
    showstringspaces=false,
    showtabs=false,
    tabsize=2
}
\usepackage{authblk}

 \newtheorem{property}{Property}
 \newtheorem{lemma}{Lemma}
\newtheorem{theorem}{Theorem}
 \newtheorem{corollary}{Corollary}

\graphicspath{{./figures/}}
\ifpdf  \fi

\def\calP{\mathcal{P}}
\def\eps{\epsilon}
\def\Mult{\mathrm{Mult}}

\def\diag{\mathrm{diag}}

\def\TV{\mathrm{TV}}
\def\var{\mathrm{var}}
\def\Pr{\mathrm{Pr}}
\def\CS{\mathrm{CS}}
\def\LC{\mathrm{LC}}
\def\KL{\mathrm{KL}}
\def\tr{\mathrm{tr}}

\def\dt{\mathrm{d}t}

\def\Ball{\mathrm{Ball}}
\def\calB{\mathcal{B}}

\def\FD{\mathrm{FD}}
\def\calI{\mathcal{I}}
\def\calX{\mathcal{X}}
\def\calC{\mathcal{C}}

\def\calV{\mathcal{V}}

\def\bbR{\mathbb{R}}
\def\Finner#1#2#3{{\langle {#1},\;{#2} \rangle}_{#3}}
\def\inner#1#2{ \langle {#1},{#2} \rangle }

\def\aff{\mathrm{affine}}
\def\co{\mathrm{hull}}
\DeclareMathOperator*{\argmax}{arg\,max}
\DeclareMathOperator*{\argmin}{arg\,min}
\def\st{\ :\ }

\def\FHR{\mathrm{FHR}}
\def\IG{\mathrm{IG}}
\def\HG{\mathrm{HG}}
\def\TG{\mathrm{TG}}
\def\BG{\mathrm{BG}}
\def\NH{\mathrm{NH}}  

\def\dP{{\mathrm{d}P}}
\def\dQ{{\mathrm{d}Q}}

\def\eqdef{{:=}}

\def\min{\mathrm{min}}
\def\max{\mathrm{max}}

\def\rhofd{\rho_{\mathrm{FD}}}
\def\rhohg{\rho_{\mathrm{HG}}}
\def\rhofhr{\rho_{\mathrm{FHR}}}
\def\rhoeuc{\rho_{\mathrm{EUC}}}
\def\rhol1{\rho_{\mathrm{L1}}}
\def\rhok{\rho_{\mathrm{K}}}

\begin{document}

\title{Clustering in Hilbert simplex geometry\footnote{Compared to the chapter published in the edited book ``Geometric Structures of Information''~\cite{ChapterHSG-2019,gsi-nielsen-2019} (Springer, 2019), this document includes the proof of information monotonicity of the non-separable Hilbert simplex distance (\S\ref{sec:hsgim}), reports a closed-form formula for Birkhoff's projective distance in the elliptope, and does not need a tailored algorithm to compute the Hilbert simplex distance but rather apply formula~\ref{eq:BG}. This paper also presents further experimental results on positive measures (Table~\ref{tab:eKLBG}: Birkhoff cone metric versus extended Kullback-Leibler divergence) and Thompson metric versus forward/reverse/symmetrized Kullback-Leibler divergences in the elliptope (Table~\ref{tab:tg}). We show that both the Aitchison and Hilbert simplex distances are norm distances on normalized logarithmic representations.}}

\date{}

\author[$\star$]{Frank Nielsen}
\affil[$\star$]{Sony Computer Science Laboratories Inc.}
\affil[$\star$]{Japan}
\affil[$\star$]{ORCID:~0000-0001-5728-0726}
\affil[$\star$]{{\small\tt Frank.Nielsen@acm.org}\vspace{.5em}}
\author[$\dagger$]{{Ke Sun}}
\affil[$\dagger$]{Data61, Australia}
\affil[$\dagger$]{The Australian National University}
\affil[$\dagger$]{ORCID:~0000-0001-6263-7355}
\affil[$\dagger$]{{\small\tt sunk@ieee.org}}
\date{}

\maketitle

\begin{abstract}
Clustering categorical distributions in the finite-dimensional probability simplex is a fundamental task met in many applications
dealing with normalized histograms.
Traditionally, the differential-geometric structures of the probability simplex have been used either by
(i) setting the Riemannian metric tensor to the Fisher information matrix of the categorical distributions, or
(ii) defining the dualistic information-geometric structure induced by a smooth dissimilarity measure, the Kullback-Leibler divergence.
In this work, we introduce for clustering tasks a novel computationally-friendly framework for modeling geometrically the probability simplex: The {\em Hilbert simplex geometry}.
In the Hilbert simplex geometry, the distance is the non-separable Hilbert's metric distance which
satisfies the property of information monotonicity with distance level set functions described by polytope boundaries.
We show that both the Aitchison and Hilbert simplex distances are norm distances on normalized logarithmic representations with respect to the $\ell_2$ and variation norms, respectively.
We discuss the pros and cons of those different statistical modelings, and benchmark experimentally  these different kind of geometries for center-based  $k$-means and $k$-center clustering.
Furthermore, since a canonical Hilbert distance can be defined on any bounded convex subset of the Euclidean space, we also consider   Hilbert's geometry of the elliptope of correlation matrices and study its clustering performances compared to Fr\"obenius and log-det divergences.
\end{abstract}


\section{Introduction and motivation}\label{sec:intro}

The categorical distributions and multinomial distributions are important probability distributions often met
in data analysis~\cite{CategoricalDA-2003}, text mining~\cite{MiningTextData-2012}, computer vision~\cite{bagoffeat-2009} and machine learning~\cite{Murphy-2012}.
A multinomial distribution over a set $\calX=\{e_0,\ldots, e_{d}\}$ of outcomes (e.g., the $d+1$ distinct colored faces of a die) is defined as follows:
Let  $\lambda_p^i>0$ denote the probability that outcome $e_i$ occurs for $i\in\{0,\ldots, d\}$ (with $\sum_{i=0}^d \lambda_p^i=1$).
Denote by $m$ the total number of events, with $m_i$ reporting the number of outcome $e_i$.
Then the probability $\Pr(X_0=m_0,\ldots, X_d=m_d)$ that a multinomial random variable $X=(X_0,\ldots, X_d)\sim\Mult(p=(\lambda_p^0,\ldots,\lambda_p^d),m)$ (where $X_i$ count the number of events $e_i$, and $\sum_{i=0}^d m_i=m$) is given by the following probability mass function (pmf):
$$
\Pr(X_0=m_0,\ldots, X_d=m_d)= \frac{m!}{\prod_{i=0}^d m_i!}  \prod_{i=0}^d \left( \lambda_p^i \right)^{m_i}.
$$
The multinomial distribution is called a binomial distribution when $d=1$ (e.g., coin tossing), a Bernoulli distribution when $m=1$,
and a ``multinoulli distribution'' (or categorical distribution) when $m=1$ and $d>1$.
The multinomial distribution is also called a generalized Bernoulli distribution.
A random variable $X$ following a multinoulli distribution is denoted by $X=(X_0,\ldots, X_d)\sim\Mult(p=(\lambda_p^0,\ldots,\lambda_p^d))$.
The multinomial/multinoulli distribution provides an important {\em feature representation} in machine learning that is often met
in applications~\cite{MetricLearning-2002, MultMixTextClustering-2007,ClusteringSimplex-2008,codapca} as normalized histograms (with non-empty bins) as illustrated in Figure~\ref{fig:trinoulli}.

\begin{figure}
\centering
\includegraphics[width=\textwidth]{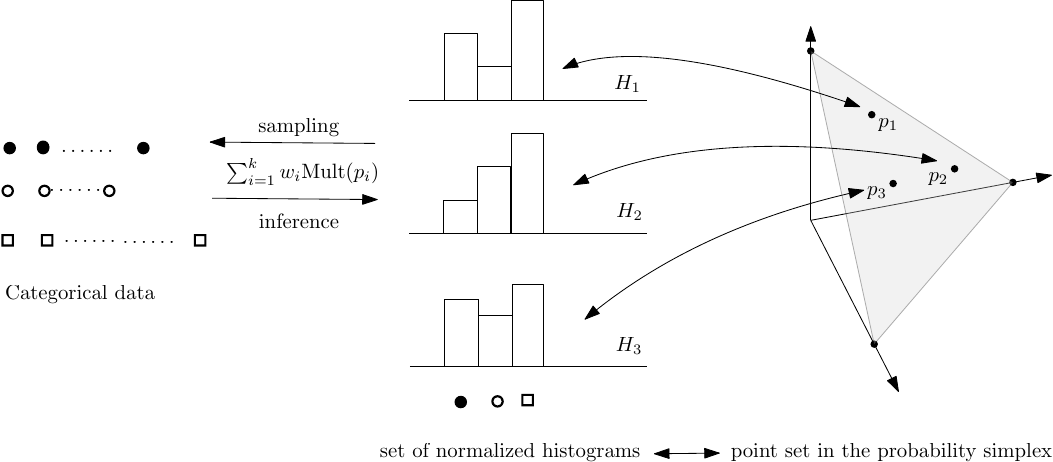}

\caption{Categorical datasets modeled by a generative statistical mixture model of multinoulli distributions can be visualized as a weighted set of normalized histograms or equivalently by a weighted point set encoding multinoulli distributions in the probability simplex $\Delta^d$ (here, $d=2$ for trinoulli distributions --- trinomial distributions with a single trial).}\label{fig:trinoulli}
\end{figure}

\begin{figure}[htp]
\centering
\begin{subfigure}{.9\textwidth}
\includegraphics[width=\textwidth]{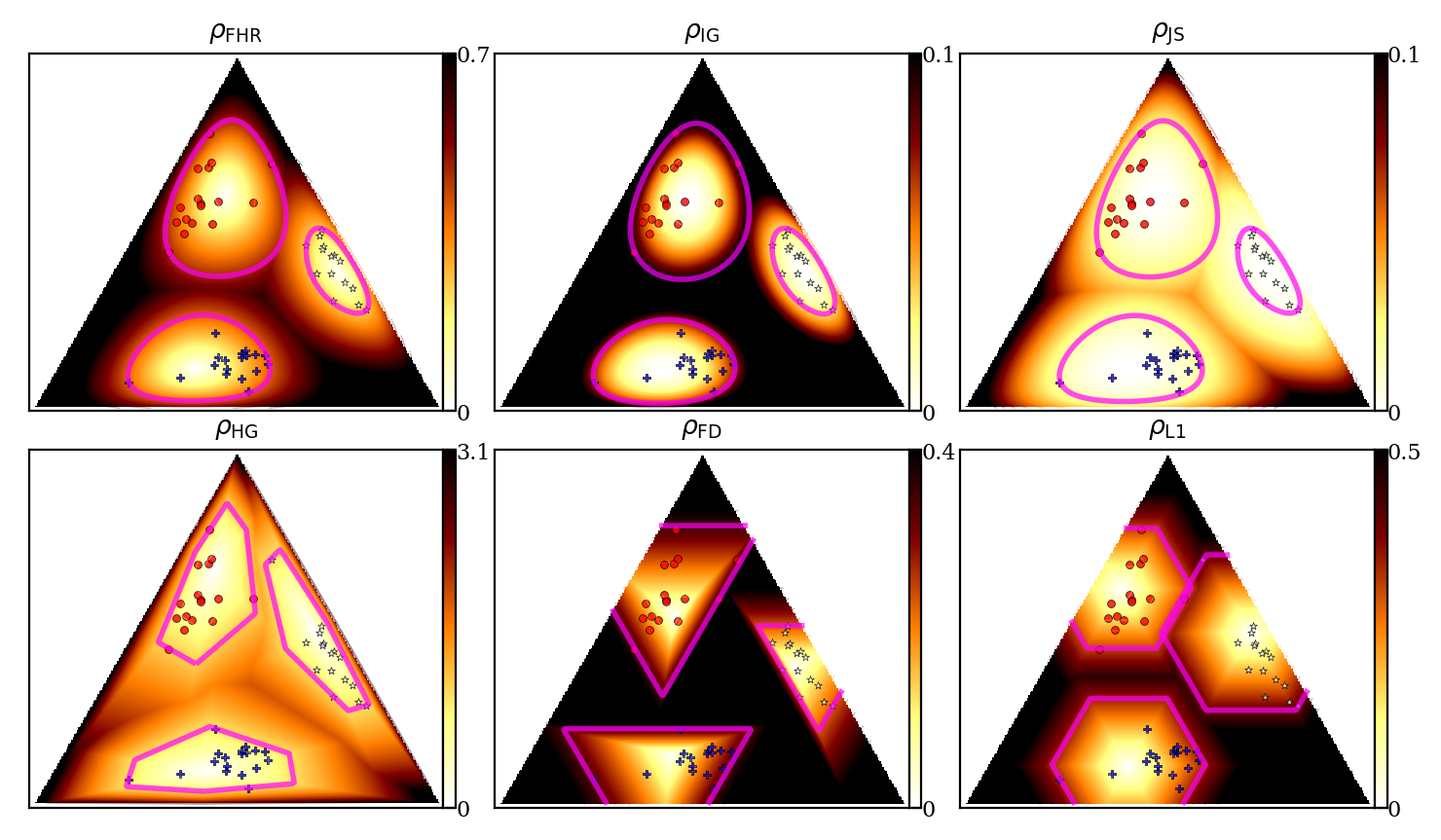}
\caption{$k=3$ clusters}
\end{subfigure}
\begin{subfigure}{.9\textwidth}
\includegraphics[width=\textwidth]{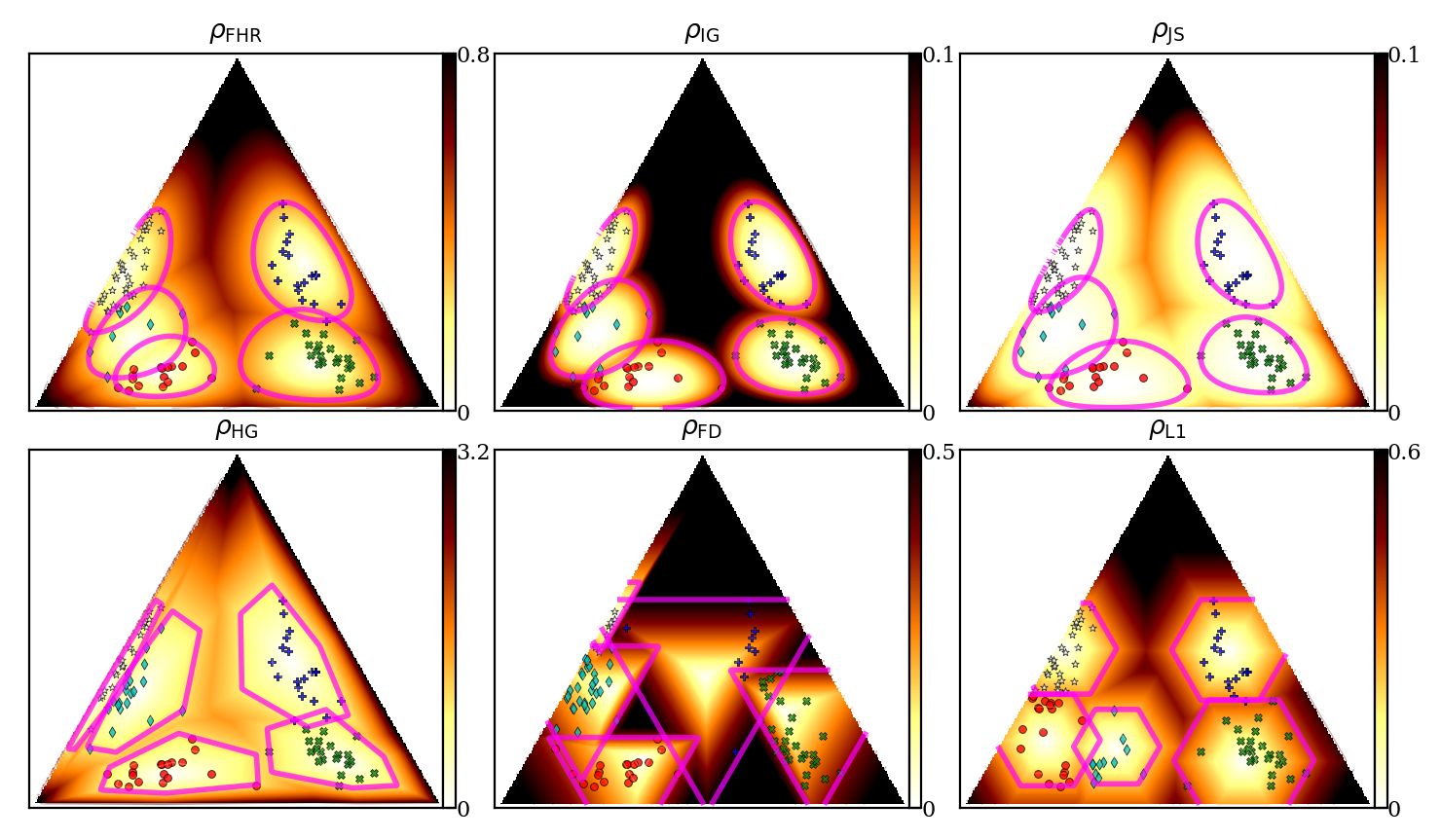}
\caption{$k=5$ clusters}
\end{subfigure}

\caption{Visualizing some $k$-center clustering results on a toy dataset in the space of trinomials $\Delta^2$ for the considered four types of distances (and underlying geometries): Fisher-Hotelling-Rao metric distance (Riemannian geometry), Kullback-Leibler non-metric divergence (information geometry), Hilbert metric distance (Hilbert projective geometry), and total variation/$L_1$ metric distance (norm geometry).
Observe that the $L_1$ balls have hexagonal shapes on the probability simplex (intersection of a rotated cube with the plane $H_{\Delta^d}$).
The color density maps indicate the distance from any point to its nearest cluster center.}\label{fig:results}
\end{figure}

A multinomial distribution  $p\in\Delta^d$ can be thought as a point lying in the probability simplex $\Delta^d$ (standard simplex) with coordinates $p=(\lambda_p^0,\ldots,\lambda_p^d)$ such that $\lambda_p^i=\Pr(X=e_i)>0$ and $\sum_{i=0}^d \lambda_p^i=1$.
The open probability simplex $\Delta^d$   can be embedded    in $\bbR^{d+1}$ on the hyperplane $H_{\Delta^d}:\;\sum_{i=0}^d x^i =1$.
Notice that observations with $D$ categorical attributes can be clustered using $k$-mode~\cite{kmode-1998} with respect to the Hamming distance.
Here, we consider the different task of clustering a set $\Lambda=\{p_1,\ldots,p_n\}$ of $n$ categorical/multinomial distributions
in $\Delta^d$~\cite{ClusteringSimplex-2008}
using center-based $k$-means++ or $k$-center clustering algorithms~\cite{kmeanspp-2007,kcenter-1985},
which rely on a dissimilarity measure (loosely called distance or divergence when smooth) between any two categorical distributions.
In this work, we mainly consider four distances with their underlying geometries:
(1) Fisher-Hotelling-Rao distance $\rho_\FHR$ (spherical geometry), (2) Kullback-Leibler divergence $\rho_\IG$ (dually flat geometry),  (3) Hilbert distance $\rho_\HG$ (generalize Klein's hyperbolic geometry), and (4) the total variation/L1 distance (norm geometry).
The geometric structures of spaces are necessary for algorithms, for example, to define midpoint distributions.
Figure~\ref{fig:results} displays the $k$-center clustering results obtained with these four geometries
as well as the  $L^1$ distance $\rho_{L1}$ normed geometry on toy synthetic datasets in $\Delta^2$.
We shall now explain the Hilbert simplex geometry applied to the probability simplex, describe
how to perform $k$-center clustering in Hilbert geometry, and report experimental results that
demonstrate the superiority of the Hilbert geometry when clustering multinomials and correlation matrices.

The rest of this paper is organized as follows:
Section~\ref{sec:distances} formally introduces the distance measures in $\Delta^d$.
Section~\ref{sec:dist} introduces how to efficiently compute the Hilbert distance, and prove that Hilbert simplex distance satisfies the information monotonicity (in \S\ref{sec:hsgim}).
Section~\ref{sec:clustering} presents algorithms for Hilbert minimax centers and Hilbert center-based clustering.
Section~\ref{sec:exp} performs an empirical study of clustering multinomial distributions,
comparing Riemannian geometry, information geometry, and Hilbert geometry.
Section~\ref{sec:elliptope} presents a second use case of Hilbert geometry in machine learning:
clustering correlation matrices in the elliptope~\cite{CorrelationElliptope-2018} (a kind of simplex with strictly convex facets).
Finally, section~\ref{sec:con} concludes this work by summarizing the pros and cons of each geometry.
Although some contents require prior knowledge of geometric structures,
we will present the detailed algorithms so that the general audience can still benefit from this work.

\section{Four distances with their underlying geometries}\label{sec:distances}

\subsection{Fisher-Hotelling-Rao Riemannian geometry}

The Rao distance between two multinomial distributions is~\cite{KassVos-1997,MetricLearning-2002}:
\begin{equation}
\rho_{\FHR}(p,q) = 2\arccos\left(\sum_{i=0}^{d} \sqrt{\lambda_p^i\lambda_q^i}\right).
\end{equation}
It is a Riemannian metric length distance (satisfying the symmetric and triangular inequality axioms)
obtained by setting the metric tensor $g$ to the {\em Fisher information matrix} (FIM)
$\mathcal{I}(p)=(g_{ij}(p))_{d\times{d}}$ with respect to the coordinate system $(\lambda_p^1,\cdots,\lambda_p^d)$,
where
$$
g_{ij}(p) = \frac{\delta_{ij}}{\lambda_p^i} + \frac{1}{\lambda_p^0}.
$$
We term this geometry the {\em Fisher-Hotelling-Rao (FHR) geometry}~\cite{Hotelling-1930,storyLM-2007,Rao-1945,Rao-reprint-1992}.
The metric tensor $g$ allows one to define an inner product on each tangent plane $T_p$ of the probability simplex manifold: $\Finner{u}{v}{p}=u^\top g(p) v$.
When $g$ is everywhere the identity matrix, we recover the Euclidean (Riemannian) geometry with the inner product being the scalar product: $\inner{u}{v}=u^\top v$.
The geodesics $\gamma(p,q;\alpha)$ are defined by the Levi-Civita metric connection~\cite{IG-2016,IG-2014} that is derived from the metric tensor.
The FHR manifold can be embedded in the positive orthant of a Euclidean $d$-sphere in $\bbR^{d+1}$
by using the {\em square root representation} $\lambda\mapsto\sqrt{\lambda}$~\cite{KassVos-1997}.
Therefore the FHR manifold modeling of $\Delta^d$ has constant {\em positive} curvature: It is a
spherical geometry restricted to the positive orthant with the metric distance measuring the arc length on a great circle.

\subsection{Information geometry}

A divergence $D$ is a smooth $C^3$ differentiable dissimilarity measure~\cite{DivIG-2010}
that allows defining a dual structure in Information Geometry~(IG), see~\cite{HessianStructure-2007,IG-2014,IG-2016,ElementaryIG-2018}.
An $f$-divergence is defined for a strictly convex function $f$ with $f(1)=0$ by:
$$
I_f(p:q)=\sum_{i=0}^{d} \lambda_p^i f\left(\frac{\lambda_q^i}{\lambda_p^i}\right)\geq f(1)=0.
$$
It is a {\em separable} divergence since the $d$-variate divergence can be written as a sum of $d$ univariate (scalar) divergences:
$I_f(p:q)=\sum_{i=0}^{d} I_f( \lambda_p^i: \lambda_q^i)$.
The class of $f$-divergences plays an essential role in information theory since they are provably the {\em only} separable divergences that satisfy
 the {\em information monotonicity} property~\cite{IG-2016,SeparableDivergence-2016} (for $d\geq 2$).
That is, by coarse-graining the histograms,
we obtain lower-dimensional multinomials, say $p'$ and $q'$, such that $0\leq I_f(p':q')\leq  I_f(p:q)$~\cite{IG-2016}.
The Kullback-Leibler (KL) divergence $\rho_{\IG}$ is a $f$-divergence obtained for the functional generator $f(u)=-\log u$:
\begin{equation}
\rho_{\IG}(p,q) = \sum_{i=0}^d \lambda_p^i \log\frac{\lambda_p^i}{\lambda_q^i}.
\end{equation}
It is an asymmetric non-metric distance: $\rho_{\IG}(p,q)\not =\rho_{\IG}(q,p)$.
In differential geometry, the structure of a manifold is defined by two independent components:
\begin{enumerate}
\item A {\em metric tensor} $g$ that allows defining an inner product $\Finner{\cdot}{\cdot}{p}$ at each tangent space (for measuring vector lengths and angles between vectors);
\item A {\em connection} $\nabla$ that defines
{\em parallel transport} ${\prod_c}^\nabla$, {\it i.e.}, a way to move a tangent vector from
one tangent plane $T_p$ to any other one $T_q$ along a smooth curve $c$, with $c(0)=p$ and $c(1)=q$.
\end{enumerate}
In FHR geometry, the implicitly-used connection is called the Levi-Civita connection that is induced by the metric $g$: $\nabla^{\mathrm{LC}}=\nabla(g)$.
It is a metric connection since it ensures that $\Finner{u}{v}{p}=\Finner{\prod_{c(t)}^{\nabla^\LC} u}{\prod_{c(t)}^{\nabla^\LC} v}{c(t)}$ for $t\in [0,1]$.
The underlying information-geometric structure of KL is characterized by a pair of \emph{dual} connections~\cite{IG-2016}
$\nabla=\nabla^{(-1)}$ (mixture connection) and $\nabla^*=\nabla^{(1)}$ (exponential connection) that induces a corresponding pair of dual geodesics
(technically, $\pm1$-autoparallel curves,~\cite{IG-2014}). Those connections are said to be \emph{flat}~\cite{ElementaryIG-2018}\ as they define two dual global affine
coordinate systems $\theta$ and $\eta$ on which the $\theta$- and $\eta$-geodesics are (Euclidean) straight line segments, respectively.
For multinomials, the {\em expectation parameters} are: $\eta = (\lambda^1,\ldots,\lambda^d)$
and they one-to-one correspond to the {\em natural parameters}:
$\theta = \left(\log\frac{\lambda^1}{\lambda^0},\ldots,\log\frac{\lambda^d}{\lambda^0}\right)\in\mathbb{R}^d$.
Thus in IG, we have two kinds of midpoint multinomials of $p$ and $q$, depending on whether we perform the (linear) interpolation
on the $\theta$- or the $\eta$-geodesics.
Informally speaking, the dual connections $\nabla^{(\pm 1)}$ are said coupled to the FIM since we have
$\frac{\nabla+\nabla^*}{2}=\nabla(g)=\nabla^{\LC}$. Those dual (torsion-free affine) connections are not metric connections but enjoy the following metric-compatibility property when used together as follows:
$\Finner{u}{v}{p} = \Finner{{\prod}_{c(t)} u}{{\prod^*}_{c(t)} v}{c(t)}$ (for $t\in [0,1]$), where $\prod:=\prod^{\nabla}$ and ${\prod^*}:={\prod^{\nabla^*}}$
are the corresponding induced dual parallel transports.
The geometry of $f$-divergences~\cite{DivIG-2010} is the $\alpha$-geometry (for $\alpha=3+2f'''(1)$) with the dual $\pm\alpha$-connections,
where $\nabla^{(\alpha)}=\frac{1+\alpha}{2}\nabla^*+\frac{1-\alpha}{2}\nabla$. The Levi-Civita metric connection is $\nabla^\LC=\nabla^{(0)}$.
More generally, it was shown how to build a dual information-geometric structure for {\em any} divergence~\cite{DivIG-2010}.
For example, we can build a dual structure from the symmetric Cauchy-Schwarz divergence~\cite{CS-2006}:
\begin{equation}
\rho_\CS(p,q)=- \log \frac{\inner{\lambda_p}{\lambda_q}}{\sqrt{\inner{\lambda_p}{\lambda_p}\inner{\lambda_q}{\lambda_q}}}.
\end{equation}

\subsection{Hilbert simplex geometry}

\begin{figure}%
\centering
\includegraphics[width=.4\textwidth]{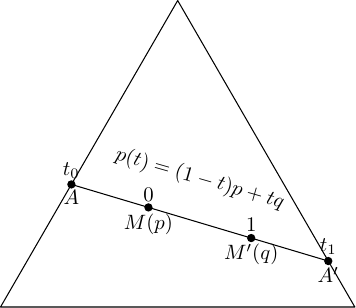}%

\caption{Computing the Hilbert metric distance for trinomials on the 2D probability simplex as the logarithm of the cross ratio $(M,M';A,A')$ of the four  collinear points $A, M, M'$ and $A'$.}%
\label{fig:hd}%
\end{figure}

In Hilbert geometry~(HG), we are given a bounded convex domain $\calC$ (here, $\calC=\Delta^d$),
and the distance between any two points $M$, $M'$ of $\calC$ is defined~\cite{Hilbert-1895} as follows:
Consider the two intersection points $AA'$ of the line $(MM')$ with  $\calC$, and order them on the line so that we have $A,M,M',A'$.
Then the Hilbert metric distance~\cite{Busemann-2011} is defined by:
\begin{equation}\label{eg:hgd}
\rhohg(M,M')=\left\{
\begin{array}{ll}
\left\vert\log\frac{|A'M| |AM'|}{|A'M'| |AM|}\right\vert, & M \not=M',\\
0 & M=M'.
\end{array}
\right.
\end{equation}
It is also called the Hilbert cross-ratio metric distance~\cite{HilbertHarpe-1991,BH-2014} or Cayley metric~\cite{Samelson-1957}.
Notice that we take the absolute value of the logarithm since the Hilbert distance is a {\em signed distance}~\cite{Richter-2011}.
When we order $A,M,M',A'$ along the line passing through them, the cross-ratio is positive.

When $\calC$ is the unit ball, HG lets us recover the Klein hyperbolic geometry~\cite{BH-2014}.
When $\calC$ is a quadric bounded convex domain, we obtain the Cayley-Klein hyperbolic geometry~\cite{CKM-2015}
which can be studied with the Riemannian structure and the corresponding metric distance called the curved Mahalanobis
distances~\cite{LMNN-2016,CayleyClassification-2016}.
Cayley-Klein hyperbolic geometries have negative curvature.
Elements on the boundary are called ideal elements~\cite{IdealHilbert-2014}.
We may scale the distance by a positive scalar
(e.g., $\frac{1}{2}$ when $\calC$ is the open ball so that the induced Hilbert geometry corresponds with the Klein model of hyperbolic geometry of negative unit curvature).

In Hilbert geometry, the geodesics are {\em straight} Euclidean lines making them convenient for computation.
Furthermore, the domain boundary $\partial\calC$ needs not to be smooth: One may also consider bounded polytopes~\cite{HGPolytope-2009}.
This is particularly interesting for modeling $\Delta^d$, the $d$-dimensional open standard simplex.
We call this geometry the \emph{Hilbert simplex geometry}~\cite{ClusteringHilbertSimplex-2017}.
In Figure~\ref{fig:hd}, we show that the Hilbert distance between two multinomial distributions
$p$ ($M$) and $q$ ($M'$) can be computed by
finding the two intersection points of the line $(1-t)p+tq$ with
$\partial\Delta^d$, denoted as $t_0\le0$ and $t_1\ge1$. Then
$$
\rhohg(p,q)=\left\vert\log\frac{(1-t_0)t_1}{(-t_0)(t_1-1)}\right\vert
=\log\left(1-\frac{1}{t_0}\right)-\log\left(1-\frac{1}{t_1}\right).
$$

It can be shown using Birkhoff geometry~\cite{BH-2014} that the Hilbert's distance between two points $p=(p^1,\ldots,p^d)$
and $q=(q^1,\ldots, q^d)$ on the $(d-1)$-dimensional probability simplex~\cite{bonnabel2011contraction} can be expressed as:
\begin{equation}
\rhohg(p, q)
=\log \frac{\max _{i\in\{1,\ldots, d\}} \frac{p_{i}}{q_{i}}}{\min _{j\in\{1,\ldots, d\}} \frac{p_{j}}{q_{j}}}.
\end{equation}

Figure~\ref{fig:VoronoiHilbert} displays the Voronoi diagram of $n=16$ trinomials with respect to the Hilbert distance.
The Hilbert Voronoi diagram is piecewise linear and amounts to a polytopal-norm Voronoi diagram~\cite{boissonnat1998voronoi} with respect to the variation norm.

\begin{figure}%
\centering%
\includegraphics[width=.35\textwidth]{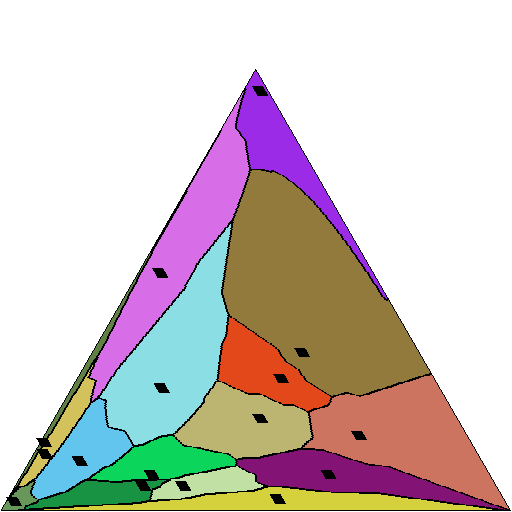}%

\caption{Hilbert Voronoi diagram of $n=16$ trinomial sites displayed in the equilateral simplex embedded in $\bbR^2$.}%
\label{fig:VoronoiHilbert}%
\end{figure}

\begin{figure}%
\centering%
\includegraphics[width=.7\textwidth]{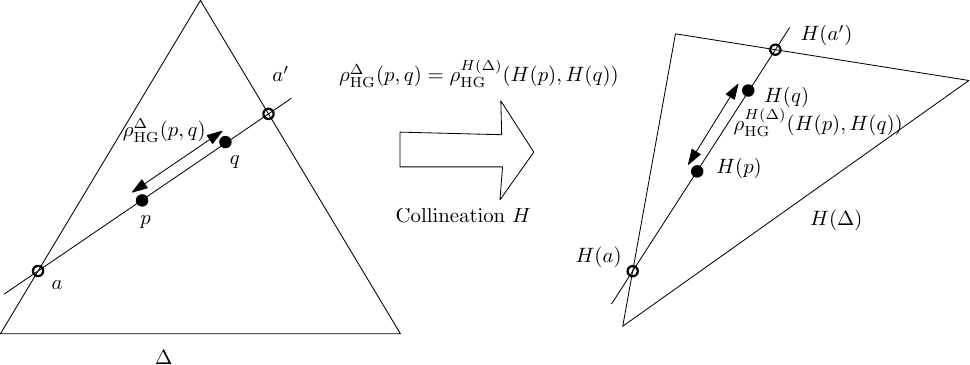}%

\caption{The Hilbert distance is invariant by construction to collineations.}%
\label{fig:invhd}%
\end{figure}

By construction, Hilbert distance is {\em invariant to collineations} $H$ (also called projectivities which are equivalent to homographies in real projective spaces). That is, let $\Delta$ denote a simplex and $H(\Delta)$ the deformed simplex by a collineation $H$.
Then $\rhohg^\Delta(p,q)=\rhohg^{H(\Delta)}(H(p),H(q))$, as depicted in Figure~\ref{fig:invhd}.

\begin{figure}
\centering
\begin{subfigure}{.98\textwidth}
\includegraphics[width=\textwidth]{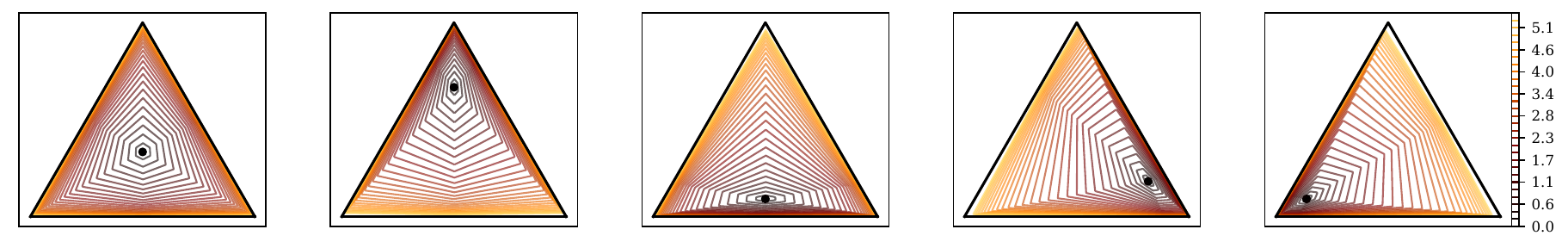}
\caption{$\rhohg(p,c)$}
\end{subfigure}
\begin{subfigure}{.98\textwidth}
\includegraphics[width=\textwidth]{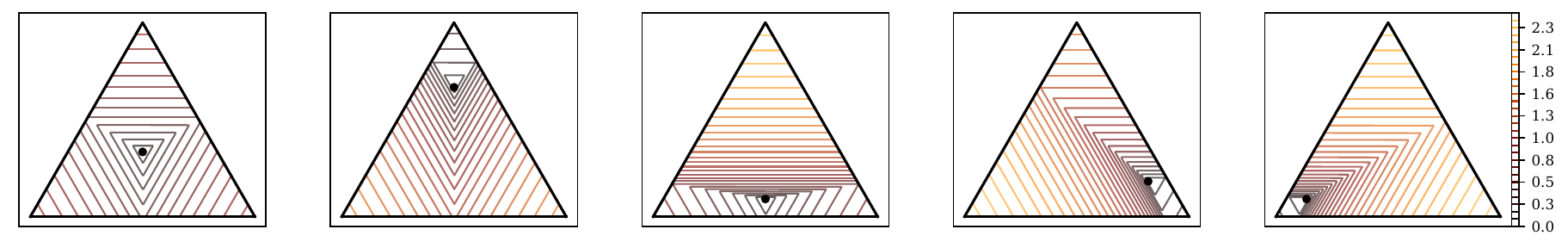}
\caption{$\rhofd(p,c)$}
\end{subfigure}
\begin{subfigure}{.98\textwidth}
\includegraphics[width=\textwidth]{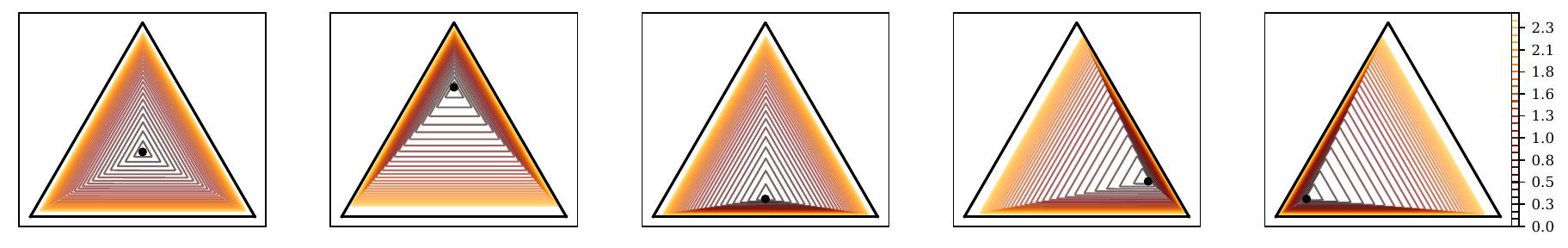}
\caption{$\rhofd(c,p)$}
\end{subfigure}

\caption{
    Balls centered at $c\in\Delta^2$ with its radius increasing at constant speed.
    Balls in the Hilbert simplex geometry $\Delta^2$
    have polygonal Euclidean shapes of constant combinatorial complexity.
At infinitesimal scale, the balls have polygonal shapes, showing that the Hilbert geometry is not Riemannian.}%
\label{fig:shape}%
\end{figure}

\begin{figure}
\centering
\includegraphics[width=.25\textwidth]{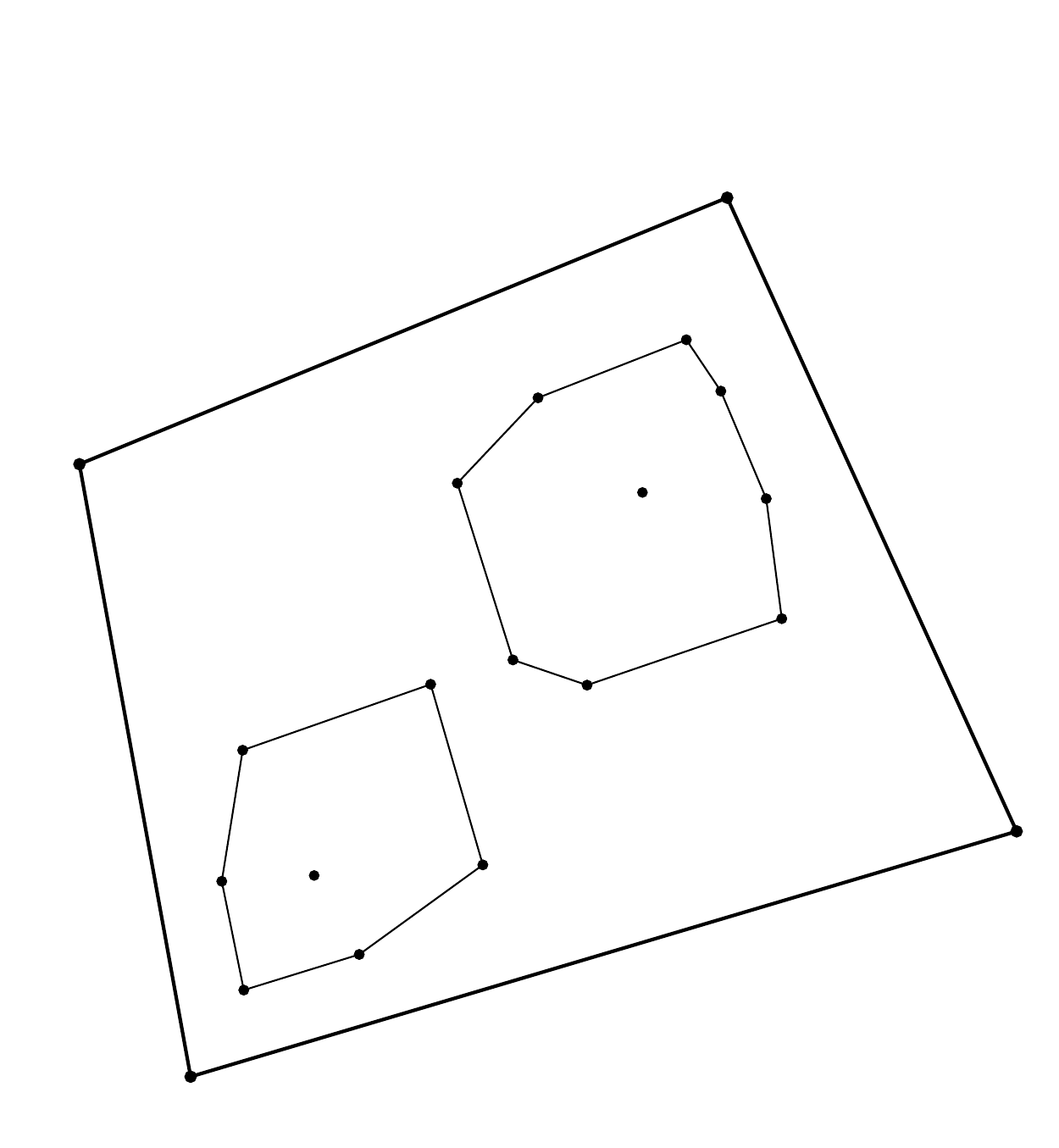}\hskip 3cm
\includegraphics[width=.25\textwidth]{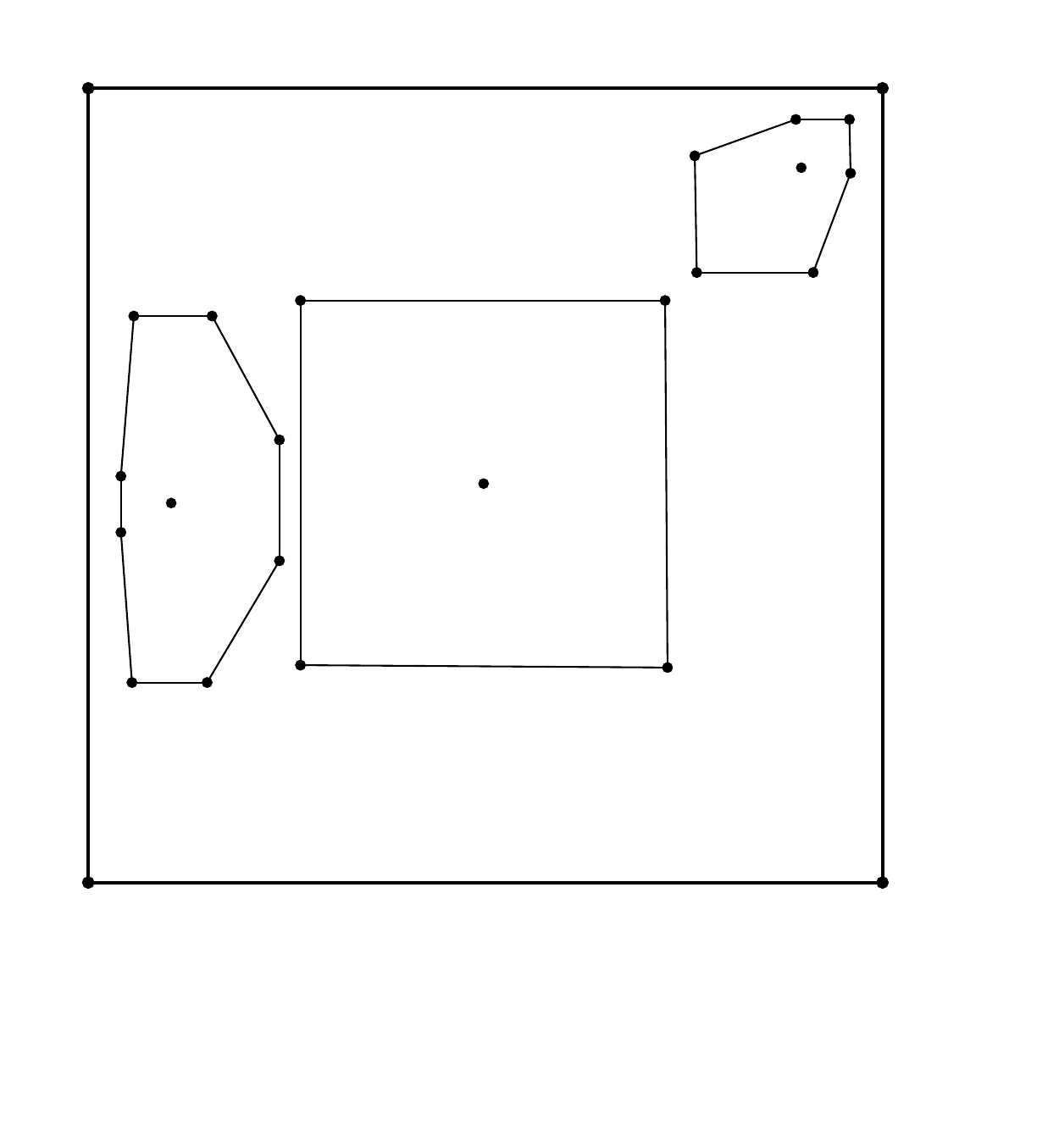}
\caption{Hilbert balls in quadrangle domains have combinatorial complexity depending on the center location.\label{fig:quad}}
\end{figure}

The shape of balls in polytope-domain HG is Euclidean polytopes\footnote{To contrast with this result, let us mention that infinitesimal  small balls in Riemannian geometry have Euclidean ellipsoidal shapes (visualized as Tissot's indicatrix in cartography).}~\cite{BH-2014},
as depicted in Figure~\ref{fig:shape}.
Furthermore, the Euclidean shape of the balls does not change with the radius.
Hilbert balls have hexagons shapes in 2D~\cite{Hilbert2DSimplex-1975,Pambuccian-2008,HG-SoCG-2017}, rhombic dodecahedra shapes in 3D, and are polytopes~\cite{BH-2014} with $d(d+1)$ facets in dimension $d$.
When the polytope domain is not a simplex, the combinatorial complexity of balls depends on the center location~\cite{HG-SoCG-2017}, see Figure~\ref{fig:quad}.
The HG of the probability simplex yields a non-Riemannian geometry, because, at an infinitesimal radius, the balls are polytopes and not ellipsoids
(corresponding to squared Mahalanobis distance balls used to visualize metric tensors \cite{VizTensor-2009}).
The isometries in Hilbert polyhedral geometries are studied in~\cite{HilbertIsometry-2011}.
In Appendix~\ref{sec:HGFG}, we recall that any Hilbert geometry induces a Finslerian structure
that becomes Riemannian iff the boundary is an ellipsoid (yielding the hyperbolic Cayley-Klein geometries~\cite{Richter-2011}).
Notice that in Hilbert simplex/polytope geometry, the geodesics are not unique (see Figure~2 of~\cite{HilbertHarpe-1991} and Figure~\ref{fig:notgeodesic}).
Observe that the Funk balls $\mathrm{ball}_\FD(c,r)=\{p\in\Delta^d\ :\ \rho_\FD(p,c)\leq r\}$ and 
reverse Funk balls $\mathrm{ball}_\FD^*(c,r)=\{p\in\Delta^d\ :\ \rho_\FD(c,p)\leq r\}$  have different shapes in  
Figure~\ref{fig:shape}. This is because for Funk balls $\mathrm{ball}_\FD(c,r)$, we have $\max_i\log \frac{p_i}{c_i}<\infty$ even when $p\in\partial\Delta^d$ but for reverse Funk balls $\mathrm{ball}_\FD^*(c,r)$, we have $\max_i\log \frac{c_i}{p_i}=+\infty$ when $p\in \partial\Delta^d$.

\subsection{$L_1$-norm geometry}

The Total Variation (TV) metric distance between two multinomials $p$ and $q$ is defined by:
$$
\mathrm{TV}(p,q) = \frac{1}{2} \sum_{i=0}^d |\lambda_p^i-\lambda_q^i|.
$$
It is a statistical $f$-divergence obtained for the generator $f(u)=\frac{1}{2}|u-1|$.
The $L_1$-norm induced distance $\rho_{L1}$ (L1) is defined by:
$$
\rho_{L1}(p,q) = \|\lambda_p-\lambda_q\|_1  =  \sum_{i=0}^d |\lambda_p^i-\lambda_q^i| = 2 \TV(p,q).
$$
Therefore the distance $\rho_{L1}$ satisfies information monotonicity (for coarse-grained histograms $p'$ and $q'$ of $\Delta^{D'}$ with $D'<D$):
$$
0\leq \rho_{L1}(p',q')\leq  \rho_{L1}(p,q).
$$

For trinomials, the $\rho_{L1}$ distance is given by:
$$
\rho_{L1}(p,q) = |\lambda_p^0-\lambda_q^0| + |\lambda_p^1-\lambda_q^1| + |\lambda_q^0-\lambda_p^0+\lambda_q^1-\lambda_p^1|.
$$


The $L_1$ distance function is a polytopal distance function described by the dual polytope $\mathcal{Z}$ of the $d$-dimensional cube called the standard (or regular) $d$-cross-polytope~\cite{L1projection-2016},  the orthoplex~\cite{Orthoplex-2016} or the $d$-cocube~\cite{L1Voronoi-1998}:
The cross-polytope $\mathcal{Z}$ can be obtained as the convex hull of the $2d$ unit standard base vectors $\pm e_i$ for $i\in\{0,\ldots, d-1\}$.
The cross-polytope is one of the three regular polytopes in dimension $d\geq 4$ (with the hypercubes and simplices):
It has $2d$ vertices and $2^d$ facets.
Therefore an $L_1$ ball on the hyperplane $H_{\Delta^d}$ supporting the probability simplex is the intersection of a $(d+1)$-cross-polytope with $d$-dimensional hyperplane $H_{\Delta^d}$.
Thus the ``multinomial ball'' $\Ball_{L_1}(p,r)$ of center $p$ and radius $r$ is defined by
 $\Ball_{L_1}(p,r) = (\lambda_p\oplus r\mathcal{Z})\cap H_{\Delta^d}$.
In 2D, the shape of $L_1$ trinomial balls is that of a
 regular octahedron  (twelve edges and eight faces) cut by the 2D plane $H_{\Delta^2}$:
Trinomial balls have hexagonal shapes as illustrated in Figure~\ref{fig:results} (for $\rho_{L1}$).
In 3D,   trinomial balls are Archimedean solid cuboctahedra, and in arbitrary dimension, the shapes are polytopes with $d(d+1)$ vertices~\cite{QG-2017}.
Let us note in passing, that in 3D, the $L_1$ multinomial cuboctahedron ball has the dual shape of the Hilbert rhombic dodecahedron ball.


\begin{table*}[t]
\centering
\caption{Comparing the  geometric modelings of the probability simplex $\Delta^d$.}\label{tab:geo}

{\small
\begin{tabular}{llll}
\toprule[1.5pt]
 & {\bf{Riemannian Geometry}} & {\bf{Information Rie. Geo.}} & {\bf{Non-Rie. Hilbert Geo.}}\\\hline
Structure & $(\Delta^d,g,\nabla^\LC=\nabla(g))$ & $$ $(\Delta^d,g,\nabla^{(\alpha)},\nabla^{(-\alpha)})$ & $(\Delta^d,\rho)$\\  
 & Levi-Civita $\nabla^\LC=\nabla^{(0)}$  & dual connections $\nabla^{(\pm\alpha)}$ so & connection of $\bbR^d$\\
 & &  that $\frac{\nabla^{(\alpha)} + \nabla^{(-\alpha)}}{2}=\nabla^{(0)}$  & \\
Distance & Rao distance (metric) & $\alpha$-divergence (non-metric) & Hilbert distance (metric)\\
 & & KL or reverse KL for $\alpha=\pm1$\\
Property & invariant to reparameterization & information monotonicity & isometric to a normed space \\
Calculation &
closed-form & closed-form & easy (Alg. 1)\\
Geodesic & shortest path & straight either in $\theta/\eta$ & straight\\
Smoothness & manifold  & manifold & non-manifold\\
Curvature & positive & dually flat & negative\\
\bottomrule[1.5pt]
\end{tabular}}
\end{table*}

Table~\ref{tab:geo} summarizes the characteristics of the three main  geometries: FHR, IG, and HG.
Let us conclude this introduction by mentioning the Cram\'er-Rao lower bound and its relationship with information geometry~\cite{CRLB-2013}:
Consider an unbiased estimator $\hat\theta=T(X)$ of a parameter $\theta$ estimated from
measurements distributed according to a smooth density $p(x;\,\theta)$ (i.e., $X\sim{p}(x;\,\theta)$).
The Cram\'er-Rao Lower Bound (CRLB) states that the variance of $T(X)$ is greater or equal to the inverse of the FIM $\calI(\theta)$:
$V_\theta[T(X)]\succ\calI^{-1}(\theta)$. For regular parametric families $\{p(x;\theta)\}_\theta$, the FIM is a positive-definite matrix and defines a metric tensor, called the Fisher metric in Riemannian geometry. The FIM is the cornerstone of information geometry~\cite{IG-2016} but requires the differentiability of the probability density function (pdf).

A better lower bound that does not require the pdf differentiability is the Hammersley-Chapman-Robbins Lower Bound~\cite{Hammersley-1950,ChapmanRobbins-1951} (HCRLB):
\begin{equation}
V_\theta[T(X)]\geq \sup_\Delta \frac{\Delta^2}{E_\theta\left[\left(\frac{p(x;\theta+\Delta)-p(x;\theta)}{p(x;\theta)}\right)^2\right]}.
\end{equation}
By introducing the $\chi^2$-divergence, $\chi^2(P:Q)=\int \left(\frac{\dP-\dQ}{\dQ}\right)^2 \dQ$, we rewrite the HCRLB
using the  $\chi^2$-divergence in the denominator as follows:
\begin{equation}
V_\theta[T(X)]\geq \sup_\Delta \frac{\Delta^2}{\chi^2(P(x;\theta+\Delta):P(x;\theta))}.
\end{equation}
Note that the FIM is not defined for non-differentiable pdfs, and therefore the Cram\'er-Rao
lower bound does not exist in that case.

A multinomial distribution $\lambda_p$ may be visualized in four different ways:
\begin{itemize}
\item using the natural coordinates $\theta^i_p=\log\frac{\lambda_p^i}{\lambda_p^0}$ for $i\in\{1,\ldots, d\}$ of information geometry ( the set of multinomials being interpreted as an exponential family),

\item using the dual moment parameters $\eta^i_p=\lambda_p^i$ for $i\in\{1,\ldots, d\}$,

\item  using the coordinates $\lambda_p^i$ for
$i\in\{0,\ldots, d\}$ describing the probability simplex embedded in $\bbR^{d+1}$, or

\item  $\sum_{i=0}^d \lambda_p^i v_i$ (probability simplex embedded as an equilateral simplex\footnote{An equilateral simplex is defined as a simplex with pairwise unit distance vertices.} in $\bbR^{d}$) where the $v_i$'s are vertices such that $\|v_i-v_j\|_2=1$ iff. $i\not=j$, and $0$ otherwise. In this last case, the coordinates $\lambda_p^i$ are called the barycentric coordinates.
For example, for trinomials, we can choose $v_1=(0,0)$, $v_2=(\frac{1}{2},\frac{\sqrt{3}}{2})$ and $v_3=(0,1)$.
\end{itemize}

\section{Computing Hilbert distance in $\Delta^d$}\label{sec:dist}

Let us start with the simplest case: The 1D probability simplex $\Delta^1$,
the space of Bernoulli distributions.
Any Bernoulli distribution can be represented by the activation
probability of the random bit $x$: $\lambda=p(x=1)\in\Delta^1$,
corresponding to a point in the interval $\Delta^1=(0,1)$. We
write the Bernoulli manifold as an exponential family as
$$
p(x) = \exp\left(x \theta - F(\theta) \right),\quad{}x\in\{0,\,1\},
$$
where $F(\theta)=\log(1+\exp(\theta))$. Therefore
$\lambda=\frac{\exp(\theta)}{1+\exp(\theta)}$ and $\theta=\log\frac{\lambda}{1-\lambda}$.

\subsection{1D probability simplex of Bernoulli distributions}

By definition, the Hilbert distance has the closed form:
$$
\rhohg(p,q)
= \left\vert\log\frac{\lambda_q(1-\lambda_p)}{\lambda_p(1-\lambda_q)}\right\vert
= \left\vert\log\frac{\lambda_p}{1-\lambda_p}-\log\frac{\lambda_q}{1-\lambda_q}\right\vert.
$$
Note that $\theta_p=\log\frac{\lambda_p}{1-\lambda_p}$ is the canonical parameter of the Bernoulli distribution.

The FIM of the Bernoulli manifold in the $\lambda$-coordinates is given by:
$g = \frac{1}{\lambda} + \frac{1}{1-\lambda} = \frac{1}{\lambda(1-\lambda)}$.
The FHR distance is obtained by integration as:
$$
\rhofhr(p,q) = 2 \arccos\left(\sqrt{\lambda_p\lambda_q}+\sqrt{(1-\lambda_p)(1-\lambda_q)}\right).
$$
Notice that $\rhofhr(p,q)$ has finite values on $\partial\Delta^1$.

The KL divergence of the $\pm1$-geometry is:
$$
\rho_{\mathrm{IG}}(p,q) = \lambda_p\log\frac{\lambda_p}{\lambda_q} + (1-\lambda_p)\log\frac{1-\lambda_p}{1-\lambda_q}.
$$
The KL divergence belongs to the family of $\alpha$-divergences~\cite{IG-2016}.

\subsection{Arbitrary dimension case using Birkhoff's cone metric}

Instead of considering the Hilbert simplex metric on the $(d-1)$-dimensional simplex $\Delta_d$, we consider the
equivalent Hilbert simplex {\em projective} metric $\rho_{BG}$ defined on the cone $\bbR_{+,*}^d$.
We call $\rho_{\BG}$ the {\em Birkhoff metric}~\cite{Birkhoff-1957,Nussbaum-1988} since it measures the distance between
any two cone rays, say $\tilde{p}$ and $\tilde{q}$ of $\bbR_{+,*}^d$.
The Birkhoff metric~\cite{Birkhoff-1957} is defined by:
\begin{equation}\label{eq:bgposcone}
\rho_{\BG}(\tilde{p},\tilde{q}) \eqdef \log\max_{i,j\in [d]} \frac{\tilde{p}_i\tilde{q}_j}{\tilde{p}_j\tilde{q}_i},
\end{equation}
and is scale-invariant:
\begin{equation}
\rho_{\BG}(\alpha\tilde{p},\beta\tilde{q})=\rho_{\BG}(\tilde{p},\tilde{q}),
\end{equation}
for any $\alpha,\beta>0$.
We have $\rho_{BG}(\tilde{p},\tilde{q})=0$ if and only if $\tilde{p}=\lambda\tilde{q}$ for $\lambda>0$.
Thus we define {\em equivalence classes} $\tilde{p}\sim \tilde{q}$ when there exists $\lambda>0$ such that $\tilde{p}=\lambda\tilde{q}$.
The Birkhoff metric satisfies the triangle inequality and is therefore a {\em projective metric}~\cite{BusemannKelly}.
For a ray $\tilde{r}$ in $\bbR_{+,*}^d$, let $r$ denote the intersection of $\tilde{r}$ with the standard simplex $\Delta_d$.
Inhomogeneous vector $r\in\Delta_d$ is obtained by dehomogeneizing the ray vector so that $\sum_{i=1}^d r_i=1$.
For a ray $\tilde{r}$, let $\underline{\tilde{r}}=r$ denote the dehomogeneization.
Then we have:
\begin{equation}
\rho_{\BG}(\tilde{p},\tilde{q}) = \rho_{\HG}(p,q).
\end{equation}
The Birkhoff metric can be calculated using an equivalent norm-induced distance on the logarithmic representation:
\begin{equation}\label{eq:BG}
\rho_{\BG}(\tilde{p},\tilde{q}) = \|\log(\tilde{p}) - \log(\tilde{q}) \|_{\mathrm{var}},
\end{equation}
where $\|x\|_{\mathrm{var}}\eqdef (\max_i x_i)-(\min_i x_i)$ is the {\em variation norm}, and $\log(x)=(\log x_1,\ldots,\log x_d)$.
Notice that computing $\rho_{\BG}(p,q)= \|\log({p}) - \log({q}) \|_{\mathrm{var}}$ yields a simple linear-time algorithm to compute the Hilbert simplex distance.
One does not need to apply Algorithm~\ref{alg:distance} in~\cite{gsi-nielsen-2019} to compute $\rho_{\HG}$
but only need to use this closed formula, which is much simpler to compute with the same
complexity $O(d)$.

Figure~\ref{fig:HilbertBirkhoff} illustrates the relationship between the Hilbert simplex metric and the Birkhoff projective cone metric.
The Birkhoff projective cone metric allows one to define Hilbert cross-ratio metric on open unbounded convex domain as well~\cite{BH-2014}.

\begin{figure}
\centering

\includegraphics[width=0.5\textwidth]{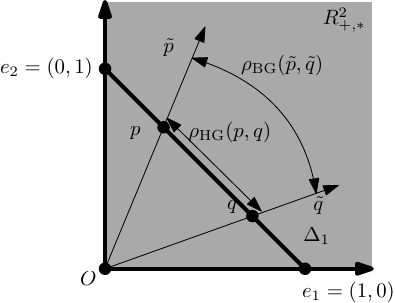}

\caption{Illustration of the equivalence between the Hilbert simplex metric $\rho_\HG$ and the Birkhoff  projective metric $\rho_\BG$ defined over the positive orthant.}\label{fig:HilbertBirkhoff}
\end{figure}

Once an arbitrary distance $\rho$ is chosen, we can define a ball centered
at $c$ and of radius $r$ as $B_\rho(c,r)=\{x\ :\ \rho(c,x)\leq r\}$.
Figure~\ref{fig:shape} displays the hexagonal shapes of the Hilbert balls for various center locations in $\Delta^2$.

\begin{theorem}[Balls in a simplicial Hilbert geometry~\cite{BH-2014}]
A ball in the Hilbert simplex geometry has a Euclidean polytope shape with $d(d+1)$ facets.
\end{theorem}
Note that when the domain is not simplicial, the Hilbert balls can have varying combinatorial complexity depending on the center location.
In 2D, the Hilbert ball can have $s\sim{}2s$ edges inclusively, where $s$ is the number of edges of the boundary of the Hilbert domain $\partial\calC$.

Since a Riemannian geometry is locally defined by a metric tensor,
at infinitesimal scales, Riemannian balls have Mahalanobis smooth ellipsoidal shapes:
$B_\rho(c,r)=\{x\ :\ (x-c)^\top g(c) (x-c)\leq r^2\}$.
This property allows one to visualize Riemannian metric tensors~\cite{VizTensor-2009}. Thus we conclude that:
\begin{lemma}[\cite{BH-2014}]
Hilbert simplex geometry is a non-manifold metric length space.
\end{lemma}

As a remark, let us notice that slicing a simplex with a hyperplane does not always produce a lower-dimensional simplex.
For example, slicing a tetrahedron by a plane yields either a triangle or a quadrilateral.
Thus the restriction of a $d$-dimensional ball $B$ in a Hilbert simplex geometry $\Delta^d$ to a hyperplane $H$ is a $(d-1)$-dimensional ball $B'=B\cap H$  of varying combinatorial complexity, corresponding to a ball in the induced Hilbert sub-geometry in the convex sub-domain $H\cap\Delta^d$.

\subsection{Hilbert simplex distance satisfies the information monotonicity \label{sec:hsgim}}

Let us prove that the Hilbert simplex metric satisfies the property of information monotonicity.
Consider a bounded linear positive operator  $M:\bbR_+^d \rightarrow \bbR_+^{d'}$ with $d'<d$ and $M_{ij}>0$ which encodes the coarse-binning scheme of the standard simplex $\Delta_d$ into the standard simplex $\Delta_{d'}$.
A $d$-dimensional positive measure $\tilde{r}$ is mapped into a $d'$-dimensional positive measure $\tilde{p}'$
such that $\tilde{p}'=M\tilde{p}$.
Matrix $M$ is a $d'\times d$ (fat) positive matrix that is column-stochastic (i.e., elements of each column sum up to one).

Birhkoff~\cite{Birkhoff-1957} and Samelson~\cite{Samelson-1957} independently proved (1957) that:
\begin{equation}
\rho_{\BG}(M\tilde{p},M\tilde{q}) \leq \kappa(M)\rho_{\BG}(\tilde{p},\tilde{q}),
\end{equation}
where $\kappa(M)$ is the {\em contraction ratio} of the binning scheme.
Furthermore, it can be shown that:
\begin{equation}
\kappa(M) \eqdef \frac{\sqrt{a(M)}-1}{\sqrt{a(M)}+1}<1,
\end{equation}
with
\begin{equation}
a(M)\eqdef \max_{i,j,k,l} \frac{M_{i,k}M_{j,l}}{M_{j,k}M_{i,l}}<\infty.
\end{equation}

Thus coarse-binning the probability simplex $\Delta_d$ to $\Delta_{d'}$ makes a {\em strict contraction} of the distance between distinct elements.

\begin{theorem}\label{thm:hsgim}
The Hilbert simplex metric is a non-separable distance satisfying the property of information monotonicity.
\end{theorem}

$\rho_{\IG}$ is invariant by permutations (i.e., a separable divergence) but not $\rho_{\HG}$  (i.e., a non-separable distance).
Notice that this proof\footnote{Birkhoff's proof~\cite{Birkhoff-1957} is more general and works for any projective distance defined over a cone (so-called Hilbert projective distances) by defining the notion of the projective diameter of a positive linear operator.} is based on a contraction theorem of a bounded linear positive operator,
and that therefore the information monotonicity property may be rewritten equivalently as a {\em contraction} property.

In information geometry, it is known that the class of separable divergences satisfying the information monotonicity are $f$-divergences
 when $d>2$, see~\cite{IG-2016} (see~\cite{Jiao-2014} for the special case $d=2$ corresponding to binary alphabets).
It is an open problem to fully characterize the class of {\em non-separable distances} that fulfills the property of information monotonicity.
For example, the Aitchison distance~\cite{Aitchison-1982,Aitchison-2006} between $p$ and $q$ is often used in Compositional Data Analysis~\cite{pawlowsky2011compositional} (CoDA).
The  Aitchison distance is also a non-separable distance in the probability simplex defined as follows:
\begin{equation}\label{eq:AD}
\rho_{\mathrm{Aitchison}}(p,q)=\sqrt{\sum_{i=1}^d \left( \log \frac{p^i}{G(p)}-\log \frac{q^i}{G(q)} \right)^2},
\end{equation}
where $G(p)$ denotes the geometric mean of the coordinates of $p\in\Delta_d$:
$$
G(p)=\left(\prod_{i=1}^d p^i\right)^{\frac{1}{d}}=\exp(\frac{1}{d}\sum_{i=1}^d \log p^i).
$$
The Aitchison distance satisfies the monotonicity property~\cite{CODA-2021}.
The Fisher-Rao distance or the Cauchy-Schwarz divergence are non-separable distances which do not satisfy the monotonicity property.

Information geometry does not consider any {\em ground metric} on the support (i.e., there is no neighborhood structure on the sample space~\cite{AmariOT-2017}).
This is to constrast with  Optimal Transport~\cite{Villani-2008} (OT) that necessarily considers a ground metric.

\subsection{Visualizing distance profiles}
Figure~\ref{vis:distance} displays the distance profile from any point in the probability simplex
to a fixed reference point (trinomial) based on the following common distance measures~\cite{IG-2014}:
Euclidean (metric) distance, Cauchy-Schwarz (CS) divergence,
Hellinger (metric) distance, Fisher-Rao (metric) distance,
KL divergence, and Hilbert simplicial (metric) distance.
The Euclidean and Cauchy-Schwarz divergence are clipped to $\Delta^2$.
The Cauchy-Schwarz distance is projective so that
$\rho_\CS(\lambda p,\lambda' q)=\rho_\CS(p,q)$ for any $\lambda,\lambda'>0$~\cite{holder}.

\begin{figure}
\centering
\begin{subfigure}{\textwidth}
\includegraphics[width=\textwidth]{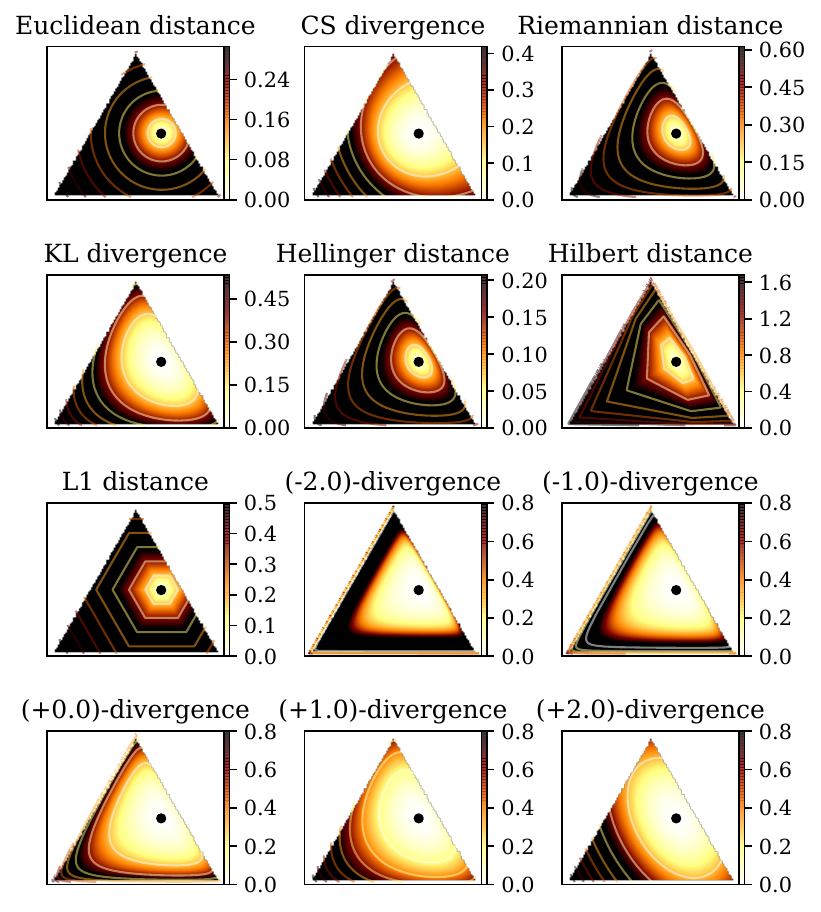}
\caption{Reference point (3/7,3/7,1/7)}
\end{subfigure}
\end{figure}

\begin{figure}\ContinuedFloat
\centering
\begin{subfigure}{\textwidth}
\includegraphics[width=\textwidth]{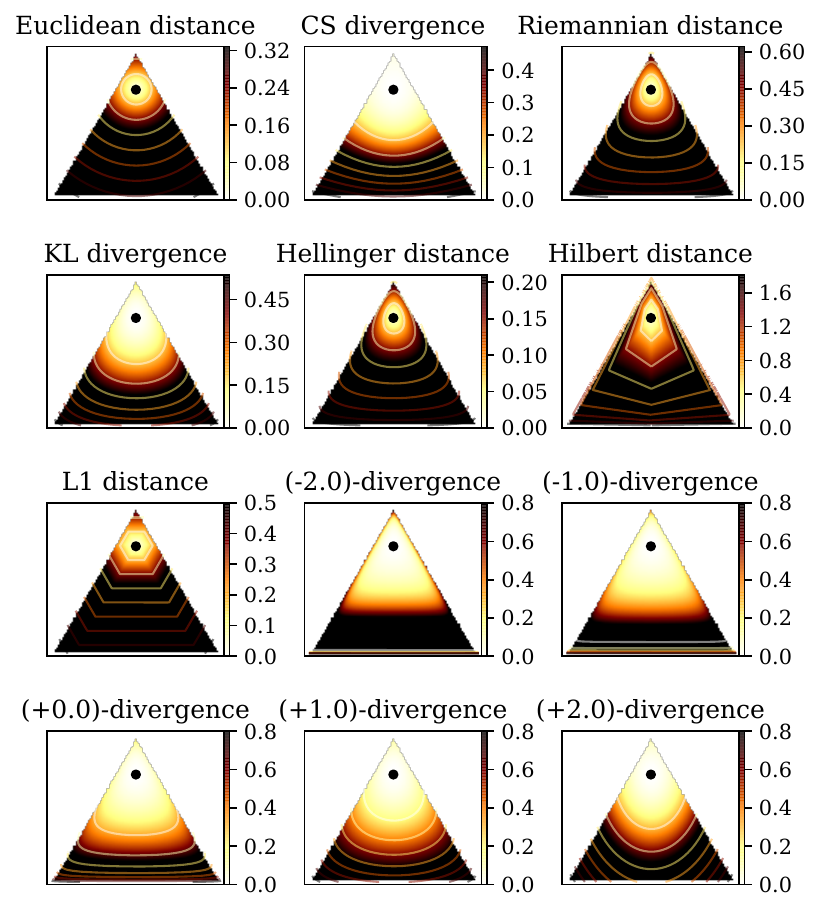}
\caption{Reference point (5/7,1/7,1/7)}
\end{subfigure}
\caption{A comparison of different distance measures on $\Delta^2$. The distance is measured from
$\forall{p}\in\Delta^2$ to a fixed reference point (the black dot). Lighter color means a shorter
distance. Darker color means longer distance. The contours show equal distance curves with a precision step of $0.2$.\label{vis:distance}}
\end{figure}

\section{Center-based clustering}\label{sec:clustering}

We concentrate on comparing the efficiency of Hilbert simplex geometry for clustering multinomials.
We shall compare the experimental results of $k$-means++ and $k$-center multinomial clustering for
the three distances: Rao and Hilbert metric distances, and KL divergence.
We describe how to adapt those clustering algorithms to the Hilbert distance.

\subsection{$k$-means++ clustering}

The celebrated $k$-means clustering~\cite{nn-kmeans-2014} minimizes the sum of within-cluster variances,
where each cluster has a center representative element.
When dealing with $k=1$ cluster, the center (also called centroid or cluster prototype)
is the center of mass defined as the minimizer of
$$
E_D(\Lambda,c) = \frac{1}{n} \sum_{i=1}^n D(p_i:c),
$$
where $D(\cdot:\cdot)$ is a dissimilarity measure.
For an arbitrary $D$, the centroid $c$ may not be available in closed form.
Nevertheless, using a generalization of the $k$-means++ initialization~\cite{kmeanspp-2007}
(picking randomly seeds), one can bypass the centroid computation,
and yet guarantee probabilistically a good clustering.

Let $C=\{c_1,\ldots, c_k\}$ denote the set of $k$ cluster centers.
Then the generalized $k$-means energy to be minimized is defined by:
$$
E_D(\Lambda,C) = \frac{1}{n}\sum_{i=1}^n \min_{j\in\{1,\ldots,k\}} D(p_i:c_j).
$$
By defining the distance $D(p,C)=\min_{j\in\{1,\ldots,k\}} D(p:c_j)$ of a point to a set,
we can rewrite the objective function as $E_D(\Lambda,C) = \frac{1}{n}\sum_{i=1}^n D(p_i,C)$.
Let $E_D^*(\Lambda,k)=\min_{C\ :\ |C|=k} E_D(\Lambda,C) $ denote the global minimum of
$E_D(\Lambda,C)$ wrt some given $\Lambda$ and $k$.

The $k$-means++ seeding proceeds for an arbitrary divergence $D$ as follows:
Pick uniformly at random at first seed $c_1$, and then iteratively choose the
$(k-1)$ remaining seeds according to the following probability distribution:
$$
\mathrm{Pr}(c_j=p_i) =
\frac{D(p_i,\{c_1,\ldots,c_{j-1}\})}
{\sum_{i=1}^n D(p_i,\{c_1,\ldots,c_{j-1}\})}
\quad(2\le{j}\le{k}).
$$
Since its inception (2007), this $k$-means++ seeding has been extensively studied~\cite{Bachem-2016}.
We state the general theorem established by~\cite{tJ-2013}:

\begin{theorem}[Generalized $k$-means++ performance, \cite{tJ-2013}]\label{thm:kmeansplus}
Let $\kappa_1$ and $\kappa_2$ be two constants such that $\kappa_1$ defines the
quasi-triangular inequality property:
$$
D(x:z) \leq \kappa_1 \left(D(x:y)+D(y:z)\right),\quad\forall{}x,y,z\in\Delta^d,
$$
and $\kappa_2$ handles the symmetry inequality:
$$
D(x:y)\leq \kappa_2 D(y:x),\quad\forall x,y\in\Delta^d.
$$
Then the generalized $k$-means++ seeding guarantees with high probability a configuration $C$ of cluster centers such that:
\begin{equation}\label{eq:kmperf}
E_D(\Lambda,C)\leq 2\kappa_1^2(1+\kappa_2)(2+\log k) E_D^*(\Lambda,k).
\end{equation}
\end{theorem}
The ratio $\frac{E_D(\Lambda,C)}{E_D^*(\Lambda,k)}$ is called the {\em competitive factor}.
The seminal result of ordinary $k$-means++ was shown~\cite{kmeanspp-2007} to be $8(2+\log{k})$-competitive.
When evaluating $\kappa_1$, one has to note that squared metric distances are not metric because they do not satisfy the triangular inequality.
For example, the squared Euclidean distance is not a metric but it satisfies the $2$-quasi-triangular inequality with $\kappa_1=2$.

We state the following general performance theorem:
\begin{theorem}[$k$-means++ performance in a metric space]\label{theo:kmms}
In any metric space $(\calX,d)$,
the $k$-means++ wrt the squared metric distance $d^2$ is $16(2+\log k)$-competitive.
\end{theorem}
\begin{proof}
Since a metric distance is symmetric, it follows that $\kappa_2=1$.
Consider the quasi-triangular inequality property for the squared non-metric dissimilarity $d^2$:
\begin{eqnarray*}
d(p,q) \leq d(p,q)+d(q,r),\\
d^2(p,q) \leq (d(p,q)+d(q,r))^2,\\
d^2(p,q) \leq d^2(p,q)+d^2(q,r)+2d(p,q)d(q,r).
\end{eqnarray*}

Let us apply the inequality of arithmetic and geometric means\footnote{For positive values $a$ and $b$,
the arithmetic-geometric mean inequality states that $\sqrt{ab}\leq\frac{a+b}{2}$.}:
\begin{equation*}
\sqrt{d^2(p,q)d^2(q,r)} \leq \frac{d^2(p,q)+d^2(q,r)}{2}.
\end{equation*}

Thus we have
\begin{equation*}
d^2(p,q) \leq d^2(p,q)+d^2(q,r)+2d(p,q)d(q,r)\leq 2(d^2(p,q)+d^2(q,r)).
\end{equation*}
That is, the squared metric distance satisfies the $2$-approximate triangle inequality,
and $\kappa_1=2$. The result is straightforward from Theorem~\ref{thm:kmeansplus}.
\end{proof}

\begin{theorem}[$k$-means++ performance in a normed space]\label{theo:kmns}
In any normed space $(\calX,\Vert\cdot\Vert)$,
the $k$-means++ with $D(x:y)=\Vert{x-y}\Vert^2$ is $16(2+\log k)$-competitive.
\end{theorem}
\begin{proof}
In any normed space $(\calX,\Vert\cdot\Vert)$,
we have both $\Vert{x-y}\Vert=\Vert{y-x}\Vert$
and the triangle inequality:
$$
\Vert{}x-z\Vert\le\Vert{x-y}\Vert+\Vert{y-z}\Vert.
$$
The proof is very similar to the proof of Theorem~\ref{theo:kmms} and is omitted.




\end{proof}

Since any inner product space $(\calX,\inner{\cdot}{\cdot})$ has an induced norm
$\Vert{x}\Vert=\sqrt{\inner{x}{x}}$, we have the following corollary.
\begin{corollary}
In any inner product space $(\calX,\inner{\cdot}{\cdot})$,
the $k$-means++ with $D(x:y)=\inner{x-y}{x-y}$ is $16(2+\log k)$-competitive.
\end{corollary}

We need to report a bound for the squared Hilbert symmetric distance ($\kappa_2=1$).
In~\cite{BH-2014} (Theorem 3.3), it was shown that Hilbert geometry of a bounded convex domain $\calC$ is isometric to a normed vector space iff
$\calC$ is an open simplex:  $(\Delta^d,\rho_\HG)\simeq (V^d,\Vert\cdot\Vert_\NH)$, where $\Vert\cdot\Vert_\NH$ is the corresponding norm.
Therefore $\kappa_1=2$. We write ``$\NH$'' for short for this equivalent normed Hilbert geometry.
Appendix~\ref{sec:isometry} recalls the construction due to~\cite{HilbertHarpe-1991}
and shows the squared Hilbert distance fails the triangle inequality
and it is not a distance induced by an inner product.

As an empirical study, we randomly generate $n=10^5$ tuples $(x, y, z)$
based on the uniform distribution in $\Delta^d$. For each tuple $(x,y,z)$,
we evaluate the ratio
$$
\kappa_1 = \frac{D(x:z)}{D(x:y)+D(y:z)}.
$$
Figure~\ref{fig:kappa} shows the statistics for four different choices of
$D$: (1) $D(x:y)=\rhofhr^2(x,y)$;
(2) $D(x:y)=\frac{1}{2}\KL(x:y)+\frac{1}{2}\KL(y:x)$;
(3) $D(x:y)=\rhohg^2(x,y)$;
(4) $D(x:y)= \rhol1^2(x,y)$.
We find experimentally that $\kappa_1$ is upper bounded by 2 for
$\rhofhr^2$, $\rhohg^2$ and $\rhol1^2$,
while the average $\kappa_1$ value is smaller than 0.5.
For all the compared distances, $\kappa_2=1$.
Therefore $\rhofhr$
and $\rhohg$ have better $k$-means++ performance guarantee as compared
to $\rho_{\mathrm{IG}}$.

We get by applying Theorem~\ref{theo:kmns}:
\begin{corollary}[$k$-means++ in  Hilbert simplex geometry]
The $k$-means++ seeding in a  Hilbert simplex geometry in fixed dimension is $16 (2+\log k)$-competitive.
\end{corollary}

Figure~\ref{fig:kmresults} displays the clustering results of $k$-means++ in Hilbert simplex geometry as
compared to the other geometries for $k\in\{3,5\}$.

\begin{figure}[tbh]
\centering
\includegraphics[width=\textwidth]{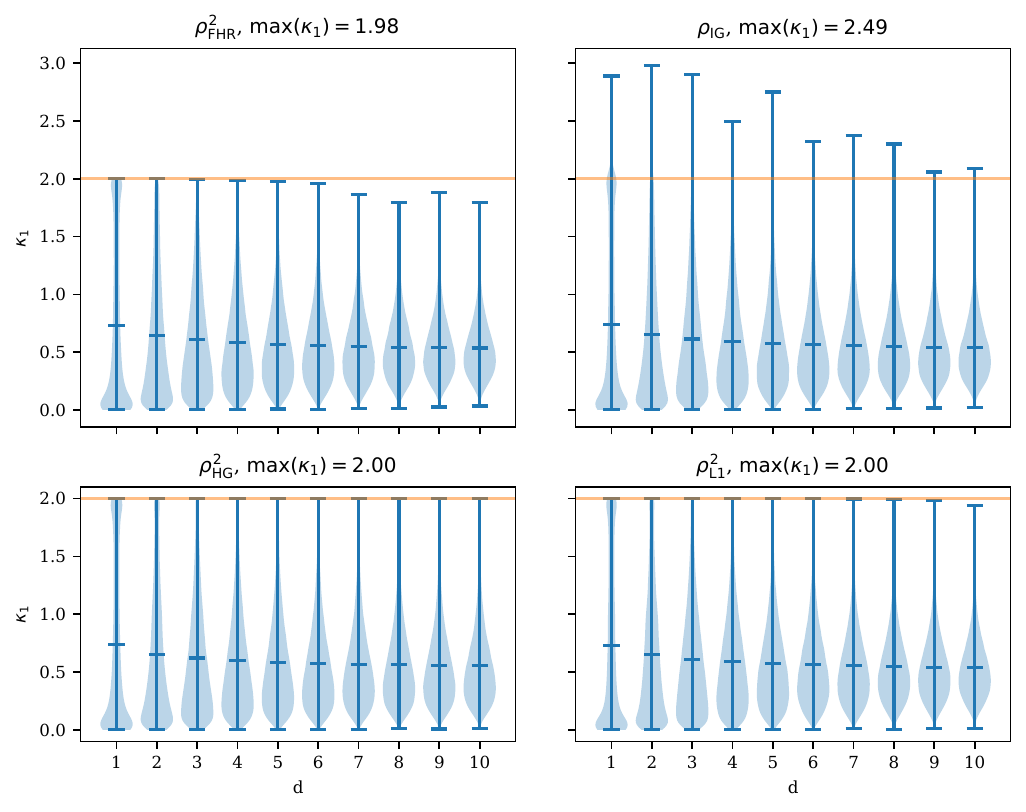}

\caption{The maximum, mean, minimum, and estimated density of $\kappa_1$
    on $10^5$ randomly generated tuples $(x,y,z)$ in
    $\Delta^d_\epsilon=\{ p=(p_0,p_1,...,p_d) : p_i\ge\epsilon \}$ (for a small prescribed $\epsilon=10^{-10}$)
    for $d=1,\dots,10$.}\label{fig:kappa}
\end{figure}

Figure~\ref{fig:compdistance} reports average/deviation and extremal minimum and maximal distances for randomly generated tuples inside the probability simplex $\Delta^d_\epsilon=\{ p=(p_0,p_1,...,p_d) : p_i\ge\epsilon \}$.

\begin{figure}[tbh]
\centering
\includegraphics[width=\textwidth]{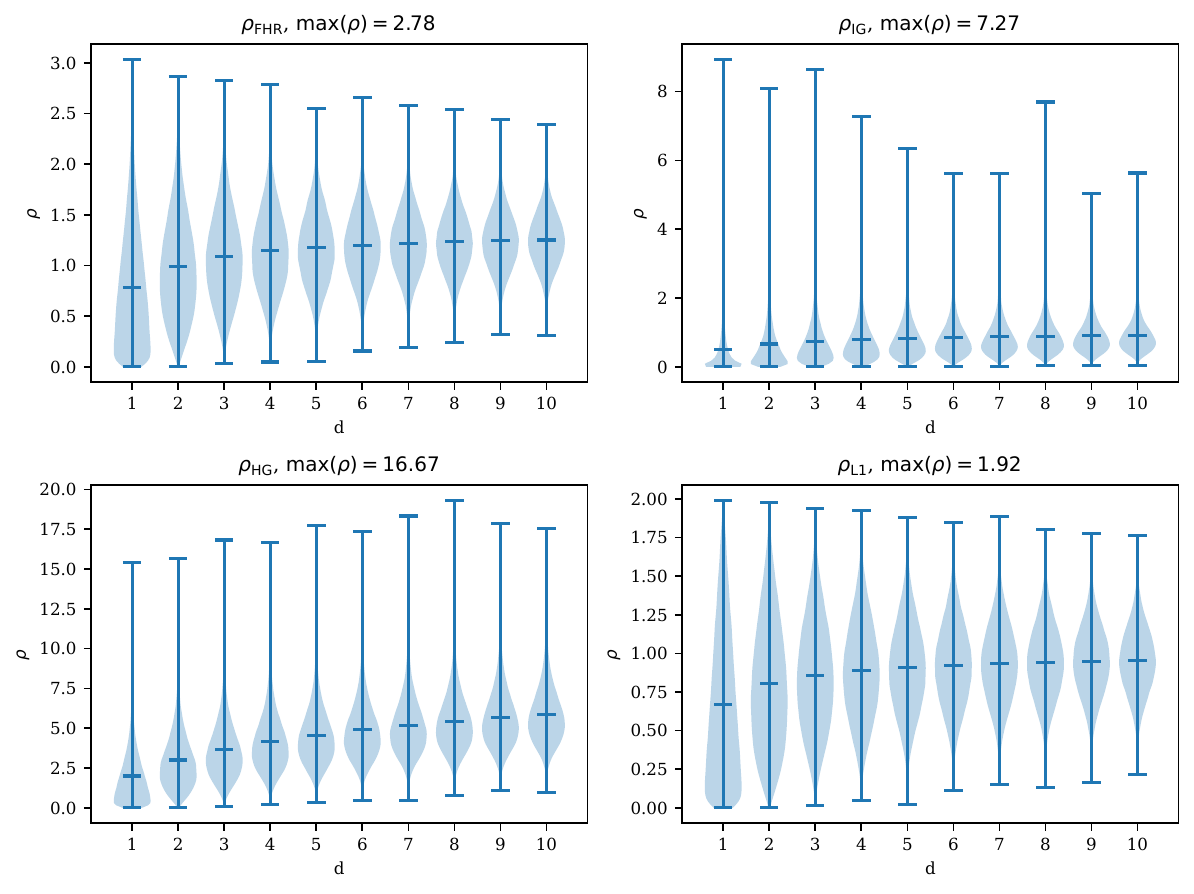}

\caption{Comparisons of distance statistics: Maximum, mean, minimum, and
    estimated density of distances on $10^5$ randomly generated tuples $(x,y,z)$
    in $\Delta^d_\epsilon=\{ p=(p_0,p_1,...,p_d) : p_i\ge\epsilon \}$
    (for a small prescribed $\epsilon=10^{-10}$)
    for $d=1,\dots,10$.}\label{fig:compdistance}
\end{figure}

\begin{figure}
\centering
\begin{subfigure}{\textwidth}
\includegraphics[width=\textwidth]{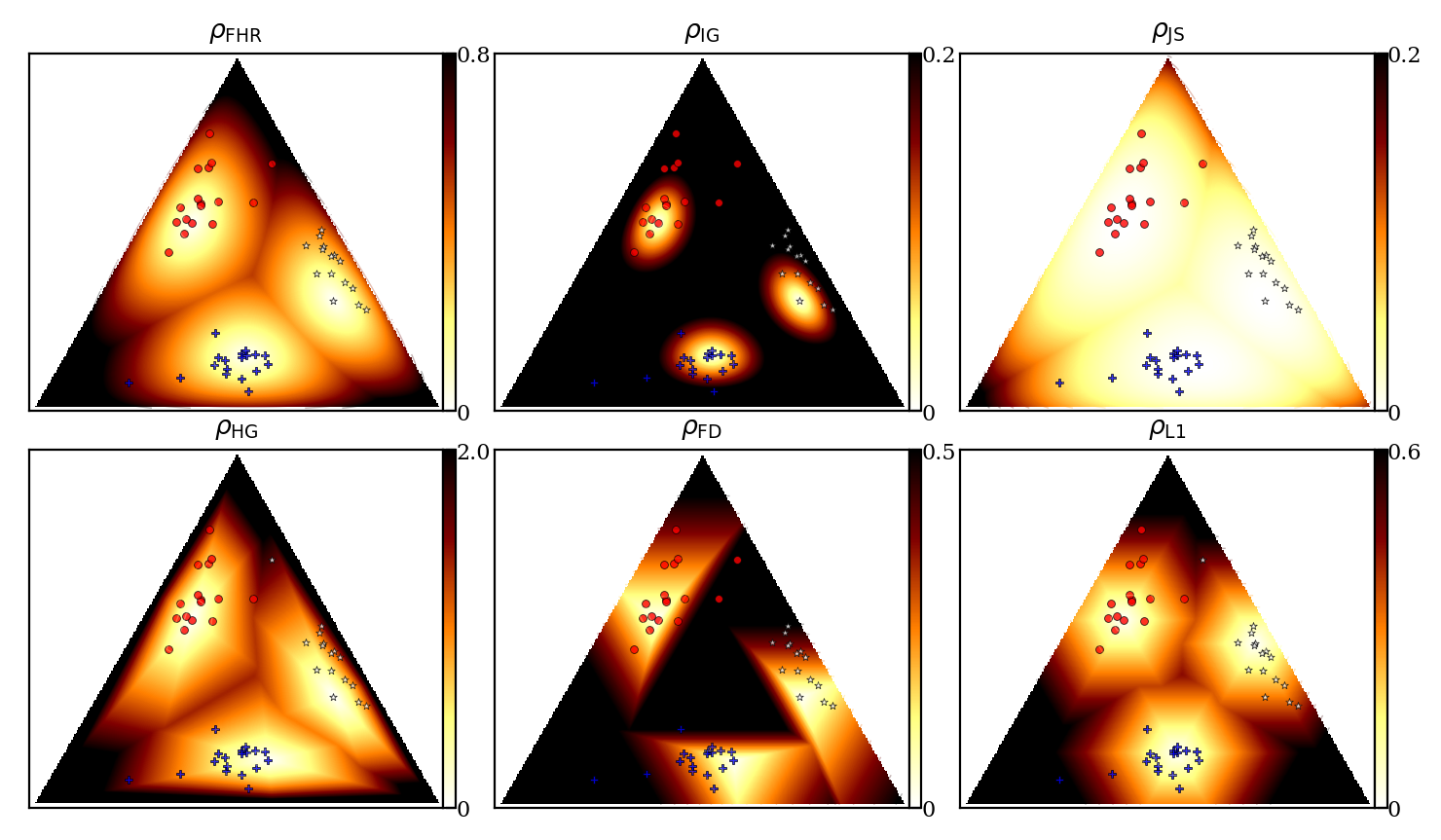}
\caption{$k=3$ clusters}
\end{subfigure}
\begin{subfigure}{\textwidth}
\includegraphics[width=\textwidth]{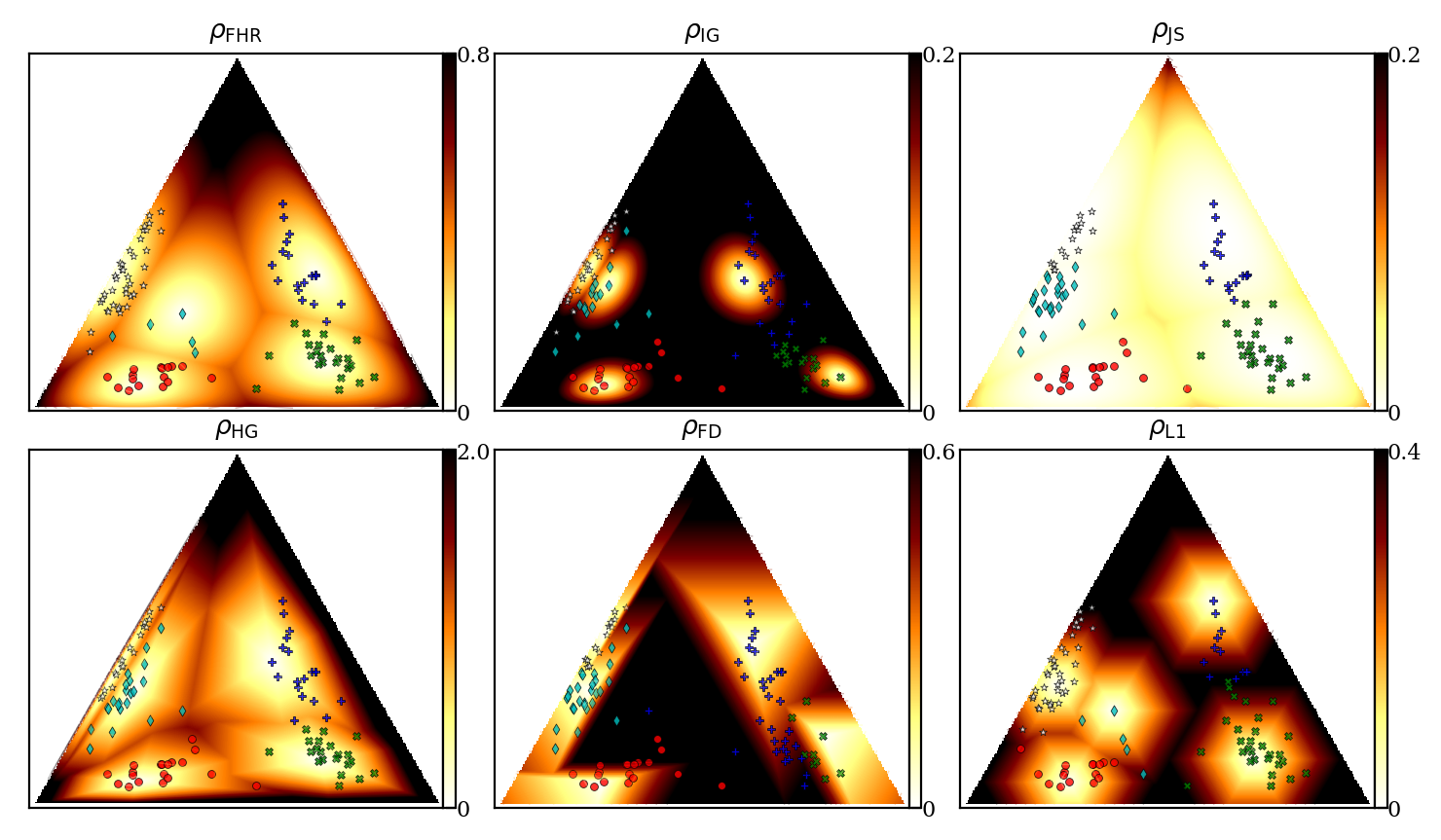}
\caption{$k=5$ clusters}
\end{subfigure}
\caption{$k$-Means++ clustering results on a toy dataset in the space of trinomials $\Delta^2$.
The color density maps indicate the distance from any point to its nearest cluster center.}\label{fig:kmresults}
\end{figure}

The KL divergence can be interpreted as a separable Bregman divergence~\cite{BregmanKmeans-2010}.
The Bregman $k$-means++ performance has been studied in~\cite{BregmanKmeans-2010,smoothedBregmankMeans-2013},
and a competitive factor of $O(\frac{1}{\mu})$ is reported using the notion of Bregman $\mu$-similarity (that is suited for data-sets on a compact domain).

In~\cite{sphericalkm-2015}, spherical $k$-means++ is studied wrt the distance $d_S(x,y)=1-\inner{x}{y}$ for any pair of points $x,y$ on the unit sphere.
Since $\inner{x}{y}=\|x\|_2 \|y\|_2 \cos(\theta_{x,y})=\cos(\theta_{x,y})$, we have  $d_S(x,y)=1-\cos(\theta_{x,y})$, where $\theta_{x,y}$
denotes the angle between a pair of unit vectors $x$ and $y$. This distance is called the cosine distance since it amounts to one minus the cosine similarity.
Notice that the cosine distance is related to the squared Euclidean distance via the identity: $d_S(x,y)=\frac{1}{2}\|x-y\|^2$.
The cosine distance is different from the spherical distance that relies on the arccos function.

Since divergences may be asymmetric, one can further consider mixed divergence $M(p:q:r)=\lambda D(p:q)+(1-\lambda)D(q:r)$
for $\lambda\in[0,1]$, and extend the $k$-means++ seeding procedure and analysis~\cite{MixedClustering-2014}.

For a given data set, we can compute $\kappa_1$ or $\kappa_2$ by inspecting triples and pairs of points
and get a data-dependent competitive factor improving the bounds mentioned above.

\subsection{$k$-Center clustering}

Let $\Lambda$ be a finite point set.
The cost function for a $k$-center clustering with centers $C$ ($|C|=k$) is:
$$
f_D(\Lambda, C) = \max_{p_i\in\Lambda} \min_{c_j\in{}C} D(p_i\,:\,c_j).
$$
The farthest first traversal heuristic~\cite{kcenter-1985} has a guaranteed approximation factor of $2$ for any metric distance (see Algorithm~\ref{alg:kcenter}).

\begin{algorithm}[t]
\KwData{A set of points $p_1,\cdots,p_n\in\Delta^d$.  A distance measure $\rho$ on $\Delta^d$.
The maximum number $k$ of clusters. The maximum number $T$ of iterations.}
\KwResult{A clustering scheme assigning each $p_i$ a label $l_i\in\{1,\ldots,k\}$}
\Begin{
Randomly pick $k$ cluster centers $c_1,\ldots,c_k$ using
the kmeans\texttt{++} heuristic\;
\For{$t=1,\cdots,T$}{
\For{$i=1,\cdots,n$}{
  $l_i\leftarrow\argmin_{l=1}^k \rho(p_i, c_l)$\;
}
\For{$l=1,\cdots,k$}{
  $c_l\leftarrow \argmin_c\max_{i:l_i=l} \rho(p_i,c)$\;
}
}
Output $\{l_i\}_{i=1}^n$\;
}
\caption{$k$-Center clustering}\label{alg:cluster}
\end{algorithm}

\begin{algorithm}
\KwData{A set $\Lambda$; a number $k$ of clusters; a metric distance $\rho$.}
\KwResult{A $2$-approximation of the $k$-center clustering}
\Begin{
$c_1\leftarrow\mathrm{ARandomPointOf}(\Lambda)$\;
$C\leftarrow\{c_1\}$\;
\For{$i=2,\cdots,{k}$}{
$c_i\leftarrow\arg\max_{p\in\Lambda} \rho(p,C)$\;
$C\leftarrow C\cup \{c_i\}$\;
}}
Output $C$\;
\caption{A $2$-approximation of the $k$-center clustering for any metric distance $\rho$.\label{alg:kcenter}}
\end{algorithm}

In order to use the $k$-center clustering algorithm described in Algorithm~\ref{alg:cluster},
we need to be able to compute the $1$-center (or minimax center) for the Hilbert simplex geometry,
that is the Minimum Enclosing Ball (MEB, also called the Smallest Enclosing Ball, SEB).

We may consider the SEB equivalently either in $\Delta^d$ or in the normed space $V^d$.
In both spaces, the shapes of the balls are convex.
Let $\Lambda=\{p_1,\ldots,p_n\}$ denote the point set in $\Delta^d$,
and $\calV=\{v_1,\ldots,v_n\}$ the equivalent point set in the normed vector space
(following the mapping explained in Appendix~\ref{sec:isometry}).
Then the SEBs $B_\HG(\Lambda)$ in $\Delta^d$ and $B_\NH(\calV)$ in $V^d$ have respectively radii $r^*_\HG$ and $r_\NH^*$ defined by:

\begin{eqnarray*}
r^*_\HG &=& \min_{c\in\Delta^d}  \max_{i\in\{1,\ldots, n\}} \rho_\HG(p_i,c),\\
r^*_\NH &=& \min_{v\in V^d}      \max_{i\in\{1,\ldots, n\}} \|v_i-v\|_\NH.
\end{eqnarray*}


The SEB in the normed vector space $(V^d,\|\cdot\|_\NH)$ amounts to find the minimum covering norm polytope  of a finite point set.
This problem has been well-studied in computational geometry~\cite{Saha-2011,Brandenberg-2013,EnclosingPolytope-2004}.
By considering the equivalent Hilbert norm polytope with $d(d+1)$ facets, we state the result of~\cite{Saha-2011}:

\begin{theorem}[SEB in Hilbert polytope normed space,~\cite{Saha-2011}]
A $(1+\epsilon)$-approximation of the SEB in $V^d$ can be computed in $O(d^3\frac{n}{\epsilon})$ time.
\end{theorem}

We shall now report two algorithms for computing the SEBs: One exact algorithm in $V^d$ that does not scale well in high dimensions,
and one approximation in $\Delta^d$ that works well for large dimensions.

\subsubsection{Exact smallest enclosing ball in a Hilbert simplex geometry\label{sec:SEBHSG}}

Given a finite point set $\{p_1,\ldots, p_n\}\in\Delta^d$, the SEB in Hilbert simplex geometry
is centered at
\begin{equation*}
c^*=\argmin_{c\in \Delta^d} \max_{i\in\{1,\ldots,n\}} \rho_\HG(c,x_i),
\end{equation*}
with radius
\begin{equation*}
r^*=\min_{c\in \Delta^d} \max_{i\in\{1,\ldots,n\}} \rho_\HG(c,x_i).
\end{equation*}

An equivalent problem is to find the SEB in the isometric normed vector space $V^d$
via the mapping reported in Appendix~\ref{sec:isometry}.
Each simplex point $p_i$ corresponds to a point $v_i$ in the $V^d$.

%
%

Figure~\ref{fig:SEB} displays some examples of the exact smallest enclosing balls in the Hilbert simplex geometry and in the corresponding normed vector space.

\begin{figure}
\centering
\begin{subfigure}[m]{.48\textwidth}
\includegraphics[width=\textwidth]{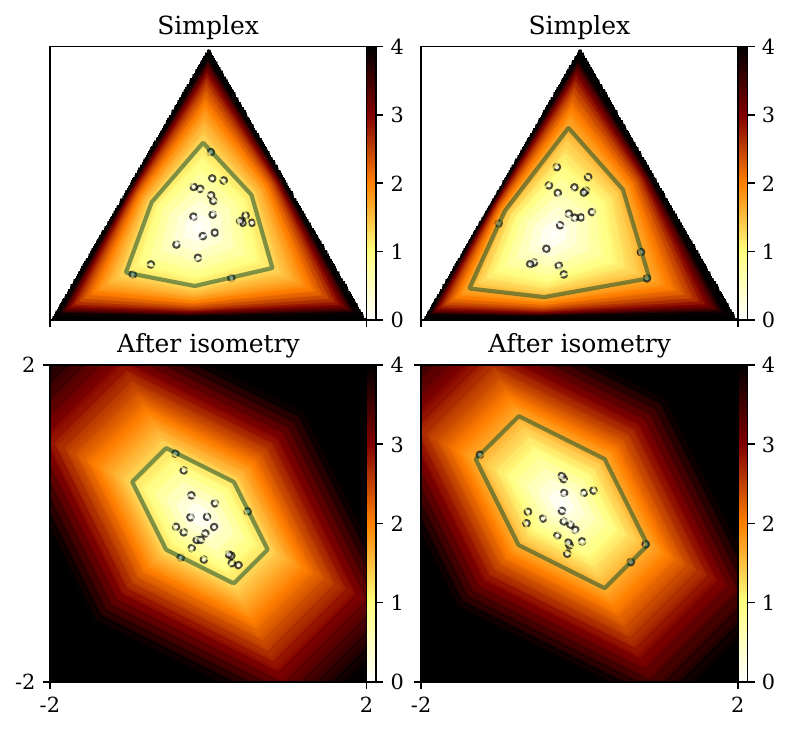}%
\caption{$3$ points on the border}
\end{subfigure}
\begin{subfigure}[m]{.48\textwidth}
\includegraphics[width=\textwidth]{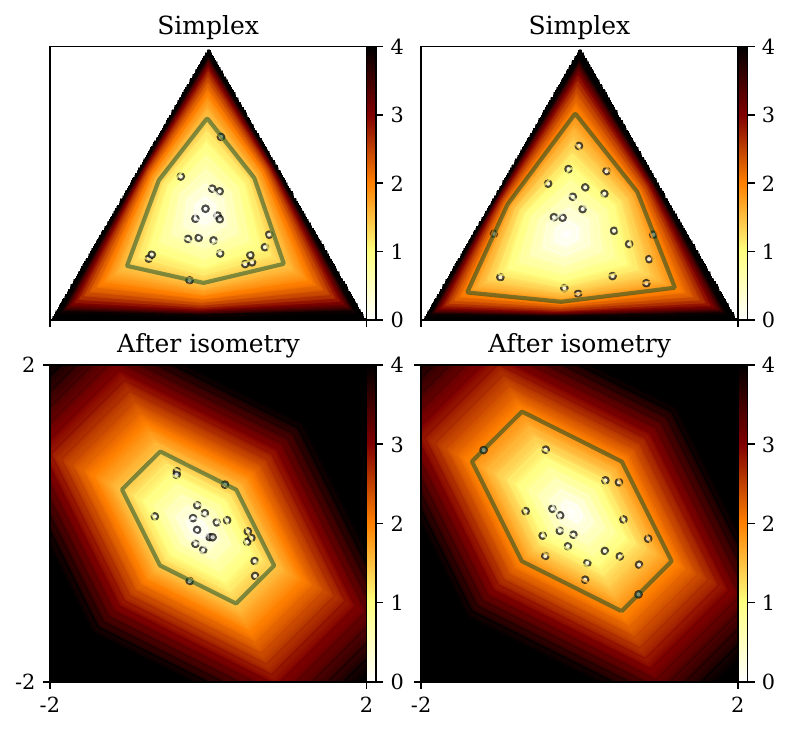}%
\caption{$2$ points on the border}
\end{subfigure}
\caption{Computing the SEB in Hilbert simplex geometry
amounts to compute the SEB in the corresponding normed
vector space.\label{fig:SEB}}
\end{figure}

To compute the SEB, one may also consider the generic LP-type randomized algorithm~\cite{SEBB-2008}.
We notice that an enclosing ball for a point set in general has
a number $k$ of points on the border of the ball, with $2\leq k\leq \frac{d(d+1)}{2}$.
Let $D=\frac{d(d+1)}{2}$ denote the varying size of the combinatorial basis,
then we can apply the LP-type framework (we check the axioms of locality and monotonicity,~\cite{LPtype-1992})
to solve efficiently the SEBs.

\begin{theorem}[Smallest Enclosing Hilbert Ball is LP-type,~\cite{Welzl-1991,LPtype-1992}]
The smallest enclosing Hilbert ball amounts to find the smallest enclosing ball in a vector space
with respect to a polytope norm that can be solved using an LP-type randomized algorithm.
\end{theorem}

The Enclosing Ball Decision Problem~(EBDP, \cite{BallML-2009}) asks for a given value $r$, whether $r\geq r^*$ or not.
The decision problem amounts to find whether a set $\{rB_V+v_i\}$ of translates can be stabbed by a point~\cite{BallML-2009}:
That is, whether $\cap_{i=1}^n (rB_V+v_i)$ is empty or not. Since these translates are polytopes with $d(d+1)$ facets,
this can be solved in linear time using {\em Linear Programming}.

\begin{theorem}[Enclosing Hilbert Ball Decision Problem]
The decision problem to test whether $r\geq r^*$ or not can be solved by Linear Programming.
\end{theorem}

This yields a simple scheme to approximate the optimal value $r^*$:
Let $r_0=\max_{i\in\{2,\ldots,n\}} \|v_i-v_1\|_\NH$. Then $r^*\in[\frac{r_0}{2},r_0]=[a_0,b_0]$.
At stage $i$, perform a dichotomic search on $[a_i,b_i]$ by answering the decision problem for
$r_{i+1}=\frac{a_i+b_i}{2}$, and update the radius range accordingly~\cite{BallML-2009}.

However, the LP-type randomized algorithm or the decision problem-based algorithm do not scale well in high dimensions.
Next, we introduce a simple approximation algorithm that relies on the fact that the line segment $[pq]$ is a geodesic in Hilbert simplex geometry.
(Geodesics are not unique. See Figure~2 of~\cite{HilbertHarpe-1991} and Figure~\ref{fig:notgeodesic}.)

Notice that in Figure~\ref{fig:SEB}, we ``rendered'' the probability simplex using a 2D isosceles right triangle $T_i$,
while in Figure~\ref{fig:results}, we used a 2D  equilateral triangle $T_e$ (embedded in 3D).
Since Hilbert geometries are invariant to collineations (including the affine transformations), the Hilbert geometry induced by $T_i$ is
isometric to the Hilbert geometry induced by $T_e$.

\subsubsection{Geodesic bisection approximation heuristic\label{sec:BCHSG}}
In Riemannian geometry, the $1$-center can be arbitrarily finely approximated by
a simple geodesic bisection algorithm~\cite{bc-2003,miniball-2013}.
This algorithm can be extended to HG straightforwardly as detailed in Algorithm~\ref{alg:center}.

\begin{algorithm}
\KwData{A set of points $p_1,\cdots,p_n\in\Delta^d$. The maximum number $T$ of iterations.}
\KwResult{$c\approx\argmin_c\max_i\rhohg(p_i,c)$}
\Begin{
$c_0\leftarrow\mathrm{ARandomPointOf}(\{p_1,\cdots,p_n\})$\;
\For{$t=1,\cdots,T$}{
$p\leftarrow\argmax_{p_i}\rhohg(p_i,c_{t-1})$\;
$c_t\leftarrow c_{t-1} \#_{1/(t+1)}^\rho p$\;
}
Output $c_{T}$\;
}
\caption{Geodesic walk for approximating the Hilbert minimax center, generalizing~\cite{bc-2003}}\label{alg:center}
\end{algorithm}

\begin{figure}[tbh]
\centering
\includegraphics[width=\textwidth]{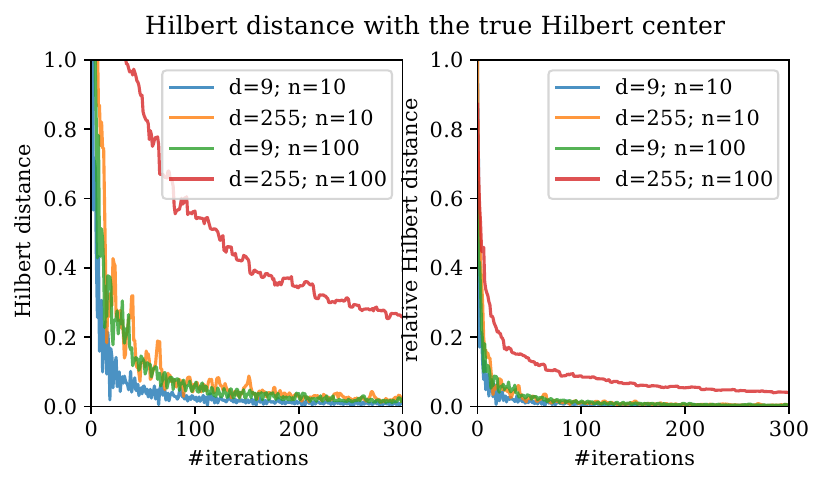}
\caption{Convergence rate of Alg.~(\ref{alg:center})
measured by the Hilbert distance between the current minimax center and the true center (left)
or their Hilbert distance divided by the Hilbert radius of the dataset (right).
The plot is based on $100$ random points in $\Delta^{9}$/$\Delta^{255}$.}\label{fig:convergence}
\end{figure}

\begin{figure}
\centering
\begin{subfigure}[m]{.7\textwidth}
\includegraphics[width=\textwidth]{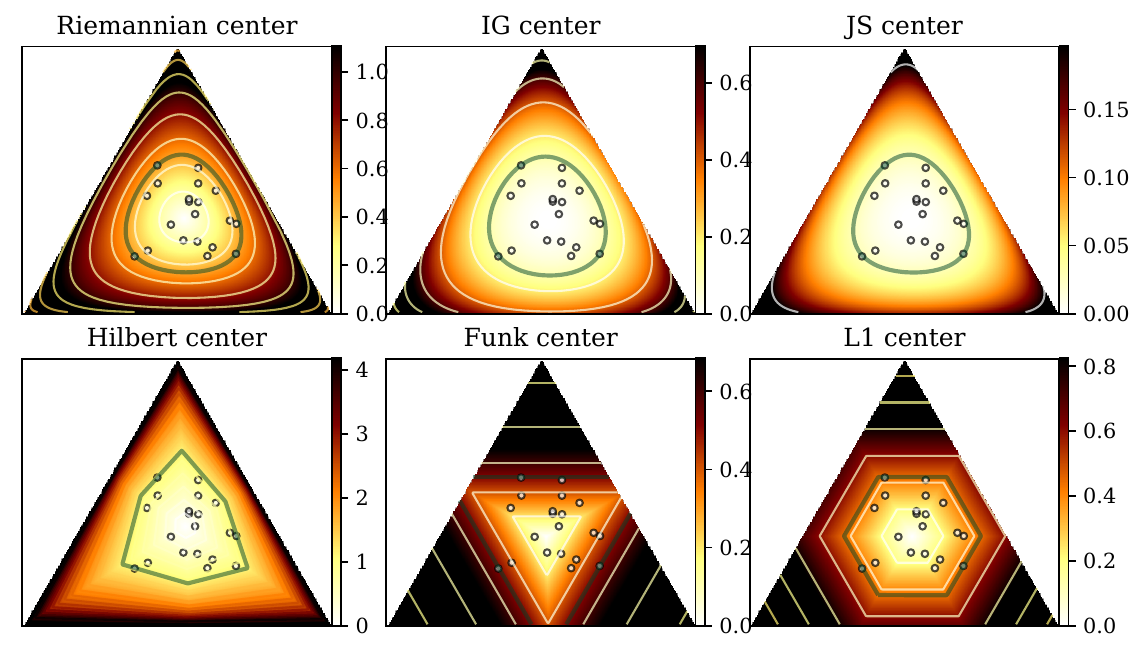}%
\caption{Point Cloud 1}
\end{subfigure}
\begin{subfigure}[m]{.7\textwidth}
\includegraphics[width=\textwidth]{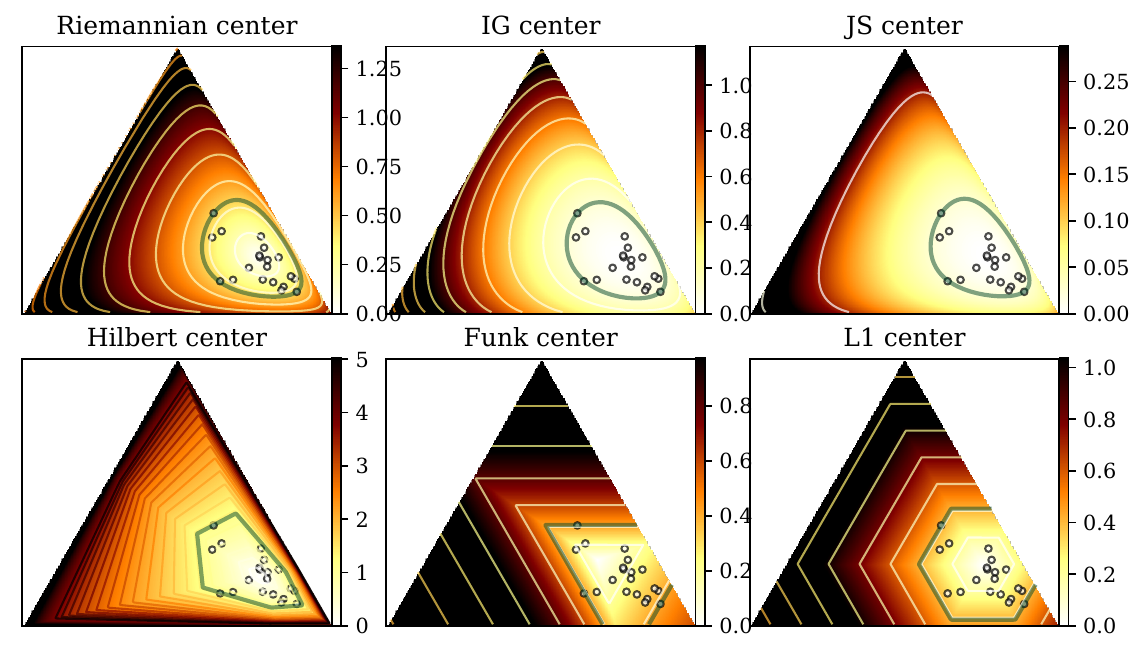}%
\caption{Point Cloud 2}
\end{subfigure}
\end{figure}

\begin{figure}\ContinuedFloat
\centering
\begin{subfigure}[m]{.7\textwidth}
\includegraphics[width=\textwidth]{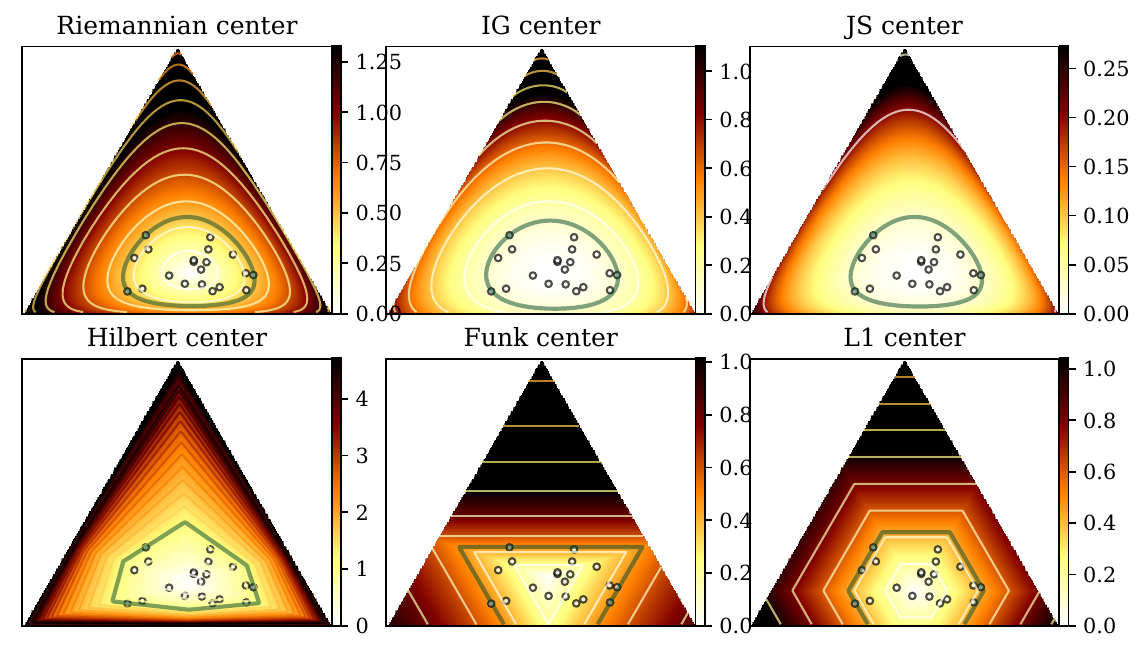}%
\caption{Point Cloud 3}
\end{subfigure}
\caption{The Riemannian/IG/Hilbert/$L_1$  minimax centers of three point clouds
in $\Delta^2$ based on Alg.~(\ref{alg:center}).
The color maps show the distance from $\forall{p}\in\Delta^2$ to the corresponding center.}
\label{fig:center}
\end{figure}

The algorithm first picks up a point $c_0$ at random from $\Lambda$ as the initial center,
then computes the farthest point $p$ (with respect to the distance $\rho$),
and then walk on the geodesic from $c_0$ to $p$ by a certain amount to define $c_1$, etc.
For an arbitrary distance $\rho$, we define the operator $\#^\rho_\alpha$ as follows:
$$
p \#^\rho_\alpha q = v=\gamma(p,q,\alpha), \quad \rho(p:v)=\alpha \rho(p:q),
$$
where $\gamma(p,q,\alpha)$ is the geodesic passing through $p$ and $q$, and parameterized by $\alpha$ ($0\le\alpha\le1$).
When the equations of the geodesics are explicitly known, we can either get a closed form solution for $\#^\rho_\alpha$
or perform a bisection search to find $v'$ such that $\rho(p:v')\approx\alpha\rho(p:q)$.
See~\cite{ApproximatingHyperbolicBall-2015} for an extension and analysis in hyperbolic geometry.
See Fig.~(\ref{fig:convergence}) to get an intuitive idea on the {\em experimental} convergence rate of Algorithm~\ref{alg:center}.
See Fig.~(\ref{fig:center}) for visualizations of centers wrt different geometries.

Furthermore, this iterative algorithm implies a core-set~\cite{coreset-2008} (namely, the set of farthest points visited during the geodesic walks)
that is useful for clustering large data-sets~\cite{coreset-2017}.
See~\cite{Brandenberg-2013} for core-set results on containment problems
wrt a convex homothetic object (the equivalent Hilbert polytope norm in our case).

A simple algorithm dubbed {\sc MinCon} \cite{EnclosingPolytope-2004}
can find an approximation of the Minimum Enclosing Polytope.
The algorithm induces a core-set of size $O(\frac{1}{\eps^2})$ although the theorem is challenged in~\cite{Brandenberg-2013}.

Thus by combining the $k$-center seeding \cite{kcenter-1985} with the Lloyd-like batched iterations, we get an efficient $k$-center clustering
algorithm for the FHR and Hilbert metric geometries.
When dealing with the Kullback-Leibler divergence, we use the fact that KL is a Bregman divergence, and
use the $1$-center algorithm~(\cite{BregmanMinimax-2005,BregmanBall-2006} for approximation in any dimension,
or \cite{SEBB-2008} which is exact but limited to small dimensions).

Since Hilbert simplex geometry is isomorphic to a normed vector space~\cite{BH-2014} with
a polytope norm with $d(d+1)$ facets, the Voronoi diagram in Hilbert geometry of $\Delta^d$ amounts to
compute a Voronoi diagram wrt a polytope norm~\cite{MinisumHypersphere-2011,VoronoiNorm-2012,voronoi-2015}.

\section{Experiments}\label{sec:exp}
We generate a dataset consisting of a set of clusters in a high dimensional statistical simplex $\Delta^d$.
Each cluster is generated independently as follows. We first pick a
random center $c=(\lambda_c^0,\ldots,\lambda_c^d)$ based on the uniform distribution on $\Delta^d$.
Then any random sample $p=(\lambda^0,\ldots,\lambda^d)$ associated with $c$ is independently generated by
$$
\lambda^i = \frac{\exp(\log\lambda_c^i + \sigma\epsilon^i)}{\sum_{i=0}^d \exp(\log\lambda_c^i + \sigma\epsilon^i) },
$$
where $\sigma>0$ is a noise level parameter, and each $\epsilon^i$ follows independently
a standard Gaussian distribution (generator 1) or the Student's $t$-distribution
with five degrees of freedom (generator 2).
Let $\sigma=0$, we get $\lambda^i=\lambda_c^i$. Therefore $p$ is randomly distributed around $c$.
We repeat generating random samples for each cluster center, and make sure
that different clusters have almost the same number of samples. Then we
perform clustering based on the configurations
$n\in\{50,100\}$, $d\in\{9,255\}$, $\sigma\in\{0.5,0.9\}$,
$\rho\in\{\rhofhr, \rho_{\mathrm{IG}}, \rhohg,
\rhoeuc, \rhol1\}$.
For simplicity, the number of clusters $k$ is set to the ground truth.
For each
configuration, we repeat the clustering experiment based on 300 different random
datasets. The performance is measured by the normalized mutual information (NMI),
which is a scalar indicator in the range $[0,1]$ (the larger the better).

The results of $k$-means++ and $k$-centers are shown in Table~\ref{tbl:plusplus}
and Table~\ref{tbl:results}, respectively. The large variance of NMI is because
that each experiment is performed on random datasets wrt different random seeds.
Generally, the performance deteriorates as we increase the number of clusters,
increase the noise level or decrease the dimensionality,
which has the same effect to reduce the inter-cluster gap.

The key comparison is the three columns $\rhofhr$, $\rhohg$
and $\rho_{\mathrm{IG}}$, as they are based on exactly the same algorithm
with the only difference being the underlying geometry.
We see clearly that in general, their clustering performance presents the order $\mathrm{HG}>\mathrm{FHR}>\mathrm{IG}$.
The performance of HG is superior to the other two geometries, especially when the noise level is large.
Intuitively, the Hilbert balls are more compact in size and
therefore can better capture the clustering structure (see Fig.~(\ref{fig:results})).


The column $\rhoeuc$ is based on the Euclidean enclosing
ball. It shows the worst scores because the intrinsic geometry of the
probability simplex is far from the Euclidean geometry.


\begin{table*}
\centering
\caption{$k$-means++ clustering accuracy in NMI on randomly generated datasets
based on different geometries. The table shows
the mean and standard deviation after 300 independent runs for each configuration.
$\rho$ is the distance measure. $n$ is the sample size.
$d$ is the dimensionality of $\Delta^d$. $\sigma$ is noise level.}\label{tbl:plusplus}
\begin{tabular}{c|c|c|c|ccccc}
\toprule[1.5pt]
$k$ & $n$ & $d$ & $\sigma$ & $\rhofhr$ & $\rho_{\mathrm{IG}}$ & $\rhohg$ & $\rhoeuc$ & $\rho_{L1}$\\\hline
\multirow{12}{*}{3}
&\multirow{6}{*}{50} &\multirow{3}{*}{$9$}
  & $0.5$ & $0.76\pm0.22$ & $0.76\pm0.24$ & $\bm{0.81\pm0.22}$ & $0.64\pm0.23$ & $0.70\pm0.22$ \\
&&& $0.9$ & $0.44\pm0.20$ & $0.44\pm0.20$ & $\bm{0.57\pm0.22}$ & $0.31\pm0.17$ & $0.38\pm0.18$ \\\cline{3-9}
&&\multirow{3}{*}{$255$}
  & $0.5$ & $0.80\pm0.24$ & $0.81\pm0.24$ & $\bm{0.88\pm0.21}$ & $0.74\pm0.25$ & $0.79\pm0.24$ \\
&&& $0.9$ & $0.65\pm0.27$ & $0.66\pm0.28$ & $\bm{0.72\pm0.27}$ & $0.46\pm0.24$ & $0.63\pm0.27$ \\\cline{2-9}
&\multirow{6}{*}{100}
&\multirow{3}{*}{$9$}
 &  $0.5$ & $0.76\pm0.22$ & $0.76\pm0.21$ & $\bm{0.82\pm0.22}$ & $0.60\pm0.21$ & $0.69\pm0.23$ \\
&&& $0.9$ & $0.42\pm0.19$ & $0.41\pm0.18$ & $\bm{0.54\pm0.22}$ & $0.27\pm0.14$ & $0.34\pm0.16$ \\\cline{3-9}
&& \multirow{3}{*}{$255$}
  & $0.5$ & $0.82\pm0.23$ & $0.82\pm0.24$ & $\bm{0.89\pm0.20}$ & $0.74\pm0.24$ & $0.80\pm0.25$ \\
&&& $0.9$ & $0.66\pm0.26$ & $0.66\pm0.28$ & $\bm{0.72\pm0.26}$ & $0.45\pm0.25$ & $0.64\pm0.27$ \\\cline{1-9}
\multirow{12}{*}{5}
& \multirow{6}{*}{50}
& \multirow{3}{*}{9}
  & $0.5$ & $0.75\pm0.14$ & $0.74\pm0.15$ & $\bm{0.81\pm0.13}$ & $0.61\pm0.13$ & $0.68\pm0.13$ \\
&&& $0.9$ & $0.44\pm0.13$ & $0.42\pm0.13$ & $\bm{0.55\pm0.15}$ & $0.31\pm0.11$ & $0.36\pm0.12$ \\\cline{3-9}
&& \multirow{3}{*}{$255$}
  & $0.5$ & $0.83\pm0.15$ & $0.83\pm0.15$ & $\bm{0.88\pm0.14}$ & $0.77\pm0.16$ & $0.82\pm0.15$ \\
&&& $0.9$ & $0.71\pm0.17$ & $0.70\pm0.19$ & $\bm{0.75\pm0.17}$ & $0.50\pm0.17$ & $0.68\pm0.18$ \\\cline{2-9}
&\multirow{6}{*}{100}
& \multirow{3}{*}{$9$}
  & $0.5$ & $0.74\pm0.13$ & $0.74\pm0.14$ & $\bm{0.80\pm0.14}$ & $0.60\pm0.13$ & $0.67\pm0.13$ \\
&&& $0.9$ & $0.42\pm0.11$ & $0.40\pm0.12$ & $\bm{0.55\pm0.15}$ & $0.29\pm0.09$ & $0.35\pm0.11$ \\\cline{3-9}
&& \multirow{3}{*}{$255$}
  & $0.5$ & $0.83\pm0.14$ & $0.83\pm0.15$ & $\bm{0.88\pm0.13}$ & $0.77\pm0.15$ & $0.81\pm0.15$ \\
&&& $0.9$ & $0.69\pm0.18$ & $0.69\pm0.18$ & $\bm{0.73\pm0.17}$ & $0.48\pm0.17$ & $0.67\pm0.18$ \\
\bottomrule[1.5pt]
\end{tabular}
\\(a) generator 1\\\vspace{1em}
\begin{tabular}{c|c|c|c|ccccc}
\toprule[1.5pt]
$k$ & $n$ & $d$ & $\sigma$ & $\rhofhr$ & $\rho_{\mathrm{IG}}$ & $\rhohg$ & $\rhoeuc$ & $\rho_{L1}$\\\hline
\multirow{12}{*}{3}
&\multirow{6}{*}{50} &\multirow{3}{*}{$9$}
  & $0.5$ & $0.62\pm0.22$ & $0.60\pm0.22$ & $\bm{0.71\pm0.23}$ & $0.45\pm0.20$ & $0.54\pm0.22$ \\
&&& $0.9$ & $0.29\pm0.17$ & $0.27\pm0.16$ & $\bm{0.39\pm0.19}$ & $0.17\pm0.13$ & $0.25\pm0.15$ \\\cline{3-9}
&&\multirow{3}{*}{$255$}
  & $0.5$ & $0.70\pm0.25$ & $0.69\pm0.26$ & $\bm{0.74\pm0.25}$ & $0.37\pm0.29$ & $0.70\pm0.26$ \\
&&& $0.9$ & $\bm{0.42\pm0.25}$ & $0.35\pm0.20$ & $0.40\pm0.19$ & $0.03\pm0.08$ & $\bm{0.44\pm0.26}$ \\\cline{2-9}
&\multirow{6}{*}{100}
&\multirow{3}{*}{$9$}
 &  $0.5$ & $0.63\pm0.22$ & $0.61\pm0.22$ & $\bm{0.71\pm0.22}$ & $0.46\pm0.19$ & $0.56\pm0.20$ \\
&&& $0.9$ & $0.29\pm0.15$ & $0.26\pm0.14$ & $\bm{0.38\pm0.20}$ & $0.18\pm0.12$ & $0.24\pm0.14$ \\\cline{3-9}
&& \multirow{3}{*}{$255$}
  & $0.5$ & $0.71\pm0.26$ & $0.69\pm0.27$ & $\bm{0.75\pm0.25}$ & $0.31\pm0.28$ & $0.70\pm0.27$ \\
&&& $0.9$ & $0.41\pm0.26$ & $0.33\pm0.20$ & $0.38\pm0.18$ & $0.02\pm0.06$ & $\bm{0.43\pm0.26}$ \\\cline{1-9}
\multirow{12}{*}{5}
& \multirow{6}{*}{50}
& \multirow{3}{*}{9}
  & $0.5$ & $0.64\pm0.15$ & $0.61\pm0.14$ & $\bm{0.70\pm0.14}$ & $0.48\pm0.14$ & $0.57\pm0.15$ \\
&&& $0.9$ & $0.31\pm0.12$ & $0.29\pm0.12$ & $\bm{0.41\pm0.15}$ & $0.20\pm0.09$ & $0.26\pm0.10$ \\\cline{3-9}
&& \multirow{3}{*}{$255$}
  & $0.5$ & $0.74\pm0.17$ & $0.72\pm0.17$ & $\bm{0.77\pm0.16}$ & $0.41\pm0.20$ & $0.74\pm0.17$ \\
&&& $0.9$ & $0.44\pm0.17$ & $0.37\pm0.16$ & $0.44\pm0.15$ & $0.04\pm0.06$ & $\bm{0.47\pm0.17}$ \\\cline{2-9}
&\multirow{6}{*}{100}
& \multirow{3}{*}{$9$}
  & $0.5$ & $0.62\pm0.14$ & $0.61\pm0.14$ & $\bm{0.71\pm0.14}$ & $0.46\pm0.13$ & $0.54\pm0.14$ \\
&&& $0.9$ & $0.30\pm0.10$ & $0.27\pm0.11$ & $\bm{0.40\pm0.13}$ & $0.19\pm0.08$ & $0.25\pm0.09$ \\\cline{3-9}
&& \multirow{3}{*}{$255$}
  & $0.5$ & $0.73\pm0.18$ & $0.70\pm0.18$ & $\bm{0.75\pm0.16}$ & $0.37\pm0.20$ & $0.73\pm0.17$ \\
&&& $0.9$ & $0.43\pm0.16$ & $0.35\pm0.14$ & $0.41\pm0.12$ & $0.03\pm0.06$ & $\bm{0.46\pm0.18}$ \\
\bottomrule[1.5pt]
\end{tabular}
\\(b) generator 2
\end{table*}

\begin{table*}
\centering\caption{$k$-center clustering accuracy in NMI on randomly generated datasets
based on different geometries. The table shows
the mean and standard deviation after $300$ independent runs for each configuration.
$\rho$ is the distance measure. $n$ is the sample size.
$d$ is the dimensionality of the statistical simplex.
$\sigma$ is noise level.}\label{tbl:results}
\begin{tabular}{c|c|c|c|ccccc}
\toprule[1.5pt]
$k$ & $n$ & $d$ & $\sigma$ & $\rhofhr$ & $\rho_{\mathrm{IG}}$ & $\rhohg$ & $\rhoeuc$ & $\rho_{L1}$\\\hline
\multirow{12}{*}{3}
&\multirow{6}{*}{50} &\multirow{3}{*}{$9$}
  & $0.5$ & $0.87\pm0.19$ & $0.85\pm0.19$ & $\bm{0.92\pm0.16}$ & $0.72\pm0.22$ & $0.80\pm0.20$ \\
&&& $0.9$ & $0.54\pm0.21$ & $0.51\pm0.21$ & $\bm{0.70\pm0.23}$ & $0.36\pm0.17$ & $0.44\pm0.19$ \\\cline{3-9}
&&\multirow{3}{*}{$255$}
  & $0.5$ & $0.93\pm0.16$ & $0.92\pm0.18$ & $\bm{0.95\pm0.14}$ & $0.89\pm0.18$ & $0.90\pm0.19$ \\
&&& $0.9$ & $0.76\pm0.24$ & $0.72\pm0.26$ & $\bm{0.82\pm0.24}$ & $0.50\pm0.28$ & $0.76\pm0.25$ \\\cline{2-9}
&\multirow{6}{*}{100}
&\multirow{3}{*}{$9$}
 &  $0.5$ & $0.88\pm0.17$ & $0.86\pm0.18$ & $\bm{0.93\pm0.14}$ & $0.70\pm0.20$ & $0.80\pm0.20$ \\
&&& $0.9$ & $0.53\pm0.20$ & $0.49\pm0.19$ & $\bm{0.70\pm0.22}$ & $0.33\pm0.14$ & $0.41\pm0.18$ \\\cline{3-9}
&& \multirow{3}{*}{$255$}
  & $0.5$ & $0.93\pm0.16$ & $0.92\pm0.17$ & $\bm{0.95\pm0.13}$ & $0.88\pm0.19$ & $0.93\pm0.16$ \\
&&& $0.9$ & $0.81\pm0.22$ & $0.75\pm0.24$ & $\bm{0.83\pm0.22}$ & $0.47\pm0.28$ & $0.79\pm0.22$ \\\cline{1-9}
\multirow{12}{*}{5}
& \multirow{6}{*}{50}
& \multirow{3}{*}{9}
  & $0.5$ & $0.82\pm0.13$ & $0.81\pm0.13$ & $\bm{0.89\pm0.12}$ & $0.67\pm0.13$ & $0.75\pm0.13$ \\
&&& $0.9$ & $0.50\pm0.13$ & $0.47\pm0.13$ & $\bm{0.66\pm0.15}$ & $0.34\pm0.11$ & $0.40\pm0.12$ \\\cline{3-9}
&& \multirow{3}{*}{$255$}
  & $0.5$ & $\bm{0.92\pm0.11}$ & $\bm{0.91\pm0.12}$ & $\bm{0.93\pm0.11}$ & $0.87\pm0.13$ & $\bm{0.92\pm0.12}$ \\
&&& $0.9$ & $0.77\pm0.15$ & $0.71\pm0.17$ & $\bm{0.85\pm0.17}$ & $0.54\pm0.19$ & $0.74\pm0.16$ \\\cline{2-9}
&\multirow{6}{*}{100}
& \multirow{3}{*}{$9$}
  & $0.5$ & $0.83\pm0.12$ & $0.81\pm0.13$ & $\bm{0.89\pm0.11}$ & $0.67\pm0.11$ & $0.76\pm0.13$ \\
&&& $0.9$ & $0.48\pm0.12$ & $0.46\pm0.12$ & $\bm{0.66\pm0.15}$ & $0.33\pm0.09$ & $0.39\pm0.10$ \\\cline{3-9}
&& \multirow{3}{*}{$255$}
  & $0.5$ & $\bm{0.93\pm0.10}$ & $\bm{0.92\pm0.11}$ & $\bm{0.94\pm0.09}$ & $0.89\pm0.11$ & $0.92\pm0.11$ \\
&&& $0.9$ & $0.81\pm0.14$ & $0.74\pm0.15$ & $\bm{0.84\pm0.16}$ & $0.52\pm0.19$ & $0.79\pm0.14$ \\
\bottomrule[1.5pt]
\end{tabular}
\\(a) generator 1\\\vspace{1em}
\begin{tabular}{c|c|c|c|ccccc}
\toprule[1.5pt]
$k$ & $n$ & $d$ & $\sigma$ & $\rhofhr$ & $\rho_{\mathrm{IG}}$ & $\rhohg$ & $\rhoeuc$ & $\rho_{L1}$\\\hline
\multirow{12}{*}{3}
&\multirow{6}{*}{50} &\multirow{3}{*}{$9$}
  & $0.5$ & $0.68\pm0.22$ & $0.67\pm0.22$ & $\bm{0.80\pm0.20}$ & $0.48\pm0.22$ & $0.60\pm0.22$ \\
&&& $0.9$ & $0.32\pm0.18$ & $0.29\pm0.17$ & $\bm{0.45\pm0.21}$ & $0.20\pm0.14$ & $0.26\pm0.15$ \\\cline{3-9}
&&\multirow{3}{*}{$255$}
  & $0.5$ & $0.79\pm0.24$ & $0.75\pm0.24$ & $\bm{0.82\pm0.22}$ & $0.13\pm0.23$ & $\bm{0.81\pm0.24}$ \\
&&& $0.9$ & $0.35\pm0.27$ & $0.35\pm0.21$ & $\bm{0.42\pm0.19}$ & $0.00\pm0.02$ & $0.32\pm0.30$ \\\cline{2-9}
&\multirow{6}{*}{100}
&\multirow{3}{*}{$9$}
 &  $0.5$ & $0.66\pm0.22$ & $0.65\pm0.22$ & $\bm{0.79\pm0.21}$ & $0.45\pm0.19$ & $0.59\pm0.20$ \\
&&& $0.9$ & $0.30\pm0.16$ & $0.28\pm0.14$ & $\bm{0.42\pm0.19}$ & $0.20\pm0.12$ & $0.26\pm0.14$ \\\cline{3-9}
&& \multirow{3}{*}{$255$}
  & $0.5$ & $0.78\pm0.25$ & $0.76\pm0.24$ & $\bm{0.82\pm0.21}$ & $0.05\pm0.14$ & $0.77\pm0.27$ \\
&&& $0.9$ & $0.29\pm0.28$ & $0.29\pm0.20$ & $\bm{0.39\pm0.20}$ & $0.00\pm0.02$ & $0.22\pm0.25$ \\\cline{1-9}
\multirow{12}{*}{5}
& \multirow{6}{*}{50}
& \multirow{3}{*}{9}
  & $0.5$ & $0.69\pm0.14$ & $0.66\pm0.14$ & $\bm{0.77\pm0.13}$ & $0.50\pm0.13$ & $0.61\pm0.14$ \\
&&& $0.9$ & $0.34\pm0.12$ & $0.30\pm0.12$ & $\bm{0.46\pm0.15}$ & $0.22\pm0.09$ & $0.28\pm0.10$ \\\cline{3-9}
&& \multirow{3}{*}{$255$}
  & $0.5$ & $\bm{0.80\pm0.15}$ & $0.76\pm0.15$ & $\bm{0.82\pm0.14}$ & $0.24\pm0.23$ & $\bm{0.81\pm0.14}$ \\
&&& $0.9$ & $0.42\pm0.21$ & $0.38\pm0.16$ & $\bm{0.46\pm0.15}$ & $0.00\pm0.02$ & $0.39\pm0.22$ \\\cline{2-9}
&\multirow{6}{*}{100}
& \multirow{3}{*}{$9$}
  & $0.5$ & $0.66\pm0.13$ & $0.64\pm0.14$ & $\bm{0.77\pm0.14}$ & $0.47\pm0.13$ & $0.57\pm0.13$ \\
&&& $0.9$ & $0.31\pm0.11$ & $0.28\pm0.10$ & $\bm{0.44\pm0.13}$ & $0.21\pm0.08$ & $0.25\pm0.09$ \\\cline{3-9}
&& \multirow{3}{*}{$255$}
  & $0.5$ & $\bm{0.80\pm0.16}$ & $0.76\pm0.15$ & $\bm{0.82\pm0.13}$ & $0.12\pm0.17$ & $\bm{0.81\pm0.16}$ \\
&&& $0.9$ & $0.32\pm0.19$ & $0.30\pm0.15$ & $\bm{0.41\pm0.13}$ & $0.00\pm0.01$ & $0.26\pm0.18$ \\
\bottomrule[1.5pt]
\end{tabular}
\\(b) generator 2
\end{table*}

We also benchmark the $k$-means++ clustering on {\em positive} measures (not necessarily normalized).
The experimental results are reported in Table~\ref{tab:eKLBG}.
Divergences $\rho_{\mathrm{IG}^+}$, $\rho_{\mathrm{rIG}^+}$  and $\rho_{\mathrm{sIG}^+}$ are the extended Kullback-Leibler divergence, the extended reverse Kullback-Leibler divergence and the extended symmetrized Kullback-Leibler divergence, respectively:

\begin{eqnarray}
\rho_{\mathrm{IG}^+}(p,q) &=& \sum_{i=0}^d \lambda_p^i \log \frac{\lambda_p^i}{\lambda_q^i}+\lambda_q^i-\lambda_p^i,\\
\rho_{\mathrm{rIG}^+}(p,q) &=& \rho_{\mathrm{IG}}(q,p) = \sum_{i=0}^d \lambda_q^i \log \frac{\lambda_q^i}{\lambda_p^i}+\lambda_p^i-\lambda_q^i,\\
\rho_{\mathrm{sIG}^+}&=&\rho_{\mathrm{IG}}(p,q)+\rho_{\mathrm{IG}}(q,p)=(\lambda_p^i-\lambda_q^i) \log \frac{\lambda_p^i}{\lambda_q^i}.
\end{eqnarray}

Table~\ref{tab:eKLBG} shows experimentally better results for the Birkhoff  metric on the standard positive cone:
\begin{equation}
\rho_{\mathrm{BG}}(p,q) = \log\max_{i\in\{0,\ldots, d\},j\in\{0,\ldots, d\}} \frac{\lambda_p^i\lambda_q^j}{\lambda_p^j\lambda_q^i}.
\end{equation}

\begin{table}
\centering
\caption{%
$k$-means\texttt{++} clustering of $N=50$ random samples of $d=10$ dimensional {\em positive} measures.
$k$ is the ground truth number of clusters, $\sigma$ is the noise level.
The data is generated similarly to the previous experiments with each
random sample $p$ multiplied by a random scalar distributed based on a gamma distribution $\Gamma(10,0.1)$.
For each configuration we repeat $300$ runs and report the mean and standard
deviation of the NMI score.\label{tab:eKLBG}}
\begin{tabular}{cccccc}
\hline
$k$ & $\sigma$ & $\rho_{\mathrm{IG}^+}$ & $\rho_{\mathrm{rIG}^+}$ & $\rho_{\mathrm{sIG}^+}$ & $\rho_{\mathrm{BG}}$ \\
\hline
3 & 0.5 & $0.66\pm0.21$ & $0.67\pm0.20$ & $0.66\pm0.21$ & $\bm{0.86\pm0.18}$ \\
3 & 0.9 & $0.37\pm0.16$ & $0.38\pm0.19$ & $0.37\pm0.19$ & $\bm{0.62\pm0.22}$ \\
5 & 0.5 & $0.68\pm0.13$ & $0.68\pm0.12$ & $0.70\pm0.12$ & $\bm{0.85\pm0.11}$ \\
5 & 0.9 & $0.42\pm0.10$ & $0.43\pm0.11$ & $0.44\pm0.12$ & $\bm{0.63\pm0.13}$\\
\hline
\end{tabular}
\end{table}

\section{Hilbert geometry of the space of correlation matrices\label{sec:elliptope}}

In this section, we present the Hilbert geometry of the space of correlation matrices
$$
\mathcal{C}^d = \left\{ C_{d\times{d}}\,:\,C\succ0; C_{ii}=1, \forall{i} \right\}.
$$
If $C_1,C_2\in\mathcal{C}$, then $C_\lambda=(1-\lambda) C_1+\lambda C_2\in\mathcal{C}$ for $0\leq \lambda\leq 1$.
Therefore $\mathcal{C}$ is a convex set, known as an \emph{elliptope}~\cite{CorrelationElliptope-2018} embedded in the cone $\calP_+$ of positive semi-definite matrices.
See Fig.~(\ref{fig:elliptope}) for an intuitive 3D rendering of $\mathcal{C}_3$,
where the coordinate system $(x,y,z)$ is the off-diagonal entries of $C\in\mathcal{C}_3$.
The boundary of $\mathcal{C}^3$ is related to a Cayley symmetroid surface (or Cayley cubic surface) in algebraic geometry~\cite{Nie-2010}.
The elliptope is smooth except at the standard simplex vertices.

\begin{figure}
\centering
\includegraphics[width=.4\textwidth]{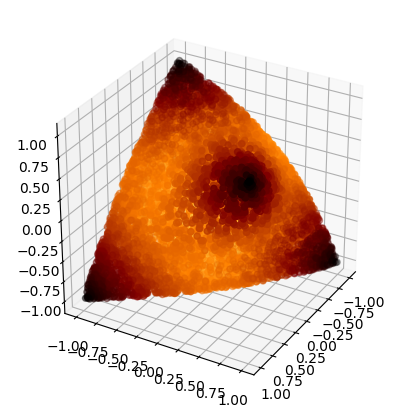}
\includegraphics[width=.4\textwidth]{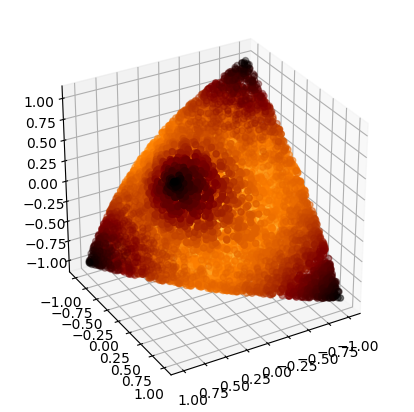}\vskip 0.5cm
\includegraphics[width=.6\textwidth]{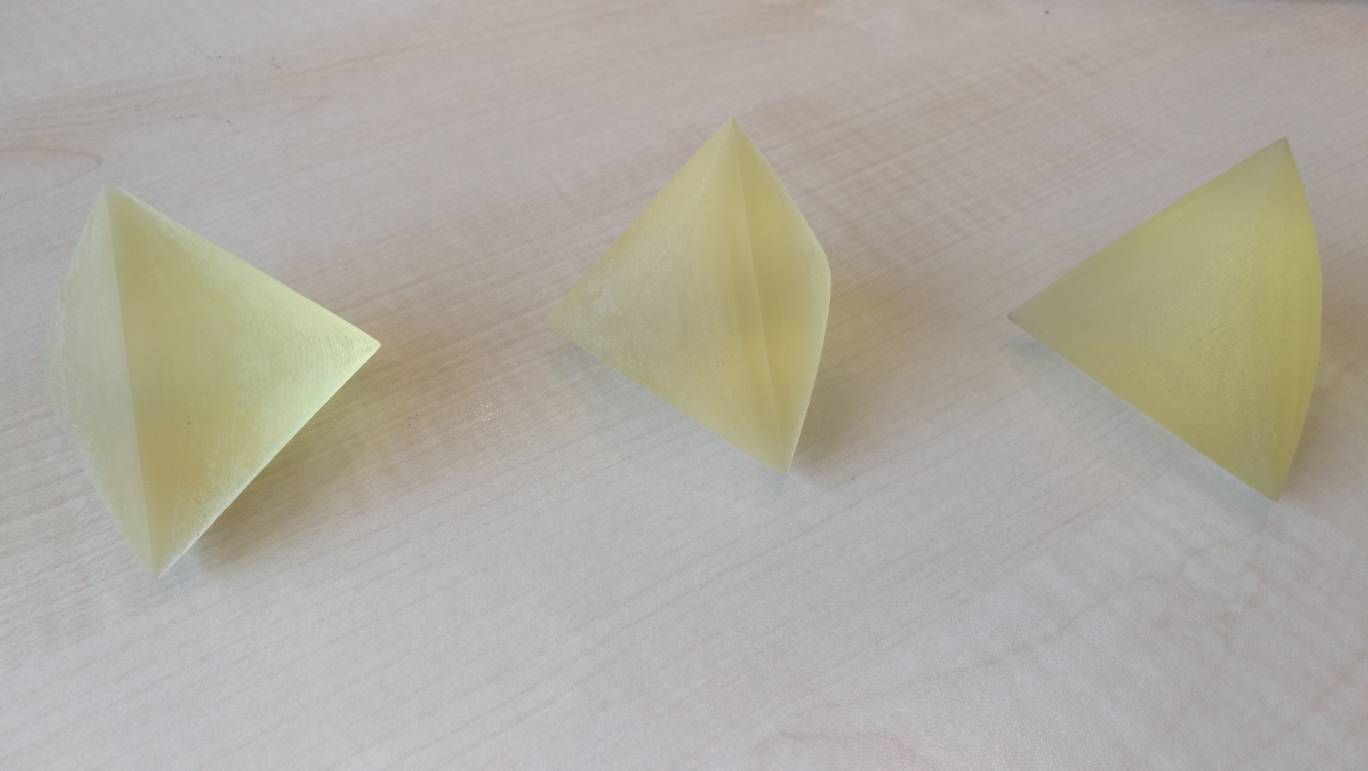}

\caption{The elliptope $\mathcal{C}_3$ (two different perspectives), and three 3D printed elliptopes.}\label{fig:elliptope}
\end{figure}

A straightforward algorithm to compute the Hilbert distance $\rhohg(C_1,C_2)$ consists in computing
the intersection of the line $(C_1,C_2)$ with
$\partial\mathcal{C}$, denoted as $C_1'$ and $C_2'$.
Then we have
$$
\rhohg(C_1,C_2)=
\left\vert
\log\frac{\Vert{}C_1-C_2'\Vert \Vert{}C_1'-C_2\Vert}{\Vert{}C_1-C_1'\Vert \Vert{}C_2-C_2'\Vert}
\right\vert,
$$
where $\|X\|=\sqrt{\inner{X}{X}}$ with the matrix inner product, the {\em trace inner product}: $\inner{X}{Y}=\tr(X^\top Y)$.

Thus a naive algorithm consists in applying a binary searching algorithm.
Note that a necessary condition for $C\in\mathcal{C}$ is that $C$ has a positive spectrum (all positive eigenvalues):
Indeed, if $C$ has at least one non-positive eigenvalue then $C\notin\mathcal{C}$.
Thus to determine whether a given matrix $C$ is inside the elliptope or not requires a spectral
decomposition of $C$ (cubic time~\cite{Eigenvalue-1999}). Therefore the computation of $C_1'$ and $C_2'$ is in general expensive.
In practice, we can approximate the largest eigenvalue using the iterative power method (and the smallest eigenvalue by applying the power method on the inverse matrix).

The problem with this bisection scheme is that we only find a correlation matrix that is slightly off the boundary of the elliptope,
but not perfectly on the boundary, as it is required for computing the log cross-ratio formula of Hilbert distances.
However, we get an interval range where  the true Hilbert distance lies as follows:
Let $\bar{P}^+$  and $\bar{P}^-$ denote the inside and outside points closed to border for $\lambda<0$, respectively.
Similarly, let $\bar{Q}^+$  and $\bar{Q}^-$ denote the inside and outside points closed to border for $\lambda>0$, respectively.
Then we use the monotonicity of Hilbert distances by nested containments,
$\rho_{\HG}^{\calC}(p,q) >  \rho_{\calC'}^{\calC'}(p,q)$ for $\calC\subset\calC'$, to get a guaranteed approximation interval of the Hilbert distance.

We describe two techniques that improve over this naive method:
(1) an approximation technique using an exact formula for polytopal Hilbert geometries, and
(2) an exact formula for the Birkhoff's projective metric using a whitening transformation for covariance matrices.

\begin{enumerate}
	\item Let us approximate the elliptope $\calC^d$ by a polytope $\calP$ as follows:
	First, consider a correlation matrix $C$ as a point in dimension $D=\frac{d(d-1)}{2}$ using half-vectorization of off-diagonal elements of the matrix~\cite{SymLowner-2017}.
 Then compute the convex hull of $s$ sample correlation matrix points $p_1,\ldots, p_s$ (e.g., using {\tt qhull}~\cite{Quickhull-1996}) to get a bounded polytope  $\calP$. Let $m$ denote the number of facets of $\mathrm{conv}(p_1,\ldots, p_s)$.
Figure~\ref{fig:elliptope} displays a photo of three 3D printed elliptopes using this discretization method.
Finally, apply the closed-form equation of Proposition~3.1 of~\cite{HGPolytope-2014} which holds for {\em any} polytopal Hilbert geometry:

\begin{equation}\label{bgpolytope}
\rho_{\HG}(c_1,c_2)= \max_{i,j\in [m]} \log \frac{L_i(c_1)L_j(c_2)}{L_j(c_1)L_i(c_2)},
\end{equation}
where the interior of the polytope is expressed as:
\begin{equation}
\calP = \{ x\in\bbR^D \st L_i(x)>0,\ i\in [m]\}.
\end{equation}
That is, the equations $L_i(x)$'s define the supporting hyperplanes of the polytope $\calP$.
Notice that the closed-form formula of Eq.~\ref{bgpolytope} generalizes the formula of Eq.~\ref{eq:bgposcone} for the positive orthant cone.

\item Exact method for the  Hilbert/Birkhoff projective metric of covariance matrices:

Let $\calP_+$ denote the pointed cone of positive semi-definite matrices.
Cone $\calP_+$ defines a partial order $\preceq$ (called L\"owner ordering~\cite{SymLowner-2017}):
$\Sigma_1 \preceq \Sigma_2$ iff. $\Sigma_2-\Sigma_1\in C$ (i.e, positive semi-definite, notationally written as $\Sigma_2-\Sigma_1\succeq 0$).
The extreme rays of the cone are matrices of rank $1$ which can be written as $\lambda v^\top v$ for $\lambda>0$ and $v\not=0\in\bbR^d$~\cite{ConePSD-2011}.

In the remainder, we consider positive definite matrices belonging to the interior $\calP_{++}$ of the cone $\calP_+$.
We say that matrix $\Sigma_2$ dominates $\Sigma_1$ iff. there exist $\alpha\in\bbR$ and $\beta\in\bbR$
such that $\alpha \Sigma_2 \preceq \Sigma_1 \preceq \beta \Sigma_2$.
We define the equivalence $\Sigma_1 \sim \Sigma_2$ iff. $\Sigma_2$ dominates $\Sigma_1$ {\em and} $\Sigma_1$ dominates $\Sigma_2$.

Let
$$
M(\Sigma_1,\Sigma_2) \eqdef \inf\{\beta\in\bbR \st \Sigma_1 \preceq \beta \Sigma_2\},
$$
and
$$
m(\Sigma_1,\Sigma_2) \eqdef \sup\{\alpha\in\bbR \st \alpha \Sigma_2 \preceq \Sigma_1\}.
$$
The Birkhoff's metric is defined as the following distance:
\begin{equation}
\rho_{\BG}(\Sigma_1,\Sigma_2) = \left\{
\begin{array}{ll}
\log \frac{M(\Sigma_1,\Sigma_2)}{m(\Sigma_1,\Sigma_2)} & \Sigma_1 \sim \Sigma_2,\\
0 & \mbox{$\Sigma_1=0$ and $\Sigma_2=0$},\\
\infty & \mbox{otherwise}
\end{array}
\right.
\end{equation}
The Birkhoff's metric $\rho_{\BG}$ is a projective distance:
$\rho_{\BG}\rho(\Sigma_1,\Sigma_2)=\rho_{\BG}(\alpha\Sigma_1,\beta\Sigma_2)$ for any $\alpha,\beta>0$.
This metric is also called the Hilbert's projective metric.
Since $\Sigma_1\succeq \Sigma_2 \Leftrightarrow \Sigma_2^{-1}\succeq \Sigma_1^{-1}$,
we have $m(\Sigma_1,\Sigma_2)=M(\Sigma_2,\Sigma_1)^{-1}$, and it comes that
$$
\rho_{\BG}(\Sigma_1,\Sigma_2) = \log M(\Sigma_1,\Sigma_2)M(\Sigma_2,\Sigma_1),
$$
for $\Sigma_1 \sim \Sigma_2$.

Consider the asymmetric distance $\rho_F$:
\begin{equation}
\rho_F(\Sigma_1,\Sigma_2)\eqdef \log M(\Sigma_1,\Sigma_2).
\end{equation}
This is the {Funk weak metric} that satisfies the triangle inequality but it is not a symmetric metric distance.
Then Birkhoff's projective metric is interpreted as the symmetrization of Funk weak metric:
$$
\rho_{\BG}(\Sigma_1,\Sigma_2)=\rho_{F}(\Sigma_1,\Sigma_2)+\rho_{F}(\Sigma_2,\Sigma_1).
$$


Now, let us consider the Birkhoff's projective metric between
a $d\times d$ positive-definite diagonal matrix $\Sigma_1=D=\diag(\lambda_1,\ldots,\lambda_d)$ and the
 identity matrix $\Sigma_2=I=\diag(1,\ldots,1)$.
We have:

$$
M(D,I) = \inf\{\beta\in\bbR \st  \beta I-D\succeq 0 \},
$$

That is, we have $M(D,I)=\min_\beta  \{\beta-\lambda_i\geq 0 \st \forall i\in [d]\}$.
That is, $\beta=\max_{i\in [d]} \lambda_i=\lambda_\max(D)=\max_i\lambda_i$, where $\lambda_\max(D)$ denote the maximal eigenvalue of matrix $D$.
Similarly, we have
$$
m(D,I) = \sup\{\alpha\in\bbR \st   D-\alpha I\succ 0\}.
$$
That is, $m(D,I)=\lambda_\min(D)$, where $\lambda_\min(D)$ denote the minimal eigenvalue of matrix $D$.
Observe that $\lambda_\min(D)=\lambda_\max(D^{-1})$, with $D^{-1}=\diag(\frac{1}{\lambda_1},\ldots,\frac{1}{\lambda_d})$.
Thus, we check that $M(D,I)=\lambda_\max(D)=M(I,D)^{-1}$.

For general symmetric positive-definite covariance matrix $\Sigma$ and $\Sigma'$, we first perform a joint diagonalization of these matrices using the whitening transformation (see~\cite{Fukunaga-2013} p. 31  and Fig 2-3)
so that we obtain $\Sigma=\Sigma_1^{-1}\Sigma_2$ and $\Sigma'=I$.
We have $\Sigma \sim \Sigma'$.
We get the closed-form formula for the Birkhoff projective covariance metric as:

\begin{equation}\label{eq:canhgf}
\rho_{\BG}(\Sigma_1,\Sigma_2) =   \log \frac{\lambda_\max(\Sigma_1^{-1}\Sigma_2)}{\lambda_\min(\Sigma_1^{-1}\Sigma_2)} = \|\log \lambda(\Sigma_1^{-1}\Sigma_2) \|_{\var},
\end{equation}
where  $\lambda(X)$, $\lambda_\min(X)$ and $\lambda_\max(X)$ denote the eigenvalues, and the smallest and largest (real) eigenvalues of matrix $X$, respectively, and $\|\cdot\|_{\var}$ denotes the variation norm.
Note that we have $\rho_{\BG}(\alpha_1\Sigma_1,\alpha_2\Sigma_2) = \rho_{\BG}(\Sigma_1,\Sigma_2)$ for any $\alpha_1,\alpha_2>0$ since $\lambda_\max(\alpha\Sigma)=\alpha\lambda_\max(\Sigma)$ and $\lambda_\min(\alpha\Sigma)=\alpha\lambda_\min(\Sigma)$.
Moreover, we have $\rho_{\BG}(\Sigma_1,\Sigma_2) =0$ iff. $\Sigma_1=\lambda\Sigma_2$ for any $\lambda>0$.
This Hilbert elliptope distance of Eq.~\ref{eq:canhgf} was also reported in~\cite{bonnabel2011contraction,georgiou2015positive} for symmetric positive-definite matrices.

Let us state the invariance properties~\cite{InvarianceCovariance-1991} of this matrix metric:

\begin{property}[Invariance]
We have $\rho_{\HG}(C_1,C_2)=\rho_{\HG}(C_1^{-1},C_2^{-1})$
and $\rho_{\HG}(AC_1B,AC_2B)=\rho_{\HG}(C_1,C_2)$ for any $A,B\in\mathrm{GL}(d)$.
\end{property}

\begin{proof}
Let $\lambda(X)=\{\lambda_1,\ldots,\lambda_d\}$ so that
$\lambda(X^{-1})=\{\frac{1}{\lambda_1},\ldots,\frac{1}{\lambda_d}\}$.
We have $\lambda_\max(X^{-1})=\frac{1}{\lambda_\min(X)}$ and $\lambda_\min(X^{-1})=\frac{1}{\lambda_\max(X)}$.
It follows that
$\frac{\lambda_\max(C_1C_2^{-1})}{\lambda_\min(C_1C_2^{-1})}=\frac{\lambda_\max(C_1^{-1}C_2)}{\lambda_\min(C_1^{-1}C_2)}$,
and therefore $\rho_{\HG}(C_1,C_2)=\rho_{\HG}(C_1^{-1},C_2^{-1})$.

Now, let us notice that $\rho_{\HG}(C_1,C_2)=\rho_{\HG}(GC_1,GC_2)$ for any invertible transformation $G\in\mathrm{GL}(d)$
since $(GC_1)^{-1}GC_2=C_1^{-1}(G^{-1}G_2)C_2=C_1^{-1}C_2$.
Similarly, since $\rho_{\HG}(C_1,C_2)=\rho_{\HG}(C_1^{-1},C_2^{-1})$, we have
 $\rho_{\HG}(C_1G,C_2G)=\rho_{\HG}(G^{-1}C_1^{-1},G^{-1}C_2^{-1})=\rho(C_1^{-1},C_2^{-1})=\rho_{\HG}(C_1,C_2)$ since $G^{-1}\in\mathrm{GL}(d)$.
Thus in general, for $A,B\in\mathrm{GL}(d)$, it holds that $\rho_{\HG}(AC_1B,AC_2B)=\rho_{\HG}(C_1,C_2)$.
\end{proof}

It follows that $\rho_{\HG}(C_1,C_2)=\rho_{\BG}(C_1,C_2)=\rho_{\BG}(I,C_1^{-1}C_2)=\rho_{\BG}(I,C_2C_1^{-1})$.
Notice that two diagonal matrices $D_1$ and $D_2$ are equal iff. $\lambda_\max(D_1^{-1}D_2)=\lambda_\min(D_1^{-1}D_2)=1$, and therefore we may not need to consider the in-between eigenvalues to define a discrepancy measure.

The inverse of a correlation matrix, its unique square root, the product of two correlation matrices, or the product of a correlation matrix with an orthogonal matrix are  in general not a correlation matrix but only a covariance matrix.
Thus $C_1^{-1}C_2$ is not a correlation matrix but a covariance matrix.




The Birkhoff's projective metric formula of Eq.~\ref{eq:canhgf} applies to correlation matrices and yields the proper Hilbert's elliptope metric:
\begin{equation}\label{eq:canhg2}
\rho_{\HG}(C_1,C_2) = \log \frac{\lambda_\max(C_1^{-1}C_2)}{\lambda_\min(C_1^{-1}C_2)} = \|\log \lambda(C_1^{-1}C_2)) \|_{\var},
\end{equation}
with $\rho_{\HG}(C_1,C_2)=0$ iff $C_1=C_2$.
Note that the elliptope is a subspace of dimension $\frac{d(d-1)}{2}$ of the positive semi-definite cone of dimension $\frac{d(d+1)}{2}$.

A Python snippet code is reported in~\ref{sec:python}.

\end{enumerate}

We refer to~\cite{ClusteringCorrelMat-1975,ClusteringCorrelMat-2005,wgcna-2008,ClusteringCorrelMat-2018} for methods and applications dealing with the clustering of correlation matrices.

We compare the Hilbert elliptope geometry~\cite{GeometryCorrelation-2008} with commonly used distance measures~\cite{DistanceCorrelation-2017}
including the $L_2$ distance $\rhoeuc$, $L_1$ distance $\rhol1$,
and the square root of the log-det divergence
$$
\rho_{\mathrm{LD}}(C_1,C_2) = \mathrm{tr}(C_1C_2^{-1}) - \log\vert{}C_1C_2^{-1}\vert - d.
$$
Due to the high computational complexity, we only investigate $k$-means++ clustering.
The investigated dataset consists of 100 matrices forming 3 clusters in $\mathcal{C}_3$ with
almost identical size.  Each cluster is independently generated according to
\begin{align}
&P\sim \mathcal{W}^{-1}(I_{3\times{3}}, \nu_1),\nonumber\\
&C_i\sim \mathcal{W}^{-1}(P,\nu_2),\nonumber
\end{align}
where $\mathcal{W}^{-1}(A,\nu)$ denotes the inverse Wishart distribution with
scale matrix $A$ and $\nu$ degrees of freedom, and $C_i$ is a point in the cluster
associated with $P$.
Table \ref{tbl:elliptope} shows the $k$-means++ clustering performance in terms of NMI.
Again Hilbert geometry is favorable as compared to alternatives, showing
that the good performance of Hilbert clustering is generalizable.
\begin{table}[t]
\centering
\caption{NMI (mean$\pm$std) of $k$-means++ clustering based on different distance measures in the elliptope
(500 independent runs)}
\label{tbl:elliptope}
\begin{tabular}{cc|cccc}
\toprule[1.5pt]
$\nu_1$ & $\nu_2$    &  $\rhohg$ & $\rhoeuc$ & $\rhol1$ & $\rho_{\mathrm{LD}}$ \\
\hline
$4$ & $10$ &  $\bm{0.62\pm0.22}$ & 0.57$\pm$0.21 & 0.56$\pm$0.22 & 0.58$\pm$0.22\\
$4$ & $30$ &  $\bm{0.85\pm0.18}$ & 0.80$\pm$0.20 & 0.81$\pm$0.19 & 0.82$\pm$0.20\\
$4$ & $50$ &  $\bm{0.89\pm0.17}$ & 0.87$\pm$0.17 & 0.86$\pm$0.18 & 0.88$\pm$0.18\\
\hline
$5$ & $10$ &  $\bm{0.50\pm0.21}$ & 0.49$\pm$0.21 & 0.48$\pm$0.20 & 0.47$\pm$0.21\\
$5$ & $30$ &  $\bm{0.77\pm0.20}$ & 0.75$\pm$0.21 & 0.75$\pm$0.21 & 0.75$\pm$0.21\\
$5$ & $50$ &  $\bm{0.84\pm0.19}$ & 0.82$\pm$0.19 & 0.82$\pm$0.20 & $\bm{0.84\pm0.18}$\\
\bottomrule[1.5pt]
\end{tabular}
\end{table}

We also consider the {\em Thompson metric distance}~\cite{Lemmens-2017} defined over the cone of positive-semidefinite matrices:
\begin{equation}
\rho_{\TG}(P,Q) \eqdef \max_i |\log \lambda_i(P^{-1}Q)|.
\end{equation}


Each psd. matrix $P_i$ belongs to a cluster $c$ which is generated based on the equation:
\begin{equation}
P_i = Q_c \mathrm{diag}(L_i) Q_c^\top + \sigma A_i A_i^\top,
\end{equation}
where
$Q_c$ is a random matrix of orthonormal columns,
$L_i$ follows a gamma distribution,
the entries of $(A_i)_{d\times{d}}$ follow iid. standard Gaussian distribution,
 and $\sigma$ is a noise level parameter.

\begin{table}
\centering
\caption{$k$-means\texttt{++} clustering performance in NMI
wrt different distances defined on the p.s.d. cone.
For each clustering experiment, $N=250$ random matrices are generated, each with size $2\times{2}$,
 forming $5$ clusters in the psd. cone.}\label{tab:tg}
\begin{tabular}{cccccc}
\hline
$L_i$ & $\sigma$ & $\rho_{\mathrm{IG}}$ & $\rho_{\mathrm{rIG}}$ & $\rho_{\mathrm{sIG}}$ & $\rho_{\mathrm{TG}}$ \\
\hline
$\Gamma(2,1)$ & 0.1 & $0.81\pm0.09$ & $\bm{0.82\pm0.10}$ & $0.81\pm0.10$ & $0.81\pm0.09$ \\
$\Gamma(2,1)$ & 0.3 & $0.51\pm0.11$ & $\bm{0.54\pm0.11}$ & $0.53\pm0.11$ & $0.52\pm0.11$ \\
$\Gamma(5,1)$ & 0.1 & $\bm{0.93\pm0.07}$ & $\bm{0.93\pm0.08}$ & ${0.92\pm0.08}$ & ${0.92\pm0.08}$ \\
$\Gamma(5,1)$ & 0.3 & $0.70\pm0.10$ & $\bm{0.72\pm0.12}$ & $0.71\pm0.10$ & $0.70\pm0.11$ \\
\hline
\end{tabular}
\end{table}

Table~\ref{tab:tg} displays the experimental results: Thompson metric does not improve over the asymmetric Kullback-Leibler divergence or its reverse divergence.


\section{Conclusion}\label{sec:con}

We introduced the Hilbert (projective) metric distance and its underlying non-Riemannian geometry
for modeling the space of multinomials or the open probability simplex.
We compared experimentally in simulated clustering tasks
this geometry with the traditional differential geometric modelings
(either the Fisher-Hotelling-Rao metric connection or the dually coupled non-metric affine connections of information geometry~\cite{IG-2016}).

The main feature of Hilbert geometry (HG) is that it is a metric non-manifold geometry, where geodesics are straight (Euclidean) line segments. This makes this geometry computationally attractive.
In simplex domains, the Hilbert balls have fixed combinatorial (Euclidean) polytope structures,
and HG is known to be isometric to a normed space~\cite{HilbertHarpe-1991,HilbertNormedSpace-2005}.
This latter isometry allows one to generalize easily the standard proofs of clustering ({\it e.g.}, $k$-means or $k$-center).
We demonstrated it for the $k$-means++ competitive performance analysis and for the convergence of the $1$-center heuristic~\cite{bc-2003} (smallest enclosing Hilbert ball allows one to implement efficiently the $k$-center clustering).
Our experimental  $k$-means++ or $k$-center comparisons of HG algorithms with the  manifold modeling approach yield superior performance.
This may be intuitively explained by the sharpness of Hilbert balls as compared to the FHR/IG ball profiles.
However, when considering the Hilbert simplex geometry, let us notice that we do not get a Busemann's geodesic space~\cite{Busemann-1955}
 (Busemann $G$-space) since the region:
\begin{equation}
R(p,q) \eqdef \left\{r \st \rho_\HG(p,q)=\rho_\HG(p,r)+\rho_\HG(q,r) \right\},
\end{equation}
defining the set of points $r$ such that
triples $(p,r,q)$ satisfy the triangle equality is {\em not} a line segment but rather the intersection of cones~\cite{HilbertHarpe-1991} as illustrated in Figure~\ref{fig:notgeodesic}.
We proved that the Hilbert simplex geometry satisfies the property of information monotonicity (Theorem~\ref{thm:hsgim}) albeit being non-separable, a key requirement in information geometry~\cite{IG-2016}.
The automorphism group (the group of motions) of the Hilbert simplex geometry is reported in~\cite{HilbertHarpe-1991}.

\begin{figure}
\centering
\includegraphics[width=.5\textwidth]{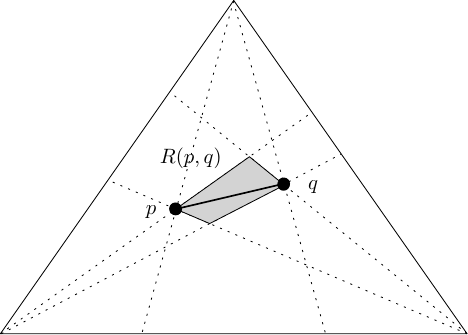}

\caption{The geodesics are {\em not} unique in a polygonal Hilbert geometry:
Region $R(p,q)$ denotes the set of points $r$ satisfying the triangle equality with respect to $p$ and $q$:
 $\rho_\HG(p,q)=\rho_\HG(p,r)+\rho_\HG(q,r)$.\label{fig:notgeodesic}}

\end{figure}

Chentsov~\cite{cencov-2000} defined statistical invariance on a probability manifold under Markov morphisms and proved that the Fisher Information Metric is the unique Riemannian metric (up to rescaling) for multinomials. However, this does not rule out that other distances (with underlying geometric structures) may be used to model statistical manifolds ({\it e.g.}, Finsler statistical manifolds~\cite{Cena-2002,FinslerIG-2016}, or the total variation distance --- the only metric $f$-divergence~\cite{L1metricfdiv-2007}). Defining statistical invariance related to geometry is the cornerstone problem of information geometry that can be tackled in many directions (see~\cite{statinvar-2017} and references therein for a short review).
The Hilbert cross-ratio metric is by construction invariant to the group of collineations.

In this paper, we introduced Hilbert geometries in machine learning by considering clustering tasks in the probability simplex and in the correlation elliptope. A canonical Hilbert metric distance can be defined on any bounded convex subset of the Euclidean space with
the key property that geodesics are straight Euclidean line segments thus making this geometry well-suited for fast and exact computations.
Thus we may consider clustering in other bounded convex subsets like the simplotopes~\cite{Simplotop-Doup-2012}.

Recently, hyperbolic geometry gained interests in the machine learning community as it enjoys the good property of isometrically embeddings tree or DAG structures with low distortions.
Ganea {\it et al.}~\cite{Hyperbolic-Ganea-2018} considered Poincar\'e embeddings.
Nickel and Kiela~\cite{LorentzEmbedding-2018} reported that learning hyperbolic embeddings in the Lorentz model is better than in the Poincar\'e ball model.
In~\cite{hMDS-2018}, a generalization of Multi-Dimensional Scaling (MDS) to hyperbolic space is given using the hyperboloid model of hyperbolic geometry.

Notice that Klein model of hyperbolic geometry (non-conformal) can be interpreted as a special case of  Hilbert geometry for the unit ball domain~\cite{HVD-2010,HVD-2014}. Gromov introduced the generic concept of $\delta$-hyperbolicity for any metric space~\cite{Gromov-1987}.
It is known that a necessary condition for Gromov-hyperbolic Hilbert geometry is to have a bounded convex $C^1$-domain without line segment~\cite{Karlsson-2002}.
(The characterization of hyperbolic Hilbert geometries have also been studied using the curvature viewpoint in~\cite{HyperbolicHilbert-2014}.)
See~\cite{Hyperbolic-Ganea-2018} for a generic low-distortion embedding theorem of a point set lying in a  Gromov-hyperbolic space into a weighted tree.

Let us summarize the pros and cons of the Hilbert polytopal geometries:
\begin{itemize}
	\item Advantage: Generic closed-form distance formula for polytope domains~\cite{Alexander-1978,Schneider-2006}, isometric to a normed vector space~\cite{HGPolytope-2014}.
	\item Inconvenient: Not a Busemann $G$-space (geodesics as shortest paths are not unique) with only polynomial volume growth of balls~\cite{Vernicos-2013}.
\end{itemize}

Figure~\ref{fig:HGchart} summarizes the various kinds of Hilbert geometries.

\begin{figure}%
\centering
\includegraphics[width=0.8\textwidth]{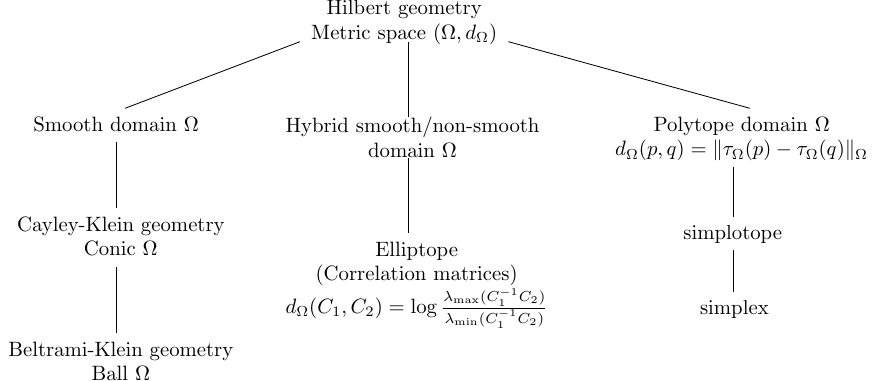}%
\caption{Various types of Hilbert geometries.}%
\label{fig:HGchart}%
\end{figure}

Potential future directions are to consider Hilbert structure embeddings, and using the Hilbert metric for regularization and sparsity in machine learning (due to its equivalence with a polytope normed distance).
\vskip 0.5cm

\vbox{
Our Python codes are freely available online for reproducible research:\\
\centerline{\url{https://franknielsen.github.io/HSG/}}
}

\bibliographystyle{plain}
\bibliography{HSG-1704.00454-v7}

\appendix

\section{Isometry of Hilbert simplex geometry to a normed vector space\label{sec:isometry}}

Consider the Hilbert simplex metric space $(\Delta^d,\rho_\HG)$ where $\Delta^d$ denotes
the $d$-dimensional open probability simplex and $\rho_\HG$ the Hilbert cross-ratio metric.
Let us recall the isometry (\cite{HilbertHarpe-1991}, 1991) of the open standard simplex to a normed vector space $(V^d,\|\cdot\|_\NH)$.
Let $V^d=\{v\in\bbR^{d+1} \st \sum_i v^i=0\}$ denote the $d$-dimensional vector space sitting in $\bbR^{d+1}$.
Map a point $p=(\lambda^0,\ldots,\lambda^{d})\in\Delta^d$ to a point $v(x)=(v^0,\ldots, v^{d})\in V^d$ as follows:
\begin{equation*}
v^i
= \frac{1}{d+1} \left(d\log \lambda^i -\sum_{j\neq{i}}\log \lambda^j \right)
= \log\lambda^i - \frac{1}{d+1}\sum_{j=0}^d \log\lambda^j.
\end{equation*}

We define the corresponding norm $\|\cdot\|_\NH$ in $V^d$ by considering the shape of its unit ball
$B_V=\{v\in V^d \st |v^i-v^j|\leq 1, \forall i\not =j\}$.
The unit ball $B_V$ is a symmetric convex set containing the origin in its interior, and thus yields a {\em polytope norm}
 $\|\cdot\|_\NH$ (Hilbert norm) with $2\binom{d+1}{2}=d(d+1)$ facets.
Reciprocally, let us notice that a norm induces a unit ball centered at the origin that is convex and symmetric around the origin.

The distance in the normed vector space between $v\in V^d$ and $v'\in V^d$ is defined by:
\begin{equation*}
\rho_V(v,v')= \|v-v'\|_\NH = \inf \left\{ \tau \st v'\in \tau(B_V\oplus \{v\}) \right\},
\end{equation*}
where $A\oplus B=\{a+b \st a\in A,b\in B\}$ is the Minkowski sum.

The reverse map from the normed space $V^d$ to the probability simplex $\Delta^d$ is given by:
\begin{equation*}
\lambda^i = \frac{\exp({v^i})}{\sum_j \exp(v^j)}.
\end{equation*}

Thus we have $(\Delta^d,\rho_\HG)\cong (V^d,\|\cdot\|_\NH)$.
In 1D, $(V^1,\|\cdot\|_\NH)$ is isometric to the Euclidean line.

Note that computing the distance in the normed vector space requires naively $O(d^2)$ time.

\begin{figure}
\centering
\includegraphics[width=.65\textwidth]{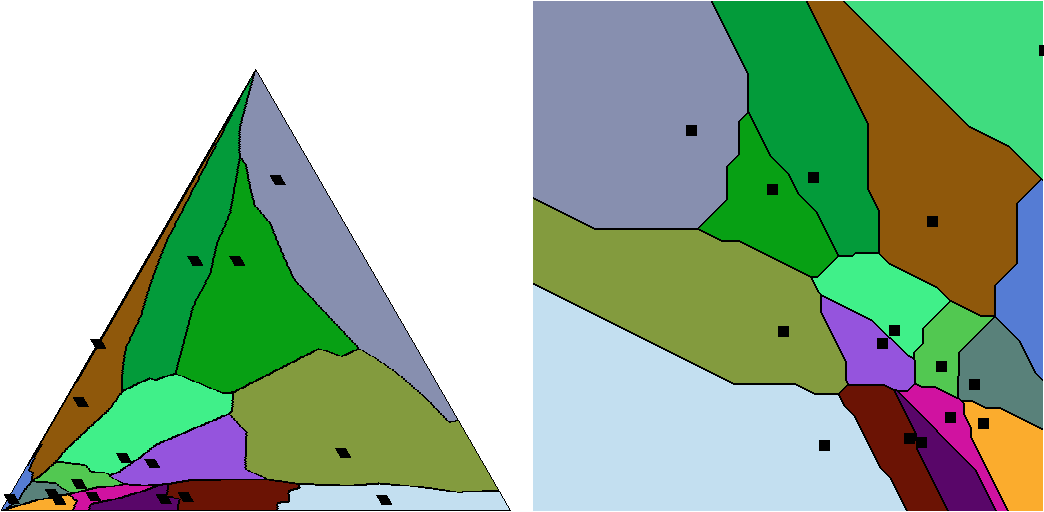}
\caption{Voronoi diagram of $n=16$ points in the Hilbert probability simplex and corresponding Voronoi diagram in the variation norm space.}\label{fig:VoronoiHilbertVarNorm}

\end{figure}

Figure~\ref{fig:VoronoiHilbertVarNorm} displays the Voronoi diagram of $n=16$ points in the Hilbert probability simplex,
and the equivalent Voronoi diagram in the normed space with respect to the variation norm-induced distance.

Let us notice that coordinate $v^i$ can be rewritten as
$$
v^i=\log \frac{\lambda^i}{G(\lambda)},
$$
where $G(\lambda)$ is the coordinate geometric means.

Recall that the Hilbert distance is a projective distance:
$$
\rho_{\HG}(p,q) = \log \frac{\max _{i\in\{1,\ldots, d\}} \frac{p_{i}}{q_{i}}}{\min _{j\in\{1,\ldots, d\}} \frac{p_{j}}{q_{j}}} =
\rho_{\HG}(\lambda p,\lambda' q),\forall \lambda>0,\lambda'>0.
$$

Thus we have:
$$
\rho_{\HG}(p,q) = \|\log p-\log q\|_{\mathrm{var}} = \|\log (\lambda p)-\log (\lambda' q)\|_{\mathrm{var}},\forall \lambda>0,\lambda'>0.
$$

Choose $\lambda=\frac{1}{G(p)}$ and $\lambda'=\frac{1}{G(q)}$ to get
$$
\rho_{\HG}(p,q) = \left\|\log \frac{p}{G(p)}-\log \frac{q}{G(q)}\right\|_{\mathrm{var}}.
$$

This highlights a nice connection with the Aitchison distance of Eq.~\ref{eq:AD}:

\begin{align}
\rhohg(p,q)
&= \left\Vert \log\frac{p}{G(p)}-\log\frac{q}{G(q)} \right\Vert_{\NH},\\
\rho_{\mathrm{Aitchison}}(p,q)
&= \left\Vert \log\frac{p}{G(p)}-\log\frac{q}{G(q)} \right\Vert_{2}.
\end{align}

Thus both the Aitchison distance and the Hilbert simplex distance are normed distances on the representation
$p\mapsto \log\frac{p}{G(p)}=\left(\log\frac{p^0}{G(p)},\ldots,\frac{p^d}{G(p)}\right)$.

Figure~\ref{fig:VoronoiAitchisonHilbertVarNorm} displays the Voronoi diagram of $n=16$ points in the probability simplex with respect to the Aitchison distance (Figure~\ref{fig:VoronoiAitchisonHilbertVarNorm}, left), and the Hilbert simplex distance (Figure~\ref{fig:VoronoiAitchisonHilbertVarNorm}, middle) and its equivalent variation norm space (Figure~\ref{fig:VoronoiAitchisonHilbertVarNorm}, right).

\begin{figure}
\centering
\includegraphics[width=.95\textwidth]{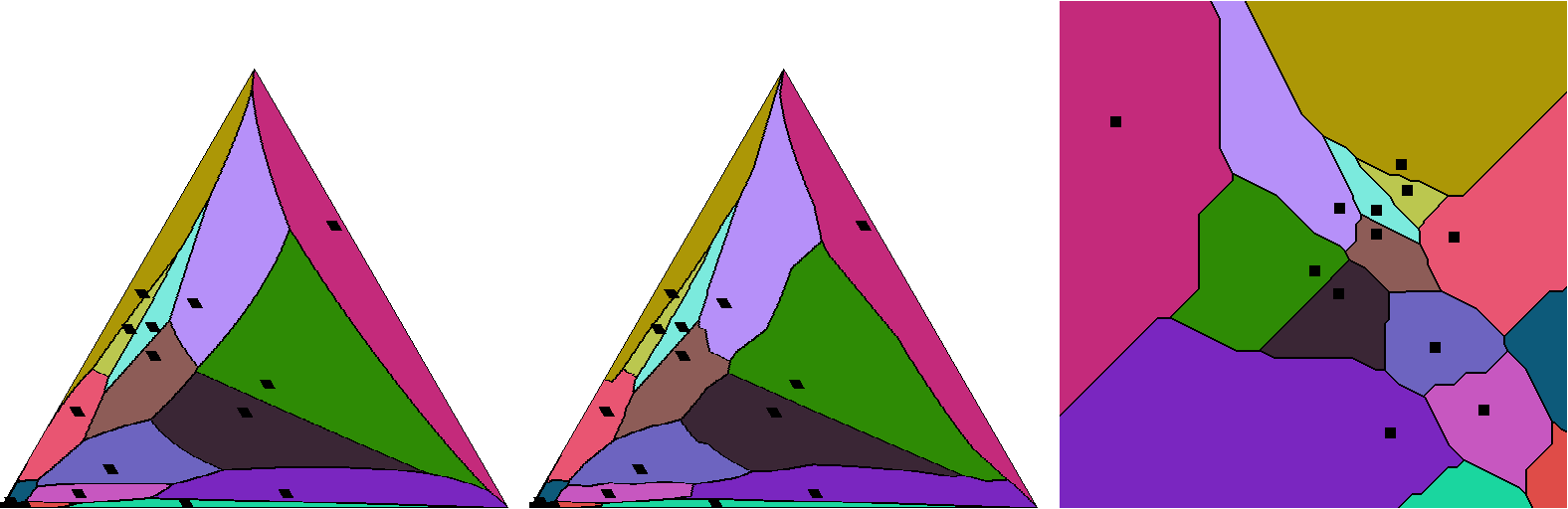}
\caption{Voronoi diagram in the probability simplex with respect to the Aitchison distance (left), Hilbert simplex distance (middle) and equivalent variation-norm induced distance on normalized logarithmic representations.}\label{fig:VoronoiAitchisonHilbertVarNorm}

\end{figure}

Unfortunately, the norm $\|\cdot\|_\NH$ does not satisfy the parallelogram law.\footnote{ Consider
$A = (1/3,1/3,1/3)$, $B = (1/6,1/2,1/3)$, $C = (1/6,2/3,1/6)$ and
$D = (1/3,1/2,1/6)$. Then  $2AB^2 +2BC^2 = 4.34$ but $AC^2 + BD^2 = 3.84362411135$.
}
Notice that a norm satisfying the parallelogram law can be associated with an inner product via the polarization identity.
Thus the isometry of the Hilbert geometry to a normed vector space is not equipped with an inner product.
However, all norms in a finite dimensional space are equivalent.
This implies that in finite dimension, $(\Delta^d,\rho_\HG)$ is {\em quasi-isometric} to the Euclidean space $\bbR^d$.
An example of Hilbert geometry in infinite dimension is reported in~\cite{HilbertHarpe-1991}.
Hilbert spaces are not CAT spaces except when $\calC$ is an ellipsoid~\cite{Vernicos-2004}.

\section{Hilbert geometry with Finslerian/Riemannian structures\label{sec:HGFG}}

In a Riemannian geometry, each tangent plane $T_pM$ of a $d$-dimensional manifold $M$ is equivalent to $\bbR^d$: $T_pM\simeq \bbR^d$.
The inner product at each tangent plane $T_pM$ can be visualized by an ellipsoid shape, a convex symmetric object centered at point $p$.
In a {\em Finslerian geometry},  a norm $\|\cdot\|_p$ is defined in each tangent plane $T_pM$,
and this norm is visualized as a symmetric convex object with non-empty interior.
Finslerian geometry thus generalizes Riemannian geometry by taking into account generic symmetric convex objects instead of ellipsoids for inducing norms at each tangent plane.
Any Hilbert geometry induced by a compact convex domain $\calC$ can be expressed by an equivalent Finslerian geometry by defining the norm in $T_p$ at $p$ as follows~\cite{Vernicos-2004}:

\begin{equation*}
\|v\|_p = F_\calC(p,v) =\frac{\|v\|}{2} \left( \frac{1}{pp^+} + \frac{1}{pp^-} \right),
\end{equation*}
where
$F_\calC$ is the {\em Finsler metric},
$\|\cdot\|$ is an {\em arbitrary norm} on $\bbR^d$,
and $p^+$ and $p^-$ are the intersection points of the line passing through $p$ with direction $v$:
$$
p^+=p+t^+v,\quad p^-=p+t^-v.
$$

A geodesic $\gamma$ in a Finslerian geometry satisfies:
\begin{equation*}
d_\calC(\gamma(t_1),\gamma(t_2)) = \int_{t_1}^{t_2} F_\calC(\gamma(t),\dot\gamma(t)) \dt .
\end{equation*}

In $T_pM$, a ball of center $c$ and radius $r$ is defined by:
\begin{equation*}
B(c,r)=\{ v \ : \ F_\calC(c,v) \leq r \}.
\end{equation*}

Thus any Hilbert geometry induces an equivalent Finslerian geometry, and since Finslerian geometries include Riemannian geometries, one may wonder which Hilbert geometries induce Riemannian structures?
The only Riemannian geometries induced by Hilbert geometries are the {\em hyperbolic Cayley-Klein geometries}~\cite{Richter-2011,LMNN-2016,CayleyClassification-2016} with the domain $\calC$ being an ellipsoid.
The Finslerian modeling of information geometry has been studied in~\cite{Cena-2002,FinslerIG-2016}.

There is not a canonical way of defining measures in a Hilbert geometry since Hilbert geometries are Finslerian but not necessary Riemannian geometries~\cite{Vernicos-2004}. The Busemann measure is defined according to the Lebesgue measure $\lambda$ of $\bbR^d$: Let $B_p$ denote the unit ball wrt. to the Finsler norm at point $p\in\calC$, and $B_e$ the Euclidean unit ball. Then the Busemann measure for a Borel set $\calB$ is defined by~\cite{Vernicos-2004}:
$$
\mu_\calC(\calB) = \int_\calB \frac{\lambda(B_e)}{\lambda(B_p)} \mathrm{d}\lambda(p).
$$

The existence and uniqueness of center points of a probability measure in Finsler geometry have been investigated in~\cite{FinslerCenter-2012}.

\section{Bounding Hilbert norm with other norms}
Let us show that $\|v\|_\NH\leq \beta_{d,c} \|v\|_c$, where $\|\cdot\|_c$
is any norm.
Let $v=\sum_{i=0}^{d} e_i x_i$, where $\{e_i\}$ is a basis of $\bbR^{d+1}$.
We have:
$$
\|v\|_c \leq \sum_{i=0}^{d} |x_i|  \|e_i\|_c \leq \|x\|_2
\underbrace{\sqrt{\sum_{i=0}^{d} \|e_i\|^2_c}}_{\beta_d},
$$
where the first inequality comes from the triangle inequality, and the
second inequality is from the Cauchy-Schwarz inequality.
Thus we have:
$$
\|v\|_\NH \leq \beta_d \|x\|_2,
$$
with $\beta_d=\sqrt{d+1}$ since $\|e_i\|_\NH\leq 1$.

Let $\alpha_{d,c}=\min_{\{v \st \|v\|_c = 1\}} \|v\|_\NH$.
Consider $u=\frac{v}{\|v\|_c}$. Then $\|u\|_c=1$ so that $\|v\|_\NH\geq
\alpha_{d,c} \|v\|_c$.
To find $\alpha_d$, we consider the unit $\ell_2$ ball in $V^d$, and
find the smallest $\lambda>0$ so that
$\lambda B_V$ fully contains the Euclidean ball (Figure~\ref{fig:boundnorm}).

\begin{figure}
\centering
\includegraphics[width=.5\textwidth]{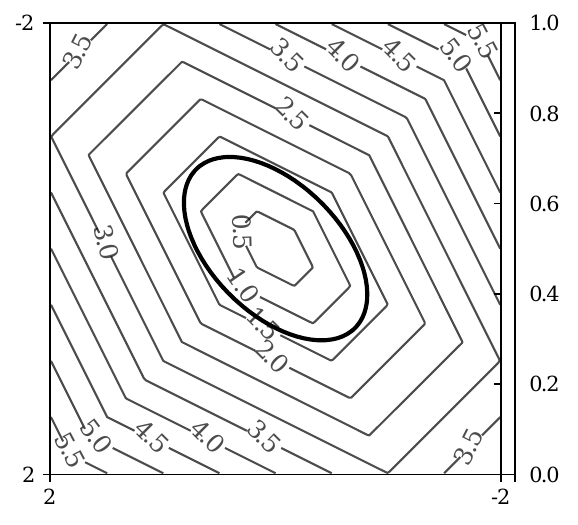}
\caption{Polytope balls $B_V$ and the Euclidean unit ball $B_E$.
From the figure the smallest polytope ball has a radius $\approx 1.5$.}\label{fig:boundnorm}
\end{figure}

Therefore, we have overall:

$$
\alpha_d \|x\|_2 \leq \|v\|_\NH \leq \sqrt{d+1} \|x\|_2
$$

In general, note that we may consider two arbitrary norms $\|\cdot\|_l$
and $\|\cdot\|_u$ so that:
$$
\alpha_{d,l} \|x\|_l \leq \|v\|_\NH \leq \beta_{d,u} \|x\|_u.
$$

\section{Funk directed metrics and Funk balls}

The Funk metric~\cite{FunkHilbert-2014}
with respect to a convex domain $\calC$ is defined by
$$
F_\calC(x,y) = \log\left( \frac{\|x-a\|}{\|y-a\|} \right),
$$
where $a$ is the intersection of
the domain boundary and the affine ray $R(x,y)$
starting from $x$ and passing through $y$
Correspondingly, the reverse Funk metric is
$$
F_\calC(y,x) = \log\left( \frac{\|y-b\|}{\|x-b\|} \right),
$$
where $b$ is the intersection of $R(y,x)$ with the boundary.
The Funk metric is \emph{not} a metric distance, but a {\em weak metric distance}~\cite{Funk-2009} (i.e., a metric without symmetry).

The Hilbert metric is simply the arithmetic symmetrization:
$$
H_\calC(x,y)=F_\calC(x,y)+F_\calC(y,x).
$$

Figure~\ref{fig:FunkSimplex} displays the Funk weak metric in the probability simplex.
The Funk Distance (FD, oriented distance) in the probability simplex is defined by
\begin{equation}
\rho_{\mathrm{FD}}(p,q):=\log \max_{i} \frac{p_i}{q_i}.
\end{equation}
Thus the Jeffreys-type arithmetic symmetrization of the Funk distance yields the Hilbert distance:
\begin{eqnarray*}
\rhofd(p,q)+\rhofd(q,p) &=& \log \max_{i} \frac{p_i}{q_i}\max_{i} \frac{q_i}{p_i},\\
&=& \log \frac{ \max_{i}\frac{p_i}{q_i}}{\min_{i} \frac{p_i}{q_i}},\\
&=& \rhohg(p,q).
\end{eqnarray*}

We checked experimentally that the Funk distance is information monotone.
The following lemma proves this property:

\begin{lemma}\label{thm:funkmono}
    Let $p,q\in\Delta^d$, then $\tilde{p}=(p_0+p_1,p_2,\cdots,p_d)$ and
    $\tilde{q}=(q_0+q_1,q_2,\cdots,q_d)$ are their coarse-grained points on
    $\Delta^{d-1}$. We have
    \begin{equation*}
        \rhofd(\tilde{p},\tilde{q}) \le \rhofd(p,q).
    \end{equation*}
\end{lemma}
\begin{proof}
    Denote $\iota=\max\{p_0/q_0, p_1/q_1\}$. As $q_0,q_1>0$, we have
    \begin{equation*}
        p_0 \le \iota q_0, \quad p_1 \le \iota q_1.
    \end{equation*}
    Therefore
    \begin{equation*}
        \frac{p_0+p_1}{q_0+q_1} \le \frac{\iota q_0+ \iota q_1}{q_0+q_1} = \iota.
    \end{equation*}
    Therefore
    \begin{equation*}
        \max\left\{
            \frac{p_0+p_1}{q_0+q_1}, \frac{p_2}{q_2}, \cdots, \frac{p_d}{q_d}
        \right\}
        \le
        \max\left\{
            \frac{p_0}{q_0}, \frac{p_1}{q_1}, \frac{p_2}{q_2}, \cdots, \frac{p_d}{q_d}
        \right\}.
    \end{equation*}
    Hence
    \begin{equation*}
        \log
        \max\left\{
            \frac{p_0+p_1}{q_0+q_1}, \frac{p_2}{q_2}, \cdots, \frac{p_d}{q_d}
        \right\}
        \le \log\max_i \frac{p_i}{q_i}.
    \end{equation*}
    By the definition of the Funk distance, we get $\rhofd(\tilde{p},\tilde{q}) \le \rhofd(p,q)$.
\end{proof}

\begin{theorem}\label{thm:FunkMonotone}
    The Funk distance $\rhofd$ in $\Delta^d$ satisfies the information monotonicity.
\end{theorem}
The proof is straightforward from Lemma~\ref{thm:funkmono} by noting that any
coarse-grained probability can be recursively defined by merging two bins.
Since the sum of two information monotone distances is monotone, we get another proof that Hilbert distance as the sum of the forward and reverse Funk (weak) metric is information monotone.

\begin{figure}
\centering
\includegraphics[width=.5\textwidth]{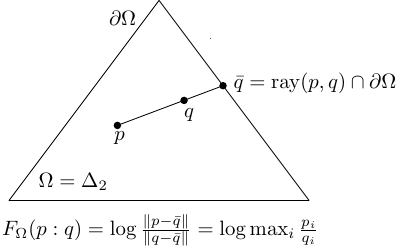}

\caption{Funk weak metric defined in the probability simplex.\label{fig:FunkSimplex}}

\end{figure}

It is interesting to explore clustering based on the Funk geometry,
which we leave as a future work.


\section{$K$-divergence}

The $K$-divergence~\cite{Lin-1991,JensenFamily-2010} is a $f$-divergence
(hence, information monotone~\cite{IG-2016}) defined by
\begin{equation*}
\rhok(p,q) = \rho_{\mathrm{IG}}\left(p,\frac{p+q}{2}\right)
=\sum_{i=0}^d \lambda_p^i\log \frac{2\lambda_p^i}{\lambda_p^i+\lambda_q^i}.
\end{equation*}
The celebrated Jensen-Shannon divergence~\cite{Lin-1991} is another $f$-divergence
obtained by symmetrizing the $K$-divergence:
$$
\rho_{\mathrm{JS}}(p,q)
=\frac{1}{2}\left( \rho_{\mathrm{K}}\left(p,q\right) + \rho_{\mathrm{K}}\left(q,p\right)\right),
$$
The square root of the Jensen-Shannon divergence~\cite{JSHilbert-2004} is a metric distance,
and since the square root function is a monotonously increasing function,
we deduce that $\sqrt{\rho_{\mathrm{JS}}(p,q)}$ is information monotone.
We checked experimentally that the square root of the $K$-divergence is also information monotone.
The proof is straightforward based on the information monotonicity of the KL divergence.
\begin{theorem}
The $K$-divergence $\rhok$, or its square root $\sqrt{\rhok}$, satisfies the information monotonicity.
\end{theorem}

Figure~\ref{fig:kball} shows concentric balls based on the $K$-divergence.
Figure~\ref{fig:kcluster} demonstrates the associated clustering results on toy datasets.

\begin{figure}[htp]
\centering
\begin{subfigure}{.98\textwidth}
\includegraphics[width=\textwidth]{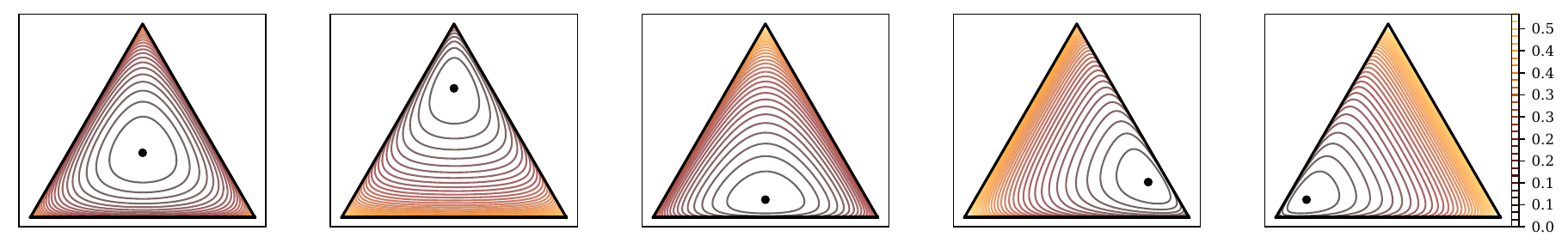}
\caption{$\rhok(p,c)$}
\end{subfigure}
\begin{subfigure}{.98\textwidth}
\includegraphics[width=\textwidth]{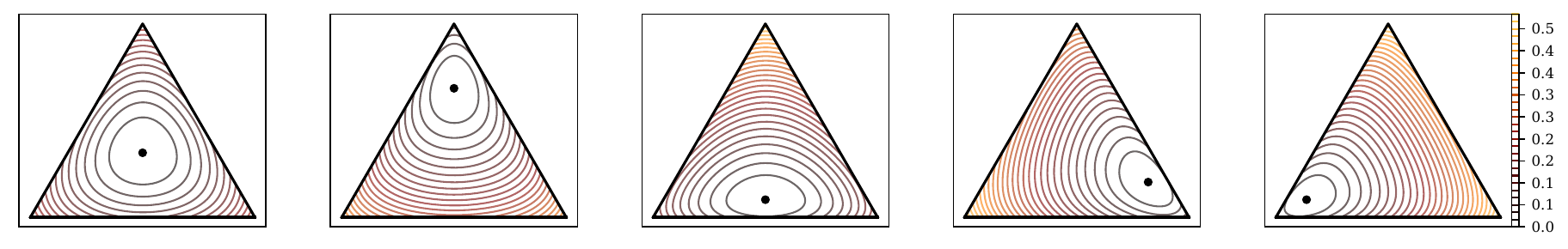}
\caption{$\rhok(c,p)$}
\end{subfigure}
\caption{Balls centered at $c\in\Delta^2$ with its radius measured by the $K$-divergence
increasing at constant speed.}%
\label{fig:kball}%
\end{figure}

\begin{figure}[htp]
\centering
\begin{subfigure}{.4\textwidth}
\includegraphics[width=\textwidth]{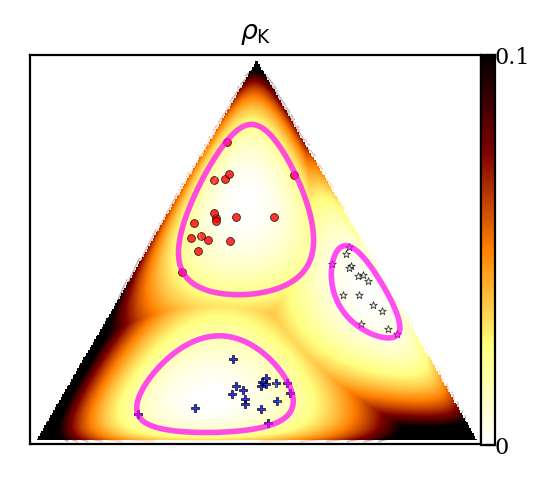}
\caption{$k=3$ clusters}
\end{subfigure}
\begin{subfigure}{.4\textwidth}
\includegraphics[width=\textwidth]{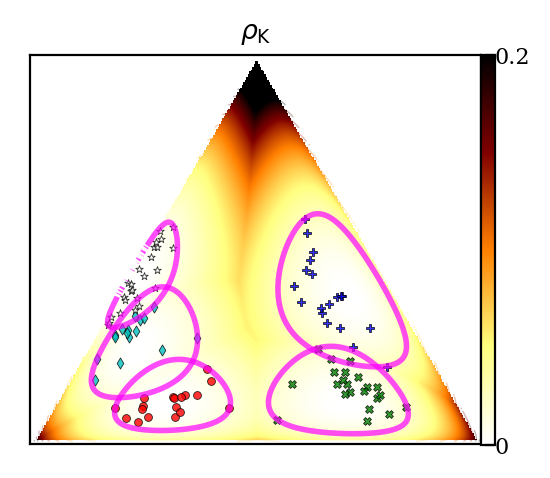}
\caption{$k=5$ clusters}
\end{subfigure}                                                                                                                                                                                                                                 \caption{$k$-center clustering results on a toy dataset in the space of trinomials $\Delta^2$
based on the $K$-divergence. The color density maps indicate the distance from any point to
its nearest cluster center.
\label{fig:kcluster}}
\end{figure}

\section{An efficient algorithm to compute Hilbert cross-ratio distance for a polytope domain}
Without loss of generality, we present the calculation technique for the simplex below (although that in that case we have the direct formula of Eq.~\ref{eq:BG}).
The method extends straightforwardly to arbirary polytope domain.

Given $p,q\in\Delta^d$, we first need to compute the intersection of the line $(pq)$ with the border of the $d$-dimensional probability simplex to get the two intersection points $p'$ and $q'$ so that $p',p,q,q'$ are ordered on $(pq)$.
Once this is done, we simply apply the formula in Eq.~\ref{eg:hgd} to get the Hilbert distance.

A $d$-dimensional simplex consists of $d+1$ vertices with their corresponding $(d-1)$-dimensional facets.
For the probability simplex $\Delta^d$, let $e_i=(\underbrace{0,\ldots,0}_{i}, 1, 0, \ldots,0)$ denote the $d+1$
vertices of the standard simplex embedded in the hyperplane $H_\Delta: \sum_{i=0}^d \lambda^{i}=1$ in $\bbR^{d+1}$.
Let $f_{\backslash{}j}$ denote the simplex facets that is the convex hull of all vertices except $e_j$:
$f_{\backslash j}=\co(e_0,\ldots,e_{j-1},e_{j+1},\ldots,e_{d})$.
Let $H_{\backslash{}j}$ denote the hyperplane supporting this facet,
which is the affine hull $f_{\backslash j}=\aff(e_0,\ldots,e_{j-1},e_{j+1},\ldots,e_d)$.

To compute the two intersection points of $(pq)$ with $\Delta^d$, a naive algorithm
consists in computing the unique intersection point $r_j$ of the line $(pq)$ with each hyperplane
$H_{\backslash{}j}$ ($j=0,\cdots,d$) and checking whether $r_j$ belongs to $f_{\backslash{}j}$.

\begin{figure}
\centering

\includegraphics[width=.5\textwidth]{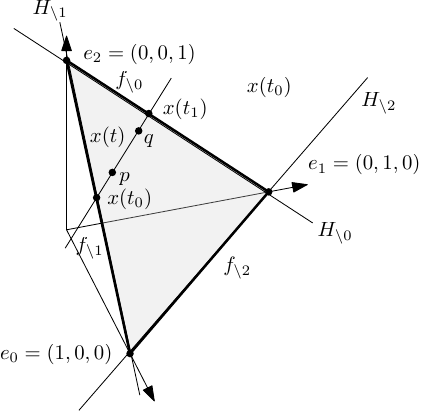}

\caption{Calculating the two intersection points $x(t_0)$ and $x(t_1)$ of the line $(pq)$ with the boundary of the probability simplex $\Delta_d$: For each facet $f_{\backslash i}$, we calculate the intersection point of line $x(t)=(1-t)p+tq$ with the $d$-dimensional hyperplane $H_{\backslash i}$ supporting the facet $f_{\backslash i}$.\label{fig:vizcomp}}
\end{figure}

A much more efficient implementation given by Alg.~(\ref{alg:distance}) calculates
the intersection point of the line $x(t)=(1-t)p+tq$ with each $H_{\backslash{}j}$ ($j=0,\cdots,d$).
These intersection points are represented using the coordinate $t$.
For example, $x(0)=p$ and $x(1)=q$. Due to convexity,
any intersection point with $H_{\backslash{}j}$ must satisfy either $t\le0$ or $t\ge1$.
Then, the two intersection points with $\partial\Delta^d$ are obtained by
$t_0=\max\{t:\,\exists{j},\;x(t)\in{H}_{\backslash{j}}\text{ and }t\le0\}$ and
$t_1=\min\{t:\,\exists{j},\;x(t)\in{H}_{\backslash{j}}\text{ and }t\ge1\}$.
Figure~\ref{fig:vizcomp} illustrates this calculation method.
This algorithm only requires $O(d)$ time and $O(1)$ memory.

\begin{lemma}
The Hilbert distance in the probability simplex can be computed in optimal $\Theta(d)$ time.
\end{lemma}

\begin{algorithm}[t]
\KwData{Two points $p=(\lambda_p^0,\cdots,\lambda_p^d)$, $q=(\lambda_q^0,\cdots,\lambda_q^d)$
in the $d$-dimensional simplex $\Delta^d$}
\KwResult{Their Hilbert distance $\rhohg(p,q)$}
\Begin{
$t_0\leftarrow-\infty$;~$t_1\leftarrow+\infty$\;
\For{$i=0\cdots{d}$}{
\If{$\lambda_p^i\neq\lambda_q^i$}{
$t\leftarrow\lambda_p^i/(\lambda_p^i-\lambda_q^i)$\;
\uIf{$t_0<t\le0$}{
$t_0\leftarrow{}t$\;
}\ElseIf{$1\le{t}<t_1$}{
$t_1\leftarrow{}t$\;}
}
}
\uIf{$t_0=-\infty$ or $t_1=+\infty$} {
Output $\rhohg(p,q)=0$\;
} \uElseIf{$t_0=0$ or $t_1=1$}{
Output $\rhohg(p,q)=\infty$\;
}\Else{
Output $\rhohg(p,q)=\left\vert\log(1-\frac{1}{t_0})-\log(1-\frac{1}{t_1})\right\vert$\;
}}
\caption{Computing the Hilbert distance}\label{alg:distance}
\end{algorithm}

\section{Python code snippets\label{sec:python}}

The code for computing the Hilbert simplex metric is as follows:

\begin{lstlisting}[language=Python]
     def HilbertSimplex( self, other ):
        if np.allclose( self.p, other.p ): return 0
        idx = np.logical_not( np.isclose( self.p, other.p ) )
        if ( idx.sum() == 1 ): return 0

        lamb = self.p[idx] / (self.p[idx] - other.p[idx])
        t0 = lamb[ lamb <= 0 ].max()
        t1 = lamb[ lamb >= 1 ].min()
        if np.isclose( t0, 0 ) or np.isclose( t1, 1 ): return np.inf

        return np.abs( np.log( 1-1/t0 ) - np.log( 1-1/t1 ) )
\end{lstlisting}

The code for computing the Hilbert elliptope metric between two correlation matrices is as follows:

\begin{lstlisting}[language=Python]
def HilbertCorrelation( self, other ):
    S1invS2 = np.linalg.solve( self.C, other.C )
    lamb = np.linalg.eigvals( S1invS2 )
    return np.log( lamb.max() ) - np.log( lamb.min() )
\end{lstlisting}

\end{document}